\documentclass{article}

\usepackage[main,final]{neurips_2026}
\usepackage{hyperref}
\usepackage{amsmath}
\usepackage{mathtools}
\usepackage{amssymb}
\usepackage{mathtools}
\usepackage{amsthm}
\usepackage{fmtcount}
\usepackage{tcolorbox}
\usepackage{subcaption}

\usepackage[capitalize,noabbrev]{cleveref}

\theoremstyle{plain}
\newtheorem{theorem}{Theorem}[section]
\newtheorem{proposition}[theorem]{Proposition}
\newtheorem{lemma}[theorem]{Lemma}
\newtheorem{corollary}[theorem]{Corollary}
\theoremstyle{definition}

\newtheorem{assumption}[theorem]{Assumption}
\theoremstyle{remark}
\newtheorem{remark}[theorem]{Remark}

\def\R{\mathbb{R}}

\def\E{\mathbb{E}}
\def\P{\mathbb{P}}
\def\I{\mathbb{I}}

\def\S{\mathcal{S}}

\def\G{\mathcal{G}}

\def\X{\mathcal{X}}
\def\Y{\mathcal{Y}}
\def\z{\mathcal{z}}

\def\F{\mathcal{F}}

\def\F{\mathcal{F}}
\def\L{\mathcal{L}}

\def\x{\mathbf{x}}
\def\y{\mathbf{y}}
\def\w{\mathbf{w}}

\def\v{\mathbf{v}}
\def\z{\mathbf{z}}
\def\e{\mathbf{\epsilon}}

\usepackage{extarrows}
\usepackage{mathrsfs}

\newcommand{\set}[2]{\left\{#1\;\middle\vert\;#2\right\}}

\newcommand{\setm}[3]{\left\{#2\vphantom{\raisebox{#1ex}{}}\;\middle\vert\;\vphantom{\raisebox{#1ex}{}}#3\right\}}

\newcommand{\ykh}[1]{\left(#1\right)}
\newcommand{\zkh}[1]{\left[#1\right]}
\newcommand{\hkh}[1]{\left\{#1\right\}}
\newcommand{\norm}[1]{\left\|#1\right\|}
\newcommand{\abs}[1]{\left|#1\right|}
\newcommand{\mc}{\mathcal}
\newcommand{\mb}{\mathbb}

\newcommand{\mr}{\mathrm}

\newcommand{\mbR}{\mathbb{R}}

\newcommand{\bm}{\boldsymbol}
\newcommand{\bigykh}[1]{\big(#1\big)}

\newcommand{\ceil}[1]{\left \lceil#1\right\rceil} 
\newcommand{\flr}[1]{\left \lfloor#1\right\rfloor} 

\newcommand{\fdnn}{\mathcal{F}_{\mathtt{N}\mkern-1.55mu\mathtt{N}}}

\newcommand{\fscoring}{\mathcal{F}_{\mathtt{S}}}

\def\[#1\]{\begin{align*}#1\end{align*}}

\newcommand{\pddm}{{\mathtt{P}\mkern-3.3mu\mathtt{D}}}

\newcommand{\zzhcyi}{C_1}
\newcommand{\zzhcer}{C_2}

\usepackage[textsize=tiny]{todonotes}

\usepackage[utf8]{inputenc} 
\usepackage[T1]{fontenc}    
\usepackage{hyperref}       
\usepackage{url}            
\usepackage{booktabs}
\usepackage{enumitem}
\usepackage{amsfonts}       
\usepackage{nicefrac}       
\usepackage{microtype}      
\usepackage{xcolor}         

\title{Approximation in Contrastive Representation Learning }

\author{%
  Yuanfan Li \\
   School of mathematics and statistics \\
 University of Sydney\\
  \texttt{yuli0846@uni.sydney.edu.au} \\
   \And
 Zihan Zhang \\
   School of mathematics and statistics \\
 University of Sydney\\
 \texttt{zihan.zhang@sydney.edu.au} \\
     \And
 Yiming Ying \thanks{Corresponding author} \\
 School of mathematics and statistics \\
 University of Sydney\\
  \texttt{yiming.ying@sydney.edu.au}
   \And
Ding-Xuan Zhou \\
  School of mathematics and statistics \\
 University of Sydney\\
  \texttt{dingxuan.zhou@sydney.edu.au} \\
}

\begin{document}

\maketitle

\begin{abstract}
  Contrastive representation learning (CRL) is a fundamental paradigm for
  learning transferable representations.  Recent work
  identifies its population target as the pointwise mutual information (PMI)
  \(
    \log\frac{p(\x,\y)}{p_\X(\x)p_\Y(\y)}
  \)
  between two modalities \(\X,\Y\), where \(p\) is the joint density function and
  \(p_\X,p_\Y\) are the marginals.  However, it remains unclear why standard inner-product scores \(f(\x)^\top g(\y)\) can effectively approximate a generally nonseparable population target PMI. In this paper, we address this question from the aspect of approximation for CRL. By linking PMI to a compact operator, we show that \(D\)-dimensional inner-product models perform rank-\(D\) spectral approximation, with error controlled by the spectral tail. We then study neural network realizations: in general settings, we show that deep neural network scoring functions are universally consistent, and the approximation error rates depend on the smoothness of the PMI function; under structured low-rank data models we obtain improved rates attained by moderate-width encoders. Finally, we establish a fast calibration rate from contrastive excess risk to downstream retrieval error under a density-ratio Tsybakov noise condition.
\end{abstract}

\section{Introduction}
Representation learning has been a prominent paradigm
in foundation models and achieved great success \citep{pmlr-v139-radford21a,jia2021scaling,karpukhin2020dense}. Typically, representations are learned from large-scale datasets and then transferred to downstream tasks.
One of the most successful approaches is   \emph{contrastive representation learning} (CRL),
 which has been widely used in computer vision, natural language processing, and multimodal learning \citep{chen2020simple,he2020momentum,an2023sparse,pmlr-v139-radford21a,jia2021scaling}.

For two modalities $\X$ and $\Y$, CRL aims to learn a scoring function $s:\X\times\Y\mapsto \R$, that assigns higher scores to positive pairs than to negative pairs. 
By pulling together semantically related pairs $(\x,\y)$ and pushing apart mismatched pairs $(\x,\y')$ in the embedding space, CRL learns transferable and task-agnostic representations. 
In the self-supervised setting, this mechanism is commonly implemented by minimizing the population objective \citep{wang2020understanding}:
\begin{equation}\label{eq:con}
\mathcal{L}(s)
=
\mathbb{E}_{(\x,\y)\sim p}
\,\tau\log
\mathbb{E}_{\y'\sim p_\Y}
\exp\!\left(
\frac{s(\x,\y')-s(\x,\y)}{\tau}
\right),
\end{equation}
where $p$ is the joint distribution of $\x,\y$,  $p_\Y$ is the marginal distribution of $\y$.

\citep{bao2022surrogate,ash2022investigating} analyse surrogate gaps and excess risks for contrastive losses, providing insights into their generalisation behaviour. However, they do not establish Fisher consistency or explicitly characterise the population minimiser of the contrastive objective. 

To address this, \cite{li2026statisticalconsistencygeneralizationcontrastive} show that the population minimiser $s^*$ of the contrastive loss indeed recovers the log-density ratio  $\log\frac{p(\x,\y)}{p_\X(\x)p_\Y(\y)}$, which is the pointwise mutual information
(PMI) of two modalities (up to an additive function).
Building on this characterisation, they further derive a calibration-type inequality: for any scoring function $s$, there holds
\begin{align}\label{eq:calibration}
\mathcal E^* - \mathcal E(s)
\lesssim
\sqrt{\mathcal L(s) - \mathcal L^*},\end{align}
which quantitatively connects pretraining excess risk to downstream retrieval performance.  Here, the retrieval risk of a scoring function $s:\mathcal X\times\mathcal Y\to\mathbb R$ is defined as
\[
\mathcal E(s)
:=
\mathbb{P}\big(s(\x,\y) \ge s(\x,\y')\big),
\]
which measures the probability that a randomly drawn positive pair is ranked above a negative pair. The optimal (Bayes) retrieval risk is
$\mathcal E^* := \max_{s} \mathcal E(s)$, where the maximum is taken over all measurable functions $s:\X\times\Y\to \R$, and $\L^*=\L(s^*)$ is the minimal contrastive loss.

Typically in the context of CRL \citep{chen2020simple,pmlr-v139-radford21a}, the scoring function takes the form of inner product of representations , i.e., $s\in \mathcal{S}_{\mathrm{IP}} =\{s(\x,\y) = f(\x)^\top g(\y) \mid    f\in\F,g\in\G\}$, where $\F,\G$ are hypothesis spaces, e.g., deep neural networks (DNNs). This limits the class of admissible scoring functions, 
 as there is no obvious way to decompose $\log\frac{p(\x,\y)}{p_\X(\x)p_\Y(\y)}$ into separate factors of $\x$ and $\y$, and thus direct
inner-product approximation appears challenging.
Yet, the learned representations still generalizes in downstream tasks. Building on this perspective, we ask the following central question:
\begin{tcolorbox}[colback=white,colframe=black,boxrule=0.4pt]
How well can a given inner-product representation class approximate pointwise mutual information underlying contrastive learning?
\end{tcolorbox}

To address this question, we study CRL within the framework of statistical learning theory. 
Let $\hat{s}$ be a minimizer of an empirical loss $\widehat{\mathcal{L}}$ over a hypothesis class $\mathcal{S}_{\mathrm{IP}}$, and let $\bar{s} \in \arg\min_{s \in \mathcal{S_{\mathrm{IP}}}} \mathcal{L}(s)$. 
The excess risk admits the decomposition
\begin{align*}
\mathcal{L}(\hat{s}) - \mathcal{L}^* 
&= \bigl[ \mathcal{L}(\hat{s}) - \widehat{\mathcal{L}}(\hat{s})\bigr]
 + \bigl[ \widehat{\mathcal{L}}(\hat{s}) - \widehat{\mathcal{L}}(\bar{s})\bigr]
 + \bigl[ \widehat{\mathcal{L}}(\bar{s}) - \mathcal{L}(\bar{s})\bigr]
 + \bigl[ \mathcal{L}(\bar{s}) - \mathcal{L}^* \bigr] \\
&\le 2 \sup_{s \in \mathcal{S}_{\mathrm{IP}}} 
 | \mathcal{L}(s) - \widehat{\mathcal{L}}(s) |
+ \inf_{s \in\mathcal{S}_{\mathrm{IP}}} 
\bigl( \mathcal{L}(s) - \mathcal{L}^* \bigr).
\end{align*}

Existing analyses \citep{pmlr-v97-saunshi19a,pmlr-v202-lei23a,li2026statisticalconsistencygeneralizationcontrastive} establish bounds on the generalization gap, i.e.,
$\sup_{s \in \mathcal{S}_{\mathrm{IP}}} | \mathcal{L}(s) - \widehat{\mathcal{L}}(s) |$, 
but do not address whether the inner-product hypothesis class $\mathcal{S}_{\mathrm{IP}}$ is expressive enough to capture the optimal log-density-ratio structure. In particular, the approximation error 
$\inf_{s \in \mathcal{S}_{\mathrm{IP}}} \bigl( \mathcal{L}(s) - \mathcal{L}^* \bigr)$
remains largely unexplored.

As the sample size increases, the generalization term vanishes, and the performance is ultimately governed by the \emph{approximation error} $\inf_{s \in \mathcal{S}_{\mathrm{IP}}} \bigl( \mathcal{L}(s) - \mathcal{L}^* \bigr)$, which measures the gap between the best achievable risk within $\mathcal{S}_{\mathrm{IP}}$ and the population optimum. Characterizing this quantity is therefore key to explaining the empirical success of overparameterized models and understanding the fundamental limits of contrastive representation learning.

In this work, we develop a theoretical analysis of CRL that characterizes the role of approximation in determining its achievable performance. We also propose a  density-ratio Tsybakov low-noise condition and establish a fast calibration rate. 
Our contributions are listed as follows:
\begin{itemize}[leftmargin=*, itemsep=0.25em]
    \item We first study the approximation capability of the finite-dimensional inner-product model class
\[
\mathcal{S}_{\mathrm{IP}}=\mathcal{S}_D \coloneqq \{\, s(\x,\y)=f(\x)^{\top}g(\y)\;|\; f:\X\to\R^{D},\, g:\Y\to\R^{D} \,\}.
\]
Let the \emph{PMI operator} $T:L^2(P_\X)\to L^2(P_\Y)$ be defined by
$ 
(Tf)(\y) := \int \log\frac{p(\x,\y)}{p_\X(\x)p_\Y(\y)} \,f(\x)\, dP_\X(\x).
$
Denote its singular value decomposition by $ 
T\psi_k = \sigma_k \phi_k, \quad T^*\phi_k = \sigma_k \psi_k,$
where $\{\psi_k\}$ and $\{\phi_k\}$ are orthonormal systems in $L^2(P_\X)$ and $L^2(P_\Y)$, respectively. 
We show that CRL implicitly performs a low-rank factorisation of $T$, where the encoders $f$ and $g$ capture the top-$D$ singular functions $\{\psi_k,\phi_k\}_{k=1}^D$. The approximation error is governed by the representation dimension $D$ and the decay of the singular values $\{\sigma_k\}$. 

    \item We further study the approximation error when $f,g$ are parameterized by DNNs. We first show the universal approximation property of DNNs. Then we derive an explicit rate under some smoothness assumptions of the PMI function. We further analyze a structured data model where the PMI operator is low-rank, for which we show that small dimensional encoders  can achieve small approximation error through learning the singular functions $\{\psi_k\},\{\phi_k\}$.
    \item We present fine-grained analysis of calibration-type inequality for retrieval in the case of low noise. We introduce a density ratio Tsybakov noise condition, formulated through the density-ratio margin between two samples. 
    Large margin implies low-noise and easy ranking task, thus leading to tighter  calibration inequality. 
\end{itemize}

\section{Related Work}\label{app:related_work}
The seminal work of \citet{pmlr-v97-saunshi19a} proposed a theoretical framework for analyzing the generalization performance of contrastive representation learning, which was later refined by \citet{pmlr-v202-lei23a}. Another line of work studies the \emph{surrogate gap}, namely the discrepancy between the upstream contrastive objective and downstream task performance \citep{bao2022surrogate,ash2022investigating}. More recently, \citet{li2026statisticalconsistencygeneralizationcontrastive} further developed the statistical learning theory of CRL by establishing Fisher consistency, calibration inequalities, and improved generalization bounds.

A complementary line of research analyzes CRL from spectral and kernel perspectives. 
\citet{haochen2021provable} introduced a spectral contrastive loss and showed that it recovers eigenfunctions of a Laplacian operator defined on an augmentation graph. 
\citet{pmlr-v162-saunshi22a} emphasized the importance of inductive bias in explaining the success of contrastive learning. 
\citet{johnson2023contrastivelearningoptimalbasis} proposed the positive-pair kernel and connected CRL with principal component analysis, while \citet{zhaiunderstanding} studied augmentation-based CRL through RKHS approximation and regression, showing that sufficiently strong augmentations alone can induce generalization.
Our work differs from these analyses in two key aspects. 
First, prior works typically study matrices/operators induced by positive-pair distributions or augmentation graphs, which do not directly correspond to the population minimizer of the standard log-sum-exp contrastive objective. 
In contrast, we analyze the operator induced by the pointwise mutual information (PMI), which characterizes the optimal contrastive score. 
Second, we explicitly study the two-encoder structure $f(\x)$ and $g(\y)$ and its approximation properties under inner-product scoring functions.

Several recent studies have also begun to investigate approximation aspects of CRL. 
\citet{oko2025statisticaltheorycontrastivepretraining} proposed a joint generative hierarchical model and considered a  scoring function of the form
$
\tau^w( \mathrm{softmax}(f(\x)), \mathrm{softmax}(g(\y))), 
$
where $\tau^w$ is a nonlinear function parameterized by $w$. They
showed that transformers can approximate the optimal score by implementing belief propagation. 
\citet{uesaka2025weightedpointsetembedding} proposed weighted point-set embeddings with a universal kernel to compute pointwise mutual information. 
However, these approaches rely on more expressive scoring functions or kernel constructions, and do not fully explain why the standard inner-product form
$
s(\x,\y)=f(\x)^\top g(\y)
$
is sufficiently expressive in practice.

Further related work is discussed in Appendix~\ref{app:further_related_work}.

\section{Problem Formulation}\label{sec:problem}

\paragraph{Notations.}
Let  $\mathbb{N}$ denote the set of positive integers, $d_1, d_2\in\mathbb{N}$, and 
$\X\subset \R^{d_1}$ and $\Y\subset\R^{d_2}$.  
Denote $(\x,\y)$ a data pair with $\x\in\X$ and $\y\in\Y$.
We denote by $P_{\X\Y}$ the joint distribution of $(\x,\y)$. The marginal distributions of $P_{\X\Y}$ on $\X$ and $\Y$  are denoted by $P_\X$ and $P_\Y$, respectively. If $P_{\X\Y}$ is continuous, we use $p(\x,\y)$, $p_\X(\x)$ and $p_\Y(\y)$ to denote the probability density functions of $P_{\X\Y}$, $P_\X$ and $P_\Y$, respectively; if $P_{\X\Y}$ is discrete, we use $p(\x,\y)$, $p_\X(\x)$ and $p_\Y(\y)$ to denote the probability mass functions of $P_{\X\Y}$, $P_\X$ and $P_\Y$, respectively. By default, we assume that $P_{\X\Y}$ is continuous. For simplicity, we sometimes write $\E_\x$ for $\E_{\x\sim p_\X}$. 
Let $L_2(P_\X)$ and $L_2(P_\Y)$ denote the Hilbert spaces of square-integrable
functions with respect to $P_\X$ and $P_\Y$, respectively. 
We write $a \lesssim b$ if $a \le c\,b$ for some constant $c>0$, and
$a \asymp b$ if there exist constants $c_1,c_2>0$ such that
$c_1a \le b \le c_2a$.  Denote by $\I$ the indicator function, i.e., $\I\{E\}=1$ if the event $E$ holds true.

We focus on self-supervised CRL throughout the paper. Our theoretical analysis can be naturally extended to general distributions for positive  and negative samples.
Recall that the population contrastive risk is defined as follows \citep{wang2020understanding,li2026statisticalconsistencygeneralizationcontrastive},
\begin{equation}\label{eq:contrastive-risk}
\mathcal{L}(s)
\coloneqq
\mathbb{E}_{(\x,\y)\sim p}\,
\left[
\tau\log
\mathbb{E}_{\y'\sim p_\Y}
\exp\left(
\frac{s(\x,\y')-s(\x,\y)}{\tau}
\right)
\right],
\end{equation}
where $\tau>0$ is the temperature parameter. This objective encourages the
score of a positive pair $(\x,\y)$ to be larger than that of a negative pair
$(\x,\y')$.
\paragraph{Population Target as PMI.}
The optimal score for the contrastive risk \eqref{eq:contrastive-risk} is established in the following lemma.
\begin{lemma}[\citep{li2026statisticalconsistencygeneralizationcontrastive}]\label{lem:minimizer-L}
Let $\mathcal{L}(s)$ be defined as in ~\eqref{eq:contrastive-risk}.
A measurable function $s:\X\times\Y\to\mathbb R$ minimizes $\mathcal L(s)$
if and only if it satisfies $s(\x,\y)
=
\tau\log\frac{p(\y\mid\x)}{p_\Y(\y)}
+
a(\x)$
for some measurable function $a:\X\to\mathbb R$.
\end{lemma}

Let $a(\x)=0$, then the essential
population target of CRL is the pointwise mutual information (PMI)
\begin{equation}\label{eq:pmi-target}
\log\frac{p(\y\mid \x)}{p_\Y(\y)}
=
\log\frac{p(\x,\y)}{p_\X(\x)p_\Y(\y)}\coloneqq \operatorname{PMI}(\x,\y)
.
\end{equation}

\paragraph{Representation Classes and Approximation Error.}
In practice, PMI is learned through a structured hypothesis class. The most common
choice is the inner-product form
\begin{equation}\label{eq:inner-product-class}
\mathcal{S}_{\mathrm{IP}}
=
\left\{
s(\x,\y)=f(\x)^\top g(\y)
\,\middle|\,
f\in\F,\ g\in\G
\right\}.
\end{equation}
where $\F$ and $\G$ are representation classes, such as deep neural network
classes. 
The central object of this paper is 
to approximate PMI through the functions in $\S_{\mathrm{IP}}$. To this end, we analyze
the approximation error $$
\inf_{s\in\S_{\mathrm{IP}}}
\left[
\mathcal L(s)-\mathcal L^*
\right].$$ In this paper, we consider two representation classes: 1) a finite-dimensional \emph{inner-product model class}, where the dimension $D$ controls the representation capacity:
\begin{equation}\label{eq:SD-class}
  \mathcal{S}_{\mathrm{IP}}  =\mathcal S_D
\coloneqq
\left\{
s(\x,\y)=f(\x)^\top g(\y)
\,\middle|\,
f:\X\to\mathbb R^D,\ 
g:\Y\to\mathbb R^D
\ \text{measurable}
\right\}.
\end{equation}
2) a hypothesis space $\mathcal{S}_{\mathrm{NN}}$ where $\F,\G$ are DNNs, which is defined in \eqref{20260505161740}.

\paragraph{Calibration Inequality for Downstream Retrieval.}
We use the scoring function learned by
CRL  for downstream retrieval. For a scoring function
$s:\X\times\Y\to\mathbb R$, we measure its retrieval quality by the
AUC-type population criterion
\begin{equation}\label{eq:retrieval-performance}
\mathcal E(s)
\coloneqq
\mathbb E_{(\x,\y)\sim p}
 \E_{\y'\sim p_\Y}
\left[
\mathbb I\{s(\x,\y)>s(\x,\y')\}+\frac12\mathbb I\{s(\x,\y)=s(\x,\y')\}
\right].
\end{equation}
 Larger values of $\mathcal E(s)$ correspond to
better retrieval performances. We define the optimal retrieval value as $\mathcal E^*
\coloneqq
\max_s \mathcal E(s)$,  where the maximum is taken over all measurable functions $s:\X\times\Y\to \R$. 
It is expected that minimizing the contrastive risk $\L(s)$ would maximize $\mathcal{E}(s)$. 
This statistical property, can be quantitatively described in the following inequality
\citep{li2026statisticalconsistencygeneralizationcontrastive}:
\begin{equation}\label{eq:distribution-free-calibration}
\mathcal E^*-\mathcal E(s)
\lesssim
\sqrt{\mathcal L(s)-\mathcal L^*}.
\end{equation}
This shows that a small contrastive excess risk is guaranteed to achieve a strong
retrieval performance in the downstream task.
Therefore, the approximation error controls the best downstream retrieval performance
achievable by $\S_{\mathrm{IP}}$:
$\inf_{s\in\S_{\mathrm{IP}}}\zkh{\mathcal E^*
-\;
\mathcal E(s)}
\lesssim\;
\sqrt{\inf_{s\in\S_{\mathrm{IP}}}
\left[
\mathcal L(s)-\mathcal L^*
\right]
}.$

\section{Approximation Analysis with Finite-dimensional Inner-Product Representations}
\label{sec:rkhs_approximation}

In this section we investigate how well the pointwise mutual information (PMI)
can be approximated by $D$-dimensional  inner-product scores $\S_D$ defined in~\eqref{eq:SD-class}.
The analysis constructs low-rank spectral comparators for a PMI operator,
with approximation guarantees controlled by two factors: the embedding dimension \(D\) and the spectral
decay of the PMI operator. 

\paragraph{Matrix factorization for discrete spaces.}
When $\mathcal X=\{\x_1,\ldots,\x_n\}$ and
$\mathcal Y=\{\y_1,\ldots,\y_m\}$ are finite spaces, the PMI function
can be viewed as a matrix $P\in\mathbb R^{n\times m}$ with entries
\[
P_{ij}\coloneqq
\log \frac{p(\x_i,\y_j)}{p_\X(\x_i)p_\Y(\y_j)} .
\]

For two encoders
$f(\x_i),g(\y_j)\in\mathbb R^D$, let
$F=[f(\x_1),\ldots,f(\x_n)]^\top\in\mathbb R^{n\times D}$ and
$G=[g(\y_1),\ldots,g(\y_m)]^\top\in\mathbb R^{m\times D}$, then a rank-at-most-$D$ factorization $P\approx FG^\top$ provides a
comparator for inner-product contrastive scores (with score scale $\tau$). By the Eckart--Young--Mirsky theorem \citep{eckart1936approximation}, the
best such approximation in the Frobenius norm is given by the truncated SVD:
if $P=U\Sigma V^\top$, then $P_D=U_D\Sigma_DV_D^\top$,
where $U_D$ is the first $D$ columns of $U$, $V_D$ is the first $D$ columns of $V$, $\Sigma_D$ is the upper left block of $\Sigma$, 
and one may take
$F=U_D\Sigma_D^{1/2}$, $G=V_D\Sigma_D^{1/2}$.

\paragraph{Spectral decomposition on continuous domains.}
The matrix viewpoint extends naturally to continuous spaces.
Assume the PMI is square-integrable,
\(\mathrm{PMI}(\x,\y)\in L_2(P_{\X}\otimes P_{\Y})\).
Define the \emph{PMI operator}
\(T:L_2(P_{\X})\to L_2(P_{\Y})\) by
\begin{equation}\label{eq:PMI_operator}
  (Tf)(\y) = \int_{\X} \mathrm{PMI}(\x,\y)\,f(\x)\,d P_\X(\x) .
\end{equation}
Its adjoint is
\(T^*g(\x)=\int_{\Y}\mathrm{PMI}(\x,\y)g(\y)d P_{\Y}(\y)\).
The compositions \(T^*T\) and \(TT^*\) are integral operators with kernels
\[
K_{\X}(\x,\x') = \E_{\y\sim p_{\Y}}\!\big[\mathrm{PMI}(\x,\y)\,
                \mathrm{PMI}(\x',\y)\big],\qquad
K_{\Y}(\y,\y') = \E_{\x\sim p_{\X}}\!\big[\mathrm{PMI}(\x,\y)\,
                \mathrm{PMI}(\x,\y')\big].
\]
\(K_{\X}\) measures similarity between two anchors through the
lens of the entire \(\Y\) modality; for instance, two images are close
if they share similar textual descriptions.

Since the kernel is square-integrable, \(T\) is a Hilbert--Schmidt
operator.  The spectral theorem then yields the following decomposition
(proof in Appendix~\ref{app:proof_spectral}).

\begin{lemma}[Spectral decomposition of the PMI]
  \label{lem:spectral_decomposition}
  \(T^*T\) and \(TT^*\) share the same non-zero eigenvalues
  \(\lambda_1\ge\lambda_2\ge\cdots\searrow 0\).
  There exist orthonormal systems \(\{\psi_i\}\subset L_2(P_{\X})\) and
  \(\{\phi_i\}\subset L_2(P_{\Y})\) such that
  \(T\psi_i = \sigma_i\phi_i,\;
   T^*\phi_i = \sigma_i\psi_i\)
  with \(\sigma_i = \sqrt{\lambda_i}\).  Moreover,
  \begin{equation}\label{eq:PMI_decom}
    \mathrm{PMI}(\x,\y) = \sum_{i=1}^{\infty}
    \sigma_i\,\psi_i(\x)\,\phi_i(\y)
    \quad\text{in }L_2(P_{\X}\otimes P_{\Y}).
  \end{equation}
Furthermore,  
  the ranges \(\operatorname{Im}(T)\) and \(\operatorname{Im}(T^*)\) are exactly the
  reproducing kernel Hilbert spaces (RKHSs) induced by \(K_{\Y}\) and
  \(K_{\X}\), respectively. We denote these spaces by
  $\mathcal H_\Y:=\operatorname{Im}(T)$ and
  $\mathcal H_\X:=\operatorname{Im}(T^*)$, equipped with their RKHS norms.
\end{lemma}

\eqref{eq:PMI_decom} is the infinite-dimensional analogue of the
SVD of the PMI matrix. Spectral truncation therefore provides explicit encoder coordinates
through the singular functions $\{\psi_i\}$ and $\{\phi_i\}$.

\paragraph{Approximation error analysis.}
If the PMI operator has finite rank \(K<\infty\)
(i.e., \(\lambda_i=0\) for \(i>K\)), then encoders of dimension \(K\) suffice to
represent the optimal score exactly:
\(s^*(\x,\y) = \tau\sum_{i=1}^{K}\sigma_i\,\psi_i(\x)\phi_i(\y)\)
can be written as \(f(\x)^{\top}g(\y)\) with
\(f(\x)=(\sqrt{\tau}\sigma_1^{1/2}\psi_1(\x),\dots,
          \sqrt{\tau}\sigma_K^{1/2}\psi_K(\x))\) and
\(g(\y)=(\sqrt{\tau}\sigma_1^{1/2}\phi_1(\y),\dots,
        \sqrt{\tau}\sigma_K^{1/2}\phi_K(\y))\).
Section~\ref{sec:latent-class-model} provides a low-rank data model
where this finite-rank structure emerges.

In the infinite-rank case, when the encoders are restricted to a fixed dimension $D$, the optimal \(L_2\)‑approximation of the PMI within \(\mathcal{S}_D\) is indeed
achieved by the truncated spectral decomposition $\sum_{i=1}^{D}\sqrt{\lambda_i}\,\psi_i(\x)\,\phi_i(\y)$.

To derive the approximation error, we introduce a regularity condition on the
density ratio. 

\begin{assumption}\label{ass:pi_+}
  There exists a constant \(\alpha>1\) such that, for all
  \(\x\in\mathcal{X}\) and \(\y\in\mathcal{Y}\),
  \(
  \frac{p(\x,\y)}{p_{\X}(\x)\,p_{\Y}(\y)}\;\le\;\alpha .
  \)
\end{assumption}

Note that we do not impose positive lower bound, which is employed in \cite{oko2025statisticaltheorycontrastivepretraining}. Large $\alpha$ usually implies imbalanced distribution. Building on this assumption, we obtain the following
bound (proofs are provided in Appendix~\ref{app:proof_L-L*<L_2}).

\begin{theorem}[Approximation error of inner-product scores]
  \label{thm:appro_rkhs}
  Suppose that $\X,\Y$  be two compact domains, the PMI function $\log\frac{p(\x,\y)}{p_\X(\x)p_\Y(\y)}$ is continuous on \(\X\times\Y\) and that
  Assumption~\ref{ass:pi_+} holds.  Let \(\mathcal{S}_D\) be as in~\eqref{eq:SD-class} and assume the spectral score
  $s_D(\x,\y)=\tau\sum_{i=1}^D\sqrt{\lambda_i}\psi_i(\x)\phi_i(\y)$ satisfy
  $\|s_D\|_\infty\le B$.
  Then
  \[
    \inf_{s\in\mathcal{S}_D} [\mathcal{L}(s) - \mathcal{L}(s^*)]
    \;\lesssim\; \sum_{i=D+1}^{\infty} \lambda_i .
  \]
\end{theorem}

Hence the tail sum of the eigenvalues \(\{\lambda_i\}\) directly measures the
intrinsic approximation hardness: a fast decay implies that a small encoder
dimension already captures most of the mutual information. 
To illustrate this,
we consider a case when the eigenvalues admit a polynomial decay~\citep{ying2008online,guo2023capacity}.

\begin{corollary}\label{cor:rkhs}
  Under the conditions of Theorem~\ref{thm:appro_rkhs},
  suppose the eigenvalues decay polynomially:
  \(\lambda_i \asymp i^{-\gamma}\) for some \(\gamma>1\).
  Then
  \[
    \inf_{s\in\mathcal{S}_D} [\mathcal{L}(s) - \mathcal{L}(s^*)]
    \;\lesssim\; \frac{(D+1)^{1-\gamma}}{\gamma-1}.
  \]
\end{corollary}

To achieve an  approximation error \(\varepsilon\), it suffices to choose
\(D \approx \varepsilon^{-1/(\gamma-1)}\).
Larger \(\gamma\) (faster decay) thus permits much smaller embedding
dimension.  Smooth PMI kernels typically exhibit rapid eigenvalue
decay~\citep{ferreira2009eigenvalues}. We conduct numerical experiments on Gaussian mixture model in Appendix~\ref{app:gmm_experiments}.

Based on our theoretical findings, we give two insights on why moderate dimensions of representations work well in practice: (i) the data exhibits a low-dimensional semantic structure,
        (ii) the PMI function is smooth, meaning that similar inputs have
        similar co-occurrence patterns and outliers are rare.

\begin{remark}[Contrastive risk versus $L_2$ approximation]
In our analysis, we first give the approximation error in $L_2$ norm then we control the contrastive risk by $L_2$ distance. However, these two measures are not equivalent in general. The log-sum-exp structure of contrastive risk is more sensitive to hard negatives. Furthermore, contrastive excess risk can be written as KL divergence between the conditional distribution $p(\cdot|\x)$ and a distribution induced by $s: q_s(\y\mid\x)=p_\Y(\y)e^{s(\x,\y)/\tau}/\E_{\y'}e^{s(\x,\y')/\tau}, \mathcal L(s)-\mathcal L^*=\tau\E_\x\mathrm{KL}(p(\cdot\mid\x)\Vert q_s(\cdot\mid\x))$. Consequently, the global minimizer over the class $\mathcal{S}_D$ of $D$-dimensional inner-product scores is an KL-divergence projection.
\end{remark}

\begin{remark}[RKHS encoders and cosine similarity]
We also study the approximation ability of RKHSs $\mathcal{H}_\X$ and $\mathcal{H}_\Y$ with bounded norms. The first $r$ spectral components have squared
encoder norms $\tau\sum_{j=1}^r\lambda_j^{-1/2}$;
RKHS with larger norm therefore permits a smaller spectral tail. For cosine scoring functions, we show that $D+2$-dimension cosine representations can realize truncated scores with top-$D$ eigenfunctions. See Appendix~\ref{app:rkhs_cosine}
for more details.
\end{remark}


\section{Approximation Analysis with DNNs}

The previous section characterises the approximation error of PMI via a spectral decomposition, indicating that inner-product spaces can approximate PMI well provided the eigenvalues of the operator decay fast.
We now turn to the question of \emph{realisability}: can PMI be effectively represented by practical models?

In practice, CRL parameterises the scoring function via inner products of deep neural network (DNN) encoders. 
In Section~\ref{sec:DNN_general},
we provide a general approximation analysis for DNN-based inner-product models, providing theoretical guarantees for their expressive power in approximating PMI. In Section~\ref{sec:latent-class-model}, we consider a structured data model where the PMI operator is low-rank, 
we show that DNNs can overcome the curse of dimensionality by learning the singular functions of PMI.

\subsection{Approximation via DNNs: General Case}
\label{sec:DNN_general}
 We show that such DNN-based representations can well approximate the target PMI function
 $\log\frac{p(\x,\y)}{p_\X(\x)p_\Y(\y)}$,
 and we establish corresponding approximation error.

\paragraph{Neural network space.}
Let $\sigma(t)=\max\hkh{0,t}$ be the ReLU function. For any $d\in\mb{N}$ and any $ \v\in\mbR^d$, we define  the so-called  $\v$-shifted ReLU  function  
\[
\sigma_\v:\mbR^d\to[0,\infty)^d,\;\x\mapsto \bigykh{\sigma((\x-\v)_1),\sigma((\x-\v)_2),\ldots, \sigma((\x-\v)_d)}. 
\]  A ReLU neural network  is a function $f$ of the form 
\begin{align}\label{240318220428}
f:\mbR^d\to\mbR^{D},\;\x\mapsto \bm W_L\sigma_{\v_L}\bm W_{L-1}\sigma_{\v_{L-1}}\cdots\bm W_2\sigma_{\v_2}\bm W_1\sigma_{\v_1}\bm W_0\x, 
\end{align} where $d\in\mb{N}$, $D\in\mb{N}$, $L\in\mb N\cup\hkh{0}$, $\mr m_i\in \mb N$ ($0\leq i\leq L+1$), $\mr{m}_0=d$, $\mr m_{L+1}=D$, $\v_i\in\mbR^{\mr m_i}$ ($1\leq i\leq L$),  $\bm{W}_i\in\mbR^{\mr{m}_{i+1}\times\mr{m}_i}$ ($0\leq i\leq L$). 
Define the  space  \[\fdnn(d, L,N, D):=\set{f}{\text{$f$ is defined by \eqref{240318220428} with $m_i\leq N$ for  $1\leq i\leq L$}},\] 
which consists of all ReLU DNNs of which the depth is $L$, the number of input units is $d$, the number of output units is $D$, and the number of units in each hidden layer is less than or equal to $N$. 

We consider scoring functions  $s(\x,\y)$ which take the form of $s(\x,\y)=f(\x)^\top g(\y)$, where $f(\x)$ and $g(\y)$ are ReLU neural networks. The hypothesis space of such scoring functions is specified as 
\begin{align}\label{20260505161740}
&\mathcal{S}_{\mathrm{NN}}= \fscoring(L_1, N_1, L_2, N_2, D)\notag\\&:=\set{s(\x,\y)=f(\x)^\top g(\y)}{f\in \fdnn(d_1,L_1, N_1, D), g\in \fdnn(d_2,L_2, N_2, D)}. 
\end{align}

\paragraph{Universal approximation.}
Classical universal approximation theorems show that neural networks can
approximate broad classes of functions \citep{barron2002universal}. However,
our scoring class has an additional structural constraint: the score must take
the inner-product form
$
s(\x,\y)=f(\x)^\top g(\y).
$
It is therefore not immediate that DNN-based inner-product scores can
approximate the PMI target of contrastive learning. We first show that this is
indeed possible under only a mild
integrability condition. 
\begin{assumption}\label{bounded risk 20260421190226}
$\mathbb{E}_{(\x,\y) \sim p}
    {
   \abs{\log{
    \frac{p(\x,\y)}{p_\X(\x)p_{\Y}(\y)}
    }}
    }<\infty$. 
\end{assumption}
\begin{theorem}\label{univApproxDNNs20260425200219}If Assumption \ref{bounded risk 20260421190226} holds, then for any   $L_1\in\mb{N}\cap(1,\infty)$,  $L_2\in\mb{N}\cap (1,\infty)$,  $\e\in(0,\infty)$, there exist 
$N_1\in\mb{N}$, $N_2\in\mb{N}$,     $D\in\mb{N}$, and $\tilde{s}\in \fscoring(L_1,N_1, L_2,N_2, D)$ such that $\mathcal{L}(\tilde{s})-\mathcal{L}^*\leq \e$.  
\end{theorem}

The proof of Theorem~\ref{univApproxDNNs20260425200219} is provided in
Appendix~\ref{20260425195716}. 
Assumption~\ref{bounded risk 20260421190226} is very mild: it is equivalent to
the finiteness of the optimal contrastive risk, i.e., $\L^*>-\infty$ (see Lemma \ref{20260507053917} for the proof of this claim), and is
expected to hold in standard CRL settings. 
Therefore, Theorem~\ref{univApproxDNNs20260425200219} gives a distribution-free
guarantee without any smoothness,
boundedness, or low-rank assumptions,  which helps explain effectiveness of DNNs across diverse datasets.

\paragraph{Approximation rates.}
The universal approximation result above is qualitative. 
To obtain explicit approximation rates, we impose some regularity assumption on the optimal scoring function $s^*(\x,\y)=\tau \log\frac{p(\x,\y)}{p_\X(\x)p_\Y(\y)}$. In this work, we consider  the Hölder continuity assumption. 

For $\beta\in(0,1]$, $k\in\mb{N}$, $\Omega\subset \R^k$, and  a real-valued function $h$ whose domain contains $\Omega$,  the H\"older-$\beta$ norm of $h$ on $\Omega$ is defined as \[\norm{h}_{\beta,\Omega}:=\max\hkh{\sup_{\z\in\Omega}\abs{h(\z)}, \sup_{\Omega\ni \x\neq \y\in\Omega}\frac{\abs{h(\x)-h(\y)}}{\norm{\x-\y}_\infty^\beta}}. \] H\"older continuity is a common regularity assumption in the theoretical analysis of approximation errors (e.g. \citep{lu2021deep, ShenYangZhang2022Optimal}) and generalization errors (e.g., \citep{schmidt2020nonparametric, kim2021fast,zhang2024classification}). 

\begin{assumption}\label{20260427030032}
$\X\times \Y \subset[-R,R]^{d_1+d_2}$ and  $s^*(\x,\y)=\tau \log\frac{p(\x,\y)}{p_\X(\x)p_\Y(\y)}$ satisfies $\norm{s^*}_{\beta,\X\times\Y}\leq B$ for some $R\in(0,\infty)$,  $\beta\in(0,1]$,  and $B\in(0,\infty)$. 
\end{assumption}

Note that  Assumption \ref{20260427030032} is stronger than Assumption \ref{bounded risk 20260421190226}.

\begin{theorem}\label{4p3thm20260427130502}
Let Assumption \ref{20260427030032} hold. Then for a
 sufficiently small number $\e\in(0,+\infty)$ and ${D\asymp \e^{-\frac{\min\hkh{d_1,d_2}}{2\beta}}}$,  ${L_1\asymp \e^{-\frac{d_1}{d_2}\cdot\frac{\min\hkh{d_1,d_2}}{8\beta}}}$, $N_1 \asymp \e^{-\ykh{4+\frac{d_1}{d_2}}\cdot\frac{\min\hkh{d_1,d_2}}{8\beta}}$, $L_2 \asymp \e^{-\frac{d_2}{d_1}\cdot\frac{\min\hkh{d_1,d_2}}{8\beta}}$, ${N_2 \asymp \e^{-\ykh{4+\frac{d_2}{d_1}}\cdot\frac{\min\hkh{d_1,d_2}}{8\beta}}}$, it holds that   $\inf\set{\L(\tilde{s})-\L^*}{ \tilde{s}\in \fscoring(L_1,N_1,L_2,N_2,D)}\lesssim \e$.  
\end{theorem}
The proof of Theorem \ref{4p3thm20260427130502} is provided in Appendix  \ref{aE20260427130617}. The proof proceeds as follows. We first derive a more general result (see Theorem~\ref{20260503141526}), where we establish approximation rate for  ${L_1\asymp \e^{-\theta_{1,1}}}$, $N_1 \asymp \e^{-\theta_{1,2}}$, $L_2 \asymp \e^{-\theta_{2,1}}$, $N_2 \asymp \e^{-\theta_{2,2}}$, $D\asymp \e^{-\theta_0}$ with $\theta_{1,1}$, $\theta_{1,2}$, $\theta_{2,1}$, $\theta_{2,2},\theta_0$ being arbitrary fixed positive numbers.  
We then select specific $\theta_{1,1}$, $\theta_{1,2}$, $\theta_{2,1}$, $\theta_{2,2},\theta_0$ to minimize the pseudo-dimension of the space $\fscoring(L_1,N_1,L_2,N_2,D)$ required to achieve the same approximation error $O(\e)$, and  Theorem \ref{20260503141526} with this specific choice of $\theta_{1,1}$, $\theta_{1,2}$, $\theta_{2,1}$, $\theta_{2,2},\theta_0$ leads to Theorem \ref{4p3thm20260427130502} immediately.
It turns out that the minimum pseudo-dimension is $\widetilde{O}\ykh{\e^{-(\min\hkh{d_1,d_2}+d_1+d_2)/(2\beta)}}$, attained by  the  space $\fscoring(L_1,N_1,L_2,N_2,D)$ specified in Theorem \ref{4p3thm20260427130502}. 
Recall that pseudo-dimension is widely used in learning theory to quantify the complexity of a hypothesis space (cf. \cite{bartlett2009}). In  Theorem \ref{20260503025931}, we show that  the pseudo-dimension of our hypothesis space $\fscoring(L_1,N_1,L_2,N_2,D)$ is $\widetilde{O}(L_1^2N_1^2\!+\!L_2^2N_2^2\!+\!DN_1\!+\!DN_2)$. We use the classical result on the approximation of Hölder continuous functions by neural networks given in \cite{shen2020DeepNe}
 (see also Lemma~\ref{20260428014457}) as a key lemma in the proofs of Theorem~\ref{20260503141526} and, as we shall see,  Theorem~\ref{20260505072733}.
 
 However, relying solely on this result in  \cite{shen2020DeepNe}  is far from sufficient to prove these theorems. It only gives rates for approximating Hölder continuous functions by a single neural network, while in our setting, we need to approximate  $s^*(\x,\y)=\tau\log\frac{p(\x,\y)}{p_{\mathcal{X}}(\x)p_{\mathcal{Y}}(\y)}$ by the  inner product of two ReLU networks $\tilde{f}(\x)^\top \tilde{g}(\y)=\sum_{i=1}^D\tilde{f}_i(\x)\tilde{g}_i(\y)$, rather than by a single  network $\tilde{h}(\x,\y)$. 
 In fact, our proof requires a \emph{nontrivial} finite-rank approximation
$s^*(\x,\y)\approx \sum_{i=1}^D f_i(\x)g_i(\y)$, 
where $f_i$ (or $g_i$) is very easy to be approximated by ReLU networks and $g_i$ (or $f_i$) is Hölder continuous. Then the only place where we invoke \cite{shen2020DeepNe}  is to construct a ReLU network $\tilde{g}_i$ (or $\tilde{f}_i$) to approximate $g_i$ (or $f_i$). We also emphasize the importance of Lemma~\ref{20260428072356}, 
which is crucial to the proof of Theorem~\ref{20260503141526} and is established through a delicate application of Hoeffding's lemma.
Lemma~\ref{20260428072356} fundamentally bridges the uniform approximation of $s^*$ to the risk $\mathcal L(\cdot)-\mathcal L^*$, regardless of the underlying hypothesis space, and therefore applies far beyond neural networks, thus constituting one of our major contributions.

Theorem~\ref{4p3thm20260427130502} gives a quantitative 
guarantee for DNN-based inner-product CRL: it specifies how the size of networks scale with the target error
$\e$. The dependence on $\beta$ shows that smoother PMI targets require smaller
networks, consistent with the spectral approximation result in
Section~\ref{sec:rkhs_approximation}. The dependence on $d_1,d_2$ reflects the
standard curse of dimensionality in approximation
\citep{JMLR:v18:14-546}.

A key implication of the bound is that the two encoders are \emph{not independent}.
The size of the $\X$-encoder depends not only on $d_1$, but also on $d_2$, and
vice versa for the $\Y$-encoder. This is because the two networks are coupled
through the shared inner-product score $f(\x)^\top g(\y)$, which aims to
approximate the PMI function on $\X\times\Y$. Furthermore, the required dimension for the representation $D$ only depends on $\min\{d_1,d_2\}$.
Therefore, each
encoder need be designed relative to the other modality in order to capture
mutual information.

\subsection{Approximation via DNNs: Structured Data}
\label{sec:latent-class-model}

The previous subsection gives a general approximation for DNN-based
inner-product CRL. However, the required representation
dimension $D$ can be exponentially large in the worst case. Practical data often exhibit
low-dimensional semantic structure, and embeddings can capture them with moderate dimensions (e.g., 512, 1024). We next introduce a latent-class
generative model (LCGM) as a tractable example to show that DNNs approximate more efficiently on structured data. In LCGM, the PMI operator is low-rank, and its
singular functions admit explicit probabilistic interpretations in terms of
posterior probabilities and log-posterior ratios of the latent classes. 

\begin{assumption}[Latent-class generative model, LCGM]
\label{ass:latent-class-model}
There exists a latent semantic variable $z\in[K]$ with prior
$\pi_k=p(z=k)$ such that, conditional on $z$, the two views are independent:
\begin{equation}\label{eq:independent}
    p(\x,\y| z=k)=p(\x| z=k)p(\y| z=k).
\end{equation}
Moreover, the anchor view is low-noise: for each $\x\in\X$, there exists a
ground-truth latent label $z_\x\in[K]$ such that
\begin{equation}\label{eq:deterministic}
    p(z=z_\x| \x)=1 .
\end{equation}
\end{assumption}

The conditional independence condition \eqref{eq:independent} is standard in
theoretical analyses of contrastive learning
\citep{chen2024understandingtransferablerepresentationlearning,
oko2025statisticaltheorycontrastivepretraining}. The deterministic condition
\eqref{eq:deterministic} means that the anchor view $\x$ reveals the latent
class. Similar low-noise assumptions have been used in
\citet{haochen2021provable}. This condition can be relaxed to a high-probability
posterior condition at the cost of additional technical details, please refer to Appendix~\ref{app:high-probability-lcgm} for more details; we keep the
deterministic version for clarity. For example, $z$ may represent the semantic
class ``cat'', $\x$ a clean image of a cat, and $\y$ a  textual description
of a cat (may be noisy).

The following proposition characterizes the spectral structure of the PMI
operator under LCGM. The formal statement and proof are provided in
Appendix~\ref{app:proof_lanten_spectrum}.

\begin{proposition}[Spectral structure under LCGM, informal of Proposition~\ref{prop:latent-spectrum}] \label{prop:latent-spectrum_informal} Under Assumption~\ref{ass:latent-class-model}, there holds \begin{enumerate}[leftmargin=*] \item The PMI operator $T:L_2(P_{\X})\to L_2(P_{\Y})$ has rank at most $K$. \item $\mathrm{span}\{\psi_1,\dots,\psi_K\}\! =\! \mathrm{span}\{p(z=k|\x)\}_{k=1}^K, \mathrm{span}\{\phi_1,\dots,\phi_K\} \!=\! \mathrm{span}\{\log p(z=k|\y)\}_{k=1}^K.$ \item $ \log\frac{p(\x,\y)}{p_{\X}(\x)p_{\Y}(\y)}=\sum_{k=1}^K p(z=k|\x)\log\frac{p(z=k|\y)}{\pi_k}$ \end{enumerate} \end{proposition}

Therefore, PMI can be approximated effectively through $K$-dimensional representations by constructing DNNs $f,g$ such that \[f(\x)\approx[p(z=1|\x),\cdots,p(z=K|\x)]^\top,\quad g(\y)\approx\Bigg[\log\frac{p(z=1|\y)}{\pi_1},\cdots,\log\frac{p(z=K|\y)}{\pi_K}\Bigg]^\top.\] In this way, the two representations have different meanings: The encoder for the low-noise modality $\X$ implements the classifier through learning the posterior probability $p(z=k|\x)$, while $g$ learns to fit the log-conditional-probability ratios $\log\frac{p(z=k|\y)}{\pi_k}$ for each class $k$.
Numerical experiments in Appendix~\ref{sec:experiments} illustrate this
mechanism in the synthetic LCGM setting: CRL recovers the predicted singular
functions, and the effective representation dimension matches the number of
latent classes.

The approximation error of DNNs for LCGM is established as follows:
\begin{theorem}\label{20260505072733}
Let Assumption \ref{20260427030032}  and Assumption \ref{ass:latent-class-model} hold. Then for a 
 sufficiently small number $\e\in(0,\infty)$ and  $L_1\asymp \log\frac1\e$, $N_1 \asymp \log\frac1\e$,  $L_2\gtrsim \log\frac1\e$, $N_2\gtrsim \log\frac1\e$ with $L_2\cdot N_2\asymp\e^{-\frac{d_2}{4\beta}}$, it holds that   
 \[\inf\set{\L(\tilde{s})-\L^*}{ \tilde{s}\in \fscoring(L_1,N_1,L_2,N_2,K)}\lesssim \e.\]  
\end{theorem}
The proof of Theorem \ref{20260505072733} is provided in Appendix \ref{20260505072637}. 
\begin{table}[t]
\centering
\begin{tabular}{lrrrr}
\toprule
 & $D$  & $L_1\cdot N_1$ & $L_2\cdot N_2$ &$\widetilde{{\pddm}}$ \\
\midrule
General case\!\!\! &${ \e^{-\frac{\min\hkh{d_1,d_2}}{2\beta}}}$&$\e^{-(4+\frac{2d_1}{d_2})\cdot\frac{\min\hkh{d_1,d_2}}{8\beta}}$& $\e^{-(4+\frac{2d_2}{d_1})\cdot\frac{\min\hkh{d_1,d_2}}{8\beta}}$ &$\e^{-\frac{\min\hkh{d_1,d_2}+d_1+d_2}{2\beta}}$\\
LCGM &$K$& $  \log^2\frac1\e$&  $\e^{-\frac{d_2}{4\beta}}$ & $\e^{-\frac{d_2}{2\beta}}$\\
\bottomrule
\end{tabular}
\caption{Comparison of the network sizes to achieve approximation error $\e$. Here $D$ is the dimension of the representation, $L_1,N_1$ denote the depth and width of the encoder $f$, $L_2,N_2$ denote the depth and width of the encoder $g$, $\widetilde{\pddm}$ is the pseudo-dimension of the space $\fscoring(L_1,N_1,L_2,N_2,D)$ (omitting logarithmic factors),  $K$ is the number of latent classes.}
\label{tab:compare}
\vspace{-2em}
\end{table}
Table~\ref{tab:compare} compares Theorem~\ref{20260505072733} with the general
DNN approximation result in Theorem~\ref{4p3thm20260427130502}. The structured
model leads to a substantially smaller representation and network size. In particular, the dimension of the encoder reduces from exponential term to a constant, while the depth and width of $f$ are scaled from exponential to logarithmic terms. One can also check that $\ykh{4+\frac{2d_2}{d_1}}\cdot\frac{\min\hkh{d_1,d_2}}{8\beta}>\frac{d_2}{4\beta}$.  Therefore, DNNs can learn PMI efficiently on structured data by approximating the spaces spanned by the singular functions $\{\psi_k\}\{\phi_k\}$.

\begin{remark}
Although   the results in this section are established for neural networks with ReLU activation, our overall proof framework \emph{does not depend on ReLU}. If networks with another activation function satisfy the approximation properties required in our proofs, then Theorem~\ref{univApproxDNNs20260425200219}, Theorem \ref{4p3thm20260427130502}, and Theorem~\ref{20260505072733} can be extended accordingly. For example, Theorem~\ref{univApproxDNNs20260425200219} also holds for sigmoid networks because, according to \cite{van2024noncompact}, sigmoid networks can uniformly approximate any compactly supported continuous function on $\mathbb{R}^d$, as ReLU networks do. 
\end{remark}
\section{Fast Calibration Rates under a Tsybakov Noise Condition}
\label{sec:fast-calibration}
Now we have shown that CRL can learn representations that approximate PMI function and achieve small approximation error. In this section, we study how these representations perform in downstream retrieval. The core of the analysis is the calibration inequality
which
compares excess risks of the upstream and downstream tasks, such as \eqref{eq:distribution-free-calibration}.
However, it suffers from low convergence rate.

In statistical learning theory, the rates can be improved under 
Tsybakov low-noise condition. 
In binary classification, the classical Tsybakov noisy condition assumes that the posterior rarely lies close to the decision boundary, which lead to faster calibration rates \citep{tsybakov2004optimal,zhang2004statistical}.
 According to Lemma~3.3 of \citet{li2026statisticalconsistencygeneralizationcontrastive}, a scoring
function $s$ achieves the optimal retrieval performance, i.e.,
$\mathcal E(s)=\mathcal E^*$, if and only if for all $\x\in\X$ and
$\y,\y'\in\Y$,
\[\left(
\frac{p(\y\mid \x)}{p_\Y(\y)}
-
\frac{p(\y'\mid \x)}{p_\Y(\y')}
\right)
\bigl(s(\x,\y)-s(\x,\y')\bigr)>0 .\]
Thus, the Bayes ranking is determined by the density ratio 
$
\frac{p(\y\mid \x)}{p_\Y(\y)}$. Intuitively, the retrieval problem is easier when the density-ratio scores of candidate
items are well separated. Motivated by this observation, we impose the following
low-noise condition for retrieval.

\begin{assumption}[Tsybakov low-noise condition for retrieval]
\label{ass:tsy_crl}
There exist constants $\mu\ge 0$ and $C>0$ such that, for all $t>0$,
\[\P\Big(\Big|\frac{p(\y\mid\x)}{p_\Y(\y)}-\frac{p(\y'|\x)}{p_\Y(\y')}\Big|\le t\Big)\le ct^\mu,\]
where $\x\sim p_\x$ and $\y,\y'\sim p_\Y$ are sampled independently. Here, $\mu$ refers to the \bf{noise exponent}.
\end{assumption}

The above assumption always holds with $\mu=0$. A larger exponent $\mu$ means
that near-ties between candidate items occur less frequently, so the Bayes
ranking has a larger margin and the retrieval problem is less noisy. Under this
condition, contrastive excess risk transfers to retrieval performance at a
faster rate.

\begin{theorem}[Fast calibration under density-ratio margin]
\label{thm:fast_rate_tsy_crl}
Let Assumptions~\ref{ass:pi_+} and~\ref{ass:tsy_crl} hold. Then
\[
\mathcal E^*-\mathcal E(s)
\lesssim
\bigl(\mathcal L(s)-\mathcal L^*\bigr)^{\frac{\mu+1}{\mu+2}} .
\]
\end{theorem}
The proof is provided in Appendix~\ref{sec:fast_rate_tsy_crl}.
Since
$
\frac{\mu+1}{\mu+2}\ge \frac12,
$
Theorem~\ref{thm:fast_rate_tsy_crl} improves the distribution-free calibration
rate in \eqref{eq:distribution-free-calibration}. Moreover, the exponent
$(\mu+1)/(\mu+2)$ increases with $\mu$ and approaches $1$ as
$\mu\to\infty$. Therefore, when the data distribution contains fewer near-tie
candidate pairs, small contrastive excess risk translates more directly into
strong downstream retrieval performance.

\section{Conclusion}
\label{sec:conclusion}
In this paper, we developed an approximation theory for contrastive
representation learning. We showed that finite-dimensional inner-product
representations perform low-rank spectral approximation of the PMI operator,
and that the resulting approximation error is governed by the representation
dimension and the spectral decay of this operator. We then studied the
realisability of such approximations by DNN encoders, establishing both a
general approximation guarantee and sharper rates under structured latent-class
distributions.  Finally, we introduced a density-ratio Tsybakov condition
under which contrastive excess risk enjoys faster calibration to downstream
retrieval performance.

There are several promising future directions. It would be very interesting to study the training dynamics of CRL, for example, can first-order optimization methods learn PMI efficiently? Another problem is to study the expressive
power of DNN-based cosine similarity function, which is more difficult to analyze.

\section*{Acknowledgement}
The authors thank the reviewers and Area Chair for their helpful feedback and suggestions. The work is supported by Australian Research Council Discovery Project (DP250101359).


\bibliographystyle{plainnat}
\bibliography{contrastive-representation,zihanref}

\newpage
\appendix



\section{Further Related Work}
\label{app:further_related_work}

\paragraph{InfoNCE loss.}
InfoNCE is one of the most widely used objectives in contrastive representation
learning. It contrasts each anchor with a finite number of negative samples,
usually drawn from the same mini-batch, and has achieved strong empirical
success in self-supervised and multimodal learning
\citep{oord2018representation,chen2020simple,pmlr-v139-radford21a}. 
InfoNCE has also been connected to mutual-information maximization
\citep{oord2018representation,zhang2023deep}. In practice, however, algorithms
based on InfoNCE often benefit from large batch sizes; for example, CLIP uses
very large batches during training \citep{pmlr-v139-radford21a}. To reduce this
dependence, \citet{pmlr-v162-yuan22b} proposed the global contrastive loss
(GCL), which contrasts each anchor against all negatives , and developed SogCLR as an efficient stochastic optimization
method that does not rely on large batch sizes. There are also some works using different losses for contrastive representation learning, for example, SigLIP\citep{zhai2023sigmoid}, which replaces batch-wise softmax by pairwise sigmoid losses for matched and unmatched pairs.
At the population level, this is binary classification between positive pairs from joint distribution $p$ and negative pairs from $P_X\otimes P_Y$. Moreover, the global minimizer of logistic loss is log-density-ratio $\log\frac{p(x,y)}{p_X(x)p_Y(y)}$. Therefore, the optimal scoring function for SigLIP is exactly PMI, making our approximation analysis relevant to SigLIP.

\paragraph{Other theories of contrastive representation learning.}
Several other theoretical perspectives have been proposed to understand CRL.
\citet{wang2020understanding} analyzed contrastive learning through the
alignment and uniformity properties of representations on the hypersphere.
Other works study the training dynamics of contrastive objectives, including
dimensional collapse and feature learning dynamics
\citep{jing2022understandingdimensionalcollapsecontrastive,tian2021understandingselfsupervisedlearningdual}.
\citet{wang2025discriminativeprobabilisticmodelingselfsupervised} analyzed
self-supervised contrastive learning from the perspective of discriminative
probabilistic modeling, while \citet{tsai2020demystifying} studied
self-supervised learning under a multi-view framework. 
Some works study InfoNCE from the perspective of Word2Vec and its training dynamics \citep{levy2014neural,karkada2026closed}.
More recently,
\citet{bansal2026understandingselfsupervisedlearninggaussian} investigated contrastive learning under Gaussian
mixture models and showed that contrastive learning can recover low-dimensional
structure. Some work also shows that InfoNCE induces Gaussian distribution \citep{betser2026infonce}. 

\paragraph{Approximation via DNNs.}
Deep neural networks have shown to admit universal approximation ability \citep{barron2002universal}. Error bounds for deep ReLU neural networks and CNNs have been established \citep{yarotsky2017error,zhou2020universality}. Following the idea of \cite{yarotsky2017error}, it is proved in \cite{schmidt2020nonparametric} that 
\[
\sup_{\substack{h:[-R,R]^d\to\R\\\norm{h}_{\beta,[-R,R]^d}}}\inf\set{\sup_{[-R,R]^d}\abs{f-h}}{\begin{minipage}{0.43\textwidth}
$f\in\fdnn\ykh{d,\log\tfrac1\e, \e^{-\frac{d}{\beta}},1}$ and the number of nonzero parameters of $f\lesssim \e^{-\frac{d}{\beta}}\cdot\log\tfrac1\e$
\end{minipage}}\lesssim \e. 
\] This rate turns to be sharp, and many results use this approximation rate to obtain optimal (up to logarithmic terms) convergence rates for generalization errors in various machine learning problems, such as regression \citep{schmidt2020nonparametric} and classification \citep{kim2021fast, zhang2024classification}. However, the restriction 
\begin{align}\label{20260507204828}
\text{ the number of nonzero parameters of $f\lesssim \e^{-\frac{d}{\beta}}\cdot\log\tfrac1\e$}
\end{align} makes training over the hypothesis space
\[
\set{f}{\begin{minipage}{0.43\textwidth}
$f\in\fdnn\ykh{d,\log\tfrac1\e, \e^{-\frac{d}{\beta}},1}$ and the number of nonzero parameters of $f\lesssim \e^{-\frac{d}{\beta}}\cdot\log\tfrac1\e$
\end{minipage}}
\] impossible to implement.
In \cite{shen2020DeepNe,ShenYangZhang2022Optimal, lu2021deep}, the authors use a novel technique based on  bit extraction to obtain sharp neural network  approximation rates with the restriction like \eqref{20260507204828} removed. Specifically, they  establish bounds of the form  \begin{align}\label{20260507205821}
\sup_{\substack{h:[-R,R]^d\to\R\\\norm{h}_{\beta,[-R,R]^d}}}\inf\set{\sup_{[-R,R]^d}\abs{f-h}}{f\in\fdnn\ykh{d,L,  N,1}}\lesssim \ykh{\frac{\log^{\iota_1} L }{L}\cdot\frac{\log^{\iota_2} N}{N}}^{2\beta/d}, 
\end{align} where $\iota_1, \iota_2\in\mb{R}$ are absolute constants.  In the recent work \citep{nagler2026optimal}, the authors obtain \eqref{20260507205821}
with the additional requirement that all network weights and biases are polynomially bounded in $N$. Such control over parameter magnitudes may  be useful in other settings.
In this work, we will mainly use Theorem 1.1 of \cite{shen2020DeepNe} to construct neural networks $\tilde{f}(\x)$ and $\tilde{g}(\y)$, such that $\tilde{f}(\x)^\top\tilde{g}(\y)$ can approximate the H\"older-$\beta$ continuous target function $s^*(\x,\y)=\tau\log\frac{p(\x,\y)}{p_\X(\x)p_\Y(\y)}$ well. 

\section{Proofs in Section~\ref{sec:rkhs_approximation}}\label{app:proof_rkhs}

\subsection{Proof of Lemma~\ref{lem:spectral_decomposition}}\label{app:proof_spectral}
\begin{proof}[Proof of Lemma~\ref{lem:spectral_decomposition}]
For $\lambda_i,\lambda_j>0$, we have
\begin{align*}
    \langle \lambda_i^{-1/2}T\psi_i,\lambda_j^{-1/2}T\psi_j \rangle_{P_\Y}
    &=\langle \lambda_i^{-1/2}T^*T\psi_i,\lambda_j^{-1/2}\psi_j \rangle_{P_\X}\\
    &=\frac{\sqrt{\lambda_i}}{\sqrt{\lambda_j}}\delta_{ij}.
\end{align*}
where the first equality is because $T,T^*$ are adjoint. Therefore, we can let $\phi_i=\lambda_i^{-1/2}T\psi_i$ and then $T\psi_i=\sqrt{\lambda_i}\phi_i$. The inverse direction also holds similarly. 

For almost every $\x$, $\log \frac{p(\x,\y)}{p_\X(\x)p_\Y(\y)}\in L_2(P_\Y)$. We have
\begin{align*}
    \Big\langle \log \frac{p(\x,\y)}{p_\X(\x)p_\Y(\y)},\phi_i \Big\rangle_{P_\Y}=\int \log \frac{p(\x,\y)}{p_\X(\x)p_\Y(\y)}\phi_i(\y)p_\Y(\y)d\y=(T^*\phi_i)(\x)=\sqrt{\lambda_i}\psi_i(\x).
\end{align*}
Therefore, $\log \frac{p(\x,\y)}{p_\X(\x)p_\Y(\y)}=\sum_{i=1}^\infty \sqrt{\lambda_i}\psi_i(\x)\phi_i(\y)$.

Now consider $g\in L_2(P_\Y)$ with coefficients
$a_i=\langle g,\phi_i\rangle$. Then
$T^*g=\sum_i a_i\sqrt{\lambda_i}\psi_i$, and
\[
\operatorname{Im}(T^*)=
\left\{\sum_{i:\lambda_i>0}c_i\psi_i:
\sum_i c_i^2/\lambda_i<\infty\right\}=\mathcal H_\X,
\qquad \left\|\sum_i c_i\psi_i\right\|_{\mathcal H_\X}^2
=\sum_i c_i^2/\lambda_i.
\]
 This is exactly the RKHS of $K_\X$ by its feature-map construction.
Interchanging $\X$ and $\Y$ gives the corresponding claim for $T$ and
$\mathcal H_\Y$.
\end{proof}

\subsection{Proof of Theorem~\ref{thm:appro_rkhs}}
\label{app:proof_L-L*<L_2}
To start, we provide the following lemma, which bounds the excess risk of contrastive loss through $L_2$ distance.
\begin{lemma}\label{lem:L-L^*le L_2}
Let Assumption~\ref{ass:pi_+} hold, let
$s^*=\tau\log r\in L_2(P_\X\otimes P_\Y)$, where
$r(\x,\y)=p(\y\mid\x)/p_\Y(\y)$, and let $\|s\|_\infty\le B$.
Then
\[
\mathcal L(s)-\mathcal L(s^*)\le
\frac{\alpha e^{2B/\tau}}{2\tau}
\|s-s^*\|_{L_2(P_\X\otimes P_\Y)}^2.
\]
\end{lemma}
\begin{proof}
Denote $\Delta:=s-s^*$ and, for $\theta\in[0,1]$, define
$s_\theta:=s^*+\theta\Delta$ and $G(\theta):=\mathcal L(s_\theta)$.
Let
\[
Z_\theta(\x):=\E_{\y\sim p_\Y}
\exp\left(\frac{s_\theta(\x,\y)}{\tau}\right),\qquad
q_{\theta,\x}(\y):=
\frac{p_\Y(\y)\exp(s_\theta(\x,\y)/\tau)}{Z_\theta(\x)}.
\]
Thus $q_{\theta,\x}$ is a probability density on $\Y$ induced by the scoring function
$s_\theta(\x,\cdot)$.
Since $r\le\alpha$ and $\E_{\y\sim p_\Y}r(\x,\y)=\int p_\Y(\y)\frac{p(\y|\x)}{p_\Y(\y)}d\y=\int p(\y|\x)d\y=1$, we have
\[
\E_{\y\sim p_\Y}r(\x,\y)^{1-\theta}\ge\alpha^{-\theta}\E_{\y\sim p_\Y}r(\x,\y)=\alpha^{-\theta}.
\]

By the definition of the contrastive risk,
\[
G(\theta)=\tau\E_{\x}\log Z_\theta(\x)
-\E_{(\x,\y)\sim p}s_\theta(\x,\y).
\]
Differentiating along this path gives
\begin{align*}
G'(\theta)
&=\E_{\x}\left[
\E_{\y\sim q_{\theta,\x}}\Delta(\x,\y)
-\E_{\y\sim p(\cdot\mid\x)}\Delta(\x,\y)\right],\\
G''(\theta)
&=\frac1\tau\E_{\x}\left[
\E_{\y\sim q_{\theta,\x}}\Delta(\x,\y)^2
-\left(\E_{\y\sim q_{\theta,\x}}\Delta(\x,\y)\right)^2\right]\\
&=\frac1\tau\E_{\x}
\operatorname{Var}_{\y\sim q_{\theta,\x}}\bigl(\Delta(\x,\y)\bigr).
\end{align*}
At $\theta=0$, $Z_0(\x)=1$ and $q_{0,\x}(\y)=p(\y\mid\x)$,
so $G'(0)=0$. The second-order Taylor formula with integral remainder
therefore yields
\begin{align}
\mathcal L(s)-\mathcal L(s^*)
&=G(1)-G(0)
=G'(0)+\int_0^1(1-\theta)G''(\theta)\,d\theta\notag\\
&=\frac1\tau\int_0^1(1-\theta)\E_{\x}
\operatorname{Var}_{\y\sim q_{\theta,\x}}\bigl(\Delta(\x,\y)\bigr)
\,d\theta.
\label{eq:rkhs-risk-taylor}
\end{align}

We now bound the variance in \eqref{eq:rkhs-risk-taylor}. The lower bound
on $Z_\theta$ and $|s|\le B$ imply
\[
\frac{q_{\theta,\x}(\y)}{p_\Y(\y)}
=\frac{r(\x,\y)^{1-\theta}e^{\theta s(\x,\y)/\tau}}
{Z_\theta(\x)}
\le\frac{\alpha^{1-\theta}e^{\theta B/\tau}}
{\alpha^{-\theta}e^{-\theta B/\tau}}
=\alpha e^{2\theta B/\tau}\le\alpha e^{2B/\tau}.
\]
Consequently,
\begin{align*}
\operatorname{Var}_{\y\sim q_{\theta,\x}}\bigl(\Delta(\x,\y)\bigr)
&=\inf_{c\in\R}\E_{\y\sim q_{\theta,\x}}
\bigl[(\Delta(\x,\y)-c)^2\bigr]\\
&\le\alpha e^{2B/\tau}\inf_{c\in\R}\E_{\y\sim p_\Y}
\bigl[(\Delta(\x,\y)-c)^2\bigr]\\
&=\alpha e^{2B/\tau}
\operatorname{Var}_{\y\sim p_\Y}\bigl(\Delta(\x,\y)\bigr).
\end{align*}

Substituting this variance bound into \eqref{eq:rkhs-risk-taylor} gives
\begin{align*}
\mathcal L(s)-\mathcal L(s^*)
&\le\frac{\alpha e^{2B/\tau}}\tau
\int_0^1(1-\theta)\,d\theta\,
\E_{\x}\operatorname{Var}_{\y\sim p_\Y}
\bigl(\Delta(\x,\y)\bigr)\\
&\le\frac{\alpha e^{2B/\tau}}{2\tau}
\E_{\x}\E_{\y\sim p_\Y}\Delta(\x,\y)^2\\
&=\frac{\alpha e^{2B/\tau}}{2\tau}
\|s-s^*\|_{L_2(P_\X\otimes P_\Y)}^2.
\end{align*}
This completes the proof.
\end{proof}

\begin{proof}[Proof of Theorem~\ref{thm:appro_rkhs}]
Use $s_D=\tau\sum_{i=1}^D\sqrt{\lambda_i}\psi_i\phi_i\in\mathcal S_D$.
Orthogonality gives $\|s_D-s^*\|_2^2=\tau^2\sum_{i>D}\lambda_i$.
The uniform bound on $s_D$ and Lemma~\ref{lem:L-L^*le L_2} imply
\[
\inf_{s\in\mathcal S_D}(\mathcal L(s)-\mathcal L^*)
\le\mathcal L(s_D)-\mathcal L^*
\le\frac{\alpha\tau e^{2B/\tau}}2\sum_{i>D}\lambda_i.
\]
The proof is completed.
\end{proof}

\subsection{RKHS balls and cosine scores}\label{app:rkhs_cosine}
\paragraph{Approximation analysis for RKHSs.}
Let $\mathcal H_\X$ and $\mathcal H_\Y$ be the RKHSs introduced in Lemma~\ref{lem:spectral_decomposition}. For $D$-dimensional encoders $f$ and $g$, define

$\|f\|_{\mathcal H_\X^D}^2
=\sum_{\ell=1}^D\|f_\ell\|_{\mathcal H_\X}^2,
\qquad
\|g\|_{\mathcal H_\Y^D}^2
=\sum_{\ell=1}^D\|g_\ell\|_{\mathcal H_\Y}^2,$

and consider

$$
\mathcal S_{D,\kappa}
=\left\{
s(\x,\y)=f(\x)^\top g(\y):
\|f\|_{\mathcal H_\X^D}\leq\kappa,\ 
\|g\|_{\mathcal H_\Y^D}\leq\kappa
\right\}.
$$

Write $\operatorname{PMI}(\x,\y)
=\sum_{j\geq1}\sqrt{\lambda_j}\psi_j(\x)\phi_j(\y),$
and set

$$
f_j(\x)=\sqrt{\tau}\lambda_j^{1/4}\psi_j(\x),
\qquad
g_j(\y)=\sqrt{\tau}\lambda_j^{1/4}\phi_j(\y).
$$

For the first $r$ components,

$$
\|f\|_{\mathcal H_\X^r}^2
=\|g\|_{\mathcal H_\Y^r}^2
=\tau\sum_{j=1}^r\lambda_j^{-1/2}.
$$

Define
\[
r_{D,\kappa}
=\max\left\{
r\leq D:
\tau\sum_{j=1}^r\lambda_j^{-1/2}\leq\kappa^2
\right\}.
\]

The first $r_{D,\kappa}$ spectral components then belong to $\mathcal S_{D,\kappa}$, so

$$
\inf_{s\in\mathcal S_{D,\kappa}}
\|s-s^*\|_{L_2(P_\X\otimes P_\Y)}^2
\leq
\tau^2\sum_{j>r_{D,\kappa}}\lambda_j.
$$

Thus, increasing $\kappa$ allows more spectral components to be represented. In particular, if

$$
\kappa^2
\geq
\tau\sum_{j=1}^D\lambda_j^{-1/2},
$$

then the full rank-$D$ spectral truncation belongs to the RKHS ball and

$$
\inf_{s\in\mathcal S_{D,\kappa}}
\|s-s^*\|_{L_2(P_\X\otimes P_\Y)}^2
\leq
\tau^2\sum_{j>D}\lambda_j.
$$

For fixed $D$, the bound reaches the rank-$D$ spectral tail once $\kappa$ is sufficiently large; increasing both $D$ and $\kappa$ allows the approximation error to converge to zero. 

\paragraph{Cosine similarity.}
 Let
  \[s_D(\x,\y)=\tau\sum_{j=1}^D \sqrt{\lambda_j}\psi_j(\x)\phi_j(\y)=\tau f_D(\x)^\top g_D(\y)\]

be the rank-$D$ spectral truncation of the PMI, where $f_D(\x)_j=\lambda_j^{1/4}\psi_j(\x),g_D(\y)_j=\lambda_j^{1/4}\phi_j(\y)$. Choose $a,b,\tau>0$ such that $\sup_{\x,j} af_D(\x)_j\le1,\sup_{\y,j}bg_D(\y)_j\le1$ and $\tau=ab$. We further define

$$u_D(\x)=\left(a f_D(\x),\sqrt{1-a^2\|f_D(\x)\|^2},0\right),v_D(\y)=\left(b g_D(\y),0,\sqrt{1-b^2\|g_D(\y)\|^2}\right).$$

Then $u_D(\x),v_D(\y)\in\mathbb S^{D+1}$ and $u_D(\x)^\top v_D(\y)=\tau f_D(\x)^\top g_D(\y)=s_D(\x,\y).$ Therefore, $\|s^*(\x,\y)-u_D(\x)^\top v_D(\y)\|_{L_2(P_\X\otimes P_\Y)}^2=\tau^2\sum_{j>D}\lambda_j$.

Thus, the rank-$D$ spectral comparator can be represented exactly by a cosine-similarity model in dimension $D+2$.

\section{Proof of Theorem \ref{univApproxDNNs20260425200219}}\label{20260425195716}

\begin{lemma}\label{20260507053917} $\L^*=\E_{(\x,\y)\sim p}\zkh{\tau\cdot\log\frac{p_\Y(\y)}{p(\y\mid\x)}}<+\infty$. Moreover, $\L^*>-\infty$ if and only if \[\E_{(\x,\y)\sim p}\zkh{\abs{\log\frac{p(\x,\y)}{p_\X(\x)p_\Y(\y)}}}<\infty.\] 
\end{lemma}
\begin{proof}
Denote $s^*(\x,\y):=\tau\cdot\log\frac{p(\y\mid\x)}{p_\Y(\y)}$. For simplicity, we denote $p^+_\x:=p(\;\cdot\mid\x)$. 
According to Lemma \ref{lem:minimizer-L},
\[
&\L^*=\L(s^*)=
\mathbb{E}_{\x}\mathbb{E}_{\y \sim p^+_\x}\zkh{\tau \log\mathbb{E}_{\y' \sim p^-_\x}\exp\ykh{\frac{s^*(\x,\y')-s^*(\x,\y)}{\tau}}}
\\&=
\mathbb{E}_{\x}\mathbb{E}_{\y \sim p^+_\x}\zkh{\tau \log\ykh{\mathbb{E}_{\y' \sim p_\Y}\frac{p^+_{\x}(\y')}{p_\Y(\y')}\cdot \frac{p_{\Y}(\y)}{p_\x^{+}(\y)}}}
\\&=\mathbb{E}_{\x}\mathbb{E}_{\y \sim p^+_\x}\zkh{\tau \log\ykh{\int{ p_\Y(\y')}\cdot\frac{p^+_{\x}(\y')}{p_\Y(\y')}\cdot \frac{p_{\Y}(\y)}{p_\x^{+}(\y)}d\y'}}
\\&=\mathbb{E}_{\x}\mathbb{E}_{\y \sim p^+_\x}\zkh{\tau \log\ykh{\frac{p_{\Y}(\y)}{p_\x^{+}(\y)}\int{p^+_{\x}(\y')}{}d\y'}}
\\&=\mathbb{E}_{\x}\mathbb{E}_{\y \sim p^+_\x}\zkh{\tau \log{\frac{p_{\Y}(\y)}{p_\x^{+}(\y)}}}
=
\mathbb{E}_{(\x,\y)\sim p}\zkh{\tau \log{\frac{p_{\Y}(\y)}{p(\y\mid\x)}}}. 
\]
Note that 
\begin{align*}
&\mathbb{E}_{(\x,\y) \sim p}
 \!\zkh{
  \text{the positive part of }{\log{
 \frac{p_{\Y}(\y)}{p(\y\mid\x)}
 }}
 }
=
\mathbb{E}_{\x}
 \mathbb{E}_{\y \sim p^+_\x}
 \!\zkh{
  \max\hkh{\log{
 \frac{p_{\Y}(\y)}{p_{\x}^{+}(\y)}
 },0}
 }\\&\leq \mathbb{E}_{\x}
  \mathbb{E}_{\y \sim p^+_\x}
  \!\zkh{
   {{
  \frac{p_{\Y}(\y)}{p_{\x}^{+}(\y)}
  }} 
  }=
  \mathbb{E}_{\x}
    \!\zkh{
     {{
   \int \frac{p_\Y(\y)}{p_{\x}^{+}(\y)}\cdot  p^+_\x (\y)d\y
    }} 
    }
    =
      \mathbb{E}_{\x}
        \!\zkh{
         {{
       \int {p_{\Y}(\y)}d\y
        }} 
        }=\mathbb{E}_{\x}
                \!\zkh{
                 {{
               1
                }} 
                }
  =1<+\infty.
\end{align*} Thus $\L^*=\mathbb{E}_{(\x,\y)\sim p}\zkh{\tau \log{\frac{p_{\Y}(\y)}{p(\y\mid\x)}}}<+\infty$, and 
\begin{align*}
\L^*> -\infty&\iff 
\mathbb{E}_{(\x,\y)\sim p}\zkh{ \log{\frac{p_{\Y}(\y)}{p(\y\mid\x)}}}> -\infty
\\&\iff
 \mathbb{E}_{(\x, \y) \sim p}
 \!\zkh{
  \text{the negative part of }{\log{
 \frac{p_{\Y}(\y)}{p(\y\mid\x)}
 }}
 }<\infty
\\&\iff 
 \mathbb{E}_{(\x, \y) \sim p}
 \!\zkh{
  \abs{\log{
 \frac{p_{\Y}(\y)}{p(\y\mid\x)}
 }}
 }<\infty
  \iff
   \mathbb{E}_{(\x, \y) \sim p}
   \!\zkh{
    \abs{\log{
   \frac{p(\y\mid\x)}{p_{\Y}(\y)}
   }}
   }
  <\infty
  \\&\iff 
    \mathbb{E}_{(\x, \y) \sim p}
     \!\zkh{
      \abs{\log{
     \frac{p(\x, \y)}{p_\X(\x)p_{\Y}(\y)}
     }}
     }
    <\infty. 
\end{align*} This completes the proof. 
\end{proof}

\begin{lemma}\label{20260426083507}
Let Assumption \ref{bounded risk 20260421190226} hold. Denote $s^*(\x,\y):=\tau\log\frac{p(\y\mid\x)}{p_\Y(\y)}$.   Then for any bounded measurable function ${s:\X\times\Y\to\R}$ there holds 
\[
\L(s)-\L^*=\mathbb{E}_{\x}\zkh{\tau \log{{\mathbb{E}_{\y \sim p(\cdot\mid\x)}{\exp\ykh{{(s(\x,\y)-s^*(\x,\y))}/{\tau}}}}}}-
\mathbb{E}_{(\x,\y)\sim p}\zkh{{{{{s(\x,\y)-s^*(\x,\y)}}}}}. 
\]
\end{lemma}

\begin{proof}For simplicity, we denote $p^+_\x:=p(\;\cdot\mid\x)$. 
It follows from Lemma \ref{20260507053917} that
\begin{align}\label{20260426022657}
\L^*=\E_{(\x,\y)\sim p}\zkh{\tau\log\frac{p_\Y(\y)}{p(\y\mid\x)}}
=
-\E_{(\x,\y)\sim p}\zkh{\tau\log\frac{p(\y\mid\x)}{p_\Y(\y)}}
=
-\mathbb{E}_{(\x, \y) \sim p}\zkh{s^*(\x,\y)}.
\end{align}
We  also have that  
\[
&\L(s)=\mathbb{E}_{\x}\mathbb{E}_{\y \sim p^+_\x}\zkh{\tau \log\mathbb{E}_{\y' \sim p_\Y}\exp\ykh{\frac{s(\x,\y')-s(\x,\y)}{\tau}}}
\\&=\mathbb{E}_{\x}\mathbb{E}_{\y \sim p^+_\x}\zkh{\tau \log\mathbb{E}_{\y' \sim p_\Y}{\frac{\exp\ykh{{s(\x,\y')}/{\tau}}}{\exp\ykh{{s(\x,\y)}/{\tau}}}}}
=
\mathbb{E}_{\x}\mathbb{E}_{\y \sim p^+_\x}\zkh{\tau \log{\frac{\mathbb{E}_{\y' \sim p_\Y}\exp\ykh{{s(\x,\y')}/{\tau}}}{\exp\ykh{{s(\x,\y)}/{\tau}}}}}
\\&=
\mathbb{E}_{\x}\mathbb{E}_{\y \sim p^+_\x}\zkh{\tau \log{{\mathbb{E}_{\y' \sim p_\Y}\exp\ykh{{s(\x,\y')}/{\tau}}}{}}}-
\mathbb{E}_{\x}\mathbb{E}_{\y \sim p^+_\x}\zkh{\tau \log{{\exp\ykh{{s(\x,\y)}/{\tau}}}}}
\\&=\mathbb{E}_{\x}\zkh{\tau \log{{\mathbb{E}_{\y' \sim p_\Y}\exp\ykh{{s(\x,\y')}/{\tau}}}{}}}-
\mathbb{E}_{\x}\mathbb{E}_{\y \sim p^+_\x}\zkh{{{{{s(\x,\y)}{}}}}}
\\&=\mathbb{E}_{\x}\zkh{\tau \log{{\mathbb{E}_{\y' \sim p^+_\x}\zkh{\exp\ykh{{s(\x,\y')}/{\tau}}\cdot \frac{p_{\Y}(\y')}{p_\x^{+}(\y')}}}}}-
\mathbb{E}_{(\x,\y) \sim p}\zkh{{{{{s(\x,\y)}{}}}}}
\\&=
\mathbb{E}_{\x}\zkh{\tau \log{{\mathbb{E}_{\y' \sim p^+_\x}\zkh{\exp\ykh{{s(\x,\y')}/{\tau}}\cdot \exp(-s^*(\x,\y')/\tau)}}}}-
\mathbb{E}_{(\x,\y) \sim p}\zkh{{{{{s(\x,\y)}{}}}}}
\\&=
\mathbb{E}_{\x}\zkh{\tau \log{{\mathbb{E}_{\y \sim p^+_\x}{\exp\ykh{{(s(\x,\y)-s^*(\x,\y))}/{\tau}}}}}}-
\mathbb{E}_{(\x,\y) \sim p}\zkh{{{{{s(\x,\y)}{}}}}}, 
\] which, together with \eqref{20260426022657}, yields 
\[
&\L(s)-\L^*
\\&=
\mathbb{E}_{\x}\zkh{\tau \log{{\mathbb{E}_{\y \sim p^+_\x}{\exp\ykh{{(s(\x,\y)-s^*(\x,\y))}/{\tau}}}}}}-
\mathbb{E}_{(\x,\y) \sim p}\zkh{{{{{s(\x,\y)}{}}}}}
+\mathbb{E}_{(\x,\y) \sim p}\zkh{s^*(\x,\y)}
\\&=
\mathbb{E}_{\x}\zkh{\tau \log{{\mathbb{E}_{\y \sim p(\cdot\mid\x)}{\exp\ykh{{(s(\x,\y)-s^*(\x,\y))}/{\tau}}}}}}-
\mathbb{E}_{(\x,\y) \sim p}\zkh{{{{{s(\x,\y)-s^*(\x,\y)}}}}}. 
\] This proves the desired result. 
\end{proof}

\begin{lemma}\label{20260426091356}
Let Assumption \ref{bounded risk 20260421190226} hold. Then for any $\e\in(0,\infty)$ there exists a compactly supported continuous  function $u:\R^{d_1}\times\R^{d_2}\to \R$  such that 
\[
\mathcal{L}(u)-\mathcal{L}^*\leq \e/2. 
\] 

\end{lemma}

\begin{proof}
For simplicity, we denote $p^+_\x:=p(\;\cdot\mid\x)$. 
 Fix an arbitrary $\e\in(0,\infty)$. 
Let $Q$ be the probability measure on $\R^{d_1}\times \R^{d_2}$ defined as 
\[
Q(A):=\E_{(\x,\y)\sim p}\zkh{\mathbb{I}\hkh{(\x,\y)\in A}},\;\forall\; \text{ measurable } \;A\subset \R^{d_1}\times \R^{d_2}. 
\] In other words, $Q$ is the zero-extension of the joint distribution of the data $(\x,\y)\sim p$ to $\R^{d_1}\times \R^{d_2}$. 
Denote  $s^*(\x,\y):=\tau\log\frac{p^+_{\x}(\y)}{p_\Y(\y)}$  and 
\[
v_0(\x,\y):= \begin{cases}
s^*(\x,\y), &\;\text{ if }(\x,\y)\in\X\times\Y,\\
0, &\;\text{ if }(\x,\y)\in  \ykh{\R^{d_1}\times\R^{d_2}}\setminus \ykh{\X\times\Y}. 
\end{cases}
\]
By Assumption \ref{bounded risk 20260421190226}, $v_0\in L^1(Q)$, and there exists a sufficiently large number $M\in(1,\infty)$ such that 
\begin{align}\label{20260426053424}
\exp(-M/\tau)+\int \abs{v_0(\x,\y)}\cdot\mathbb{I}\hkh{\abs{v_0(\x,\y)}>M} dQ(\x,\y)\leq \frac{\e}{10\tau+10}
\end{align} 
Let \[
T_M:\R\to[-M,M], t\mapsto \begin{cases}
t, &\;\text{ if }-M\leq t\leq M,\\
M, &\;\text{ if }t >M,\\
-M, &\;\text{ if }t<-M, 
\end{cases}
\] and $v_1=T_M\circ v_0$. Then it follows from \eqref{20260426053424} that 
\[
\norm{v_1-v_0}_{L^1(Q)}\leq \int \abs{v_0(\x,\y)}\cdot\mathbb{I}\hkh{\abs{v_0(\x,\y)}>M} dQ(\x,\y)<\frac{\e}{10\tau+10}<\frac{\e}{10}. 
\] Note that 
\[
&\int \exp\ykh{\ykh{M-v_0(\x,\y)}/\tau}dQ(\x,\y)=\exp(M/\tau)\E_{\x}\E_{\y\sim p_{\x}^+}\zkh{\exp\ykh{-v_0(\x,\y)/\tau}}
\\&=\exp(M/\tau)\E_{\x}\E_{\y\sim p_{\x}^+}{\frac{p_{\Y}(\y)}{p_\x^{+}(\y)}}
=
\exp(M/\tau)\E_{\x}\zkh{\int{\frac{p_{\Y}(\y)}{p_\x^{+}(\y)}} \cdot p_{\x}^+(\y)d\y}
\\&=\exp(M/\tau)\E_{\x}\zkh{\int  p_{\Y}(\y)d\y}
=\exp(M/\tau)\E_{\x}\zkh{1}=\exp(M/\tau)<\infty. 
\] By the absolute continuity of the integral, there exists a number $\delta\in (0,1)$ such that for any measurable set $B$ with $Q(B)\leq \delta$ there holds $\int_B \exp\ykh{\ykh{M-v_0(\x,\y)}/\tau}dQ(\x,\y)<\frac{\e}{11\tau}$. 
According to Theorem 7.8  of \cite{folland1999real}, $Q$ is a Radon measure on $\R^{d_1}\times \R^{d_2}$. Let ${\mathcal{C}_{\mathtt{c}}(\R^{d_1}\times \R^{d_2})}$ be the space of all compactly supported continuous functions from $\mathbb{R}^{d_1}\times\mathbb{R}^{d_2}$ to $\mathbb{R}$. Then it follows from Proposition 7.9 of \cite{folland1999real} that  $\mathcal{C}_{\mathtt{c}}(\R^{d_1}\times \R^{d_2})$ is dense in $L^1(Q)$. 
Obviously, ${v_1\in L^1(Q)}$. Thus 
 there exists a function $v_2\in \mathcal{C}_{\mathtt{c}}(\R^{d_1}\times \R^{d_2})$ such that $\norm{v_2-v_1}_{L^1(Q)}< \frac{{\e\delta}}{13}$. Take $v_3:=T_M\circ v_2$. Note that $\sup_{x\neq y}\abs{\frac{T_M(x)-T_M(y)}{x-y}}=1$ and  $T_M\circ T_M=T_M$, meaning that \[T_M\circ v_1=T_M\circ T_M\circ v_0=T_M\circ v_0=v_1. \] Therefore,  we have that 
 \[
 &\norm{v_3-v_1}_{L^1(Q)}=\norm{T_M\circ v_2-T_M\circ v_1}_{L^1(Q)}\leq \sup_{x\neq y}\abs{\frac{T_M(x)-T_M(y)}{x-y}}\cdot \norm{v_2-v_1}_{L^1(Q)}
 \\&\leq \norm{v_2-v_1}_{L^1(Q)}<{\frac{\e\delta}{13}}\leq\frac{\e}{13}. 
 \] Take $A:=\set{(\x,\y)\in \R^{d_1}\times \R^{d_2}}{\abs{v_3(\x,\y)-v_1(\x,\y)}>\frac{\e}{12}}$. Then it follows from Markov's inequality that 
 \[Q(A)\leq  \frac{\norm{v_3-v_1}_{L^1(Q)}}{\e/12}\leq \frac{{{\e\delta}/{13}}}{\e/12}=\delta\cdot\frac{12}{13}<\delta. 
 \] Thus we have \[\int_A \exp\ykh{\ykh{M-v_0(\x,\y)}/\tau}dQ(\x,\y)<\frac{\e}{11\tau}. \] Denote
 \[
& J_{\x}:=\E_{\y\sim p^{+}_\x}\zkh{\exp\bigykh{(v_3(\x,\y)-v_0(\x,\y))/\tau}\cdot\mathbb{I}\hkh{(\x,\y)\notin A}}, \\
&  I_{\x}:=\E_{\y\sim p^{+}_\x}\zkh{\exp\bigykh{(v_3(\x,\y)-v_0(\x,\y))/\tau}\cdot\mathbb{I}\hkh{(\x,\y)\in A}}, 
\\
 &H_{\x}:=\exp\ykh{\tfrac{\e}{12\tau}}\E_{\y\sim p^{+}_\x}\zkh{\exp\bigykh{(M-v_0(\x,\y))/\tau}\cdot\mathbb{I}\hkh{(\x,\y)\in A}}. 
 \] Then we have 
 \[
 &I_{\x}\leq \E_{\y\sim p^{+}_\x}\zkh{\exp\bigykh{(M-v_0(\x,\y))/\tau}\cdot\mathbb{I}\hkh{(\x,\y)\in A}}
 \\&\leq \exp\ykh{\tfrac{\e}{12\tau}}\E_{\y\sim p^{+}_\x}\zkh{\exp\bigykh{(M-v_0(\x,\y))/\tau}\cdot\mathbb{I}\hkh{(\x,\y)\in A}}=H_\x, 
 \] and 
 \[
 & \exp\ykh{-\tfrac{\e}{12\tau}}J_\x=\exp\ykh{-\tfrac{\e}{12\tau}}
 \E_{\y\sim p^{+}_\x}\zkh{\exp\bigykh{(v_3(\x,\y)-v_0(\x,\y))/\tau}\cdot\mathbb{I}\hkh{(\x,\y)\notin A}}
 \\&\leq\exp\ykh{-\tfrac{\e}{12\tau}}\E_{\y\sim p^{+}_\x}\zkh{\exp\bigykh{(v_1(\x,\y)+\tfrac{\e}{12}-v_0(\x,\y))/\tau}\cdot\mathbb{I}\hkh{(\x,\y)\notin A}}
 \\&=
\E_{\y\sim p^{+}_\x}\zkh{\exp\bigykh{(T_M(v_0(\x,\y))-v_0(\x,\y))/\tau}\cdot\mathbb{I}\hkh{(\x,\y)\notin A}} 
\\&\leq \E_{\y\sim p^{+}_\x}\zkh{\exp\bigykh{(T_M(v_0(\x,\y))-v_0(\x,\y))/\tau}} 
\\&= \E_{\y\sim p^{+}_\x}\zkh{\exp\bigykh{(T_M(v_0(\x,\y))-v_0(\x,\y))/\tau}\cdot\mathbb{I}\hkh{v_0(\x,\y)\geq -M}} 
\\&\;\;\;\;\;\;\;+\E_{\y\sim p^{+}_\x}\zkh{\exp\bigykh{(T_M(v_0(\x,\y))-v_0(\x,\y))/\tau}\cdot\mathbb{I}\hkh{v_0(\x,\y)< -M}} 
\\&\leq \E_{\y\sim p^{+}_\x}\zkh{\exp\bigykh{(v_0(\x,\y)-v_0(\x,\y))/\tau}\cdot\mathbb{I}\hkh{v_0(\x,\y)\geq -M}} 
\\&\;\;\;\;\;\;\;+\E_{\y\sim p^{+}_\x}\zkh{\exp\bigykh{(-M-v_0(\x,\y))/\tau}\cdot\mathbb{I}\hkh{v_0(\x,\y)< -M}} 
\\&\leq \E_{\y\sim p^{+}_\x}\zkh{\exp\bigykh{(v_0(\x,\y)-v_0(\x,\y))/\tau}} 
+\E_{\y\sim p^{+}_\x}\zkh{\exp\bigykh{(-M-v_0(\x,\y))/\tau}} 
\\&=
\E_{\y\sim p^{+}_\x}\zkh{\exp\ykh{0}} 
+\exp(-M/\tau)\E_{\y\sim p^{+}_\x}\zkh{\frac{p_{\Y}(\y)}{p_\x^{+}(\y)}}
\\&=1+\exp(-M/\tau)\int{\frac{p_{\Y}(\y)}{p_\x^{+}(\y)}}\cdot p_\x^{+}(\y) d\y=1+\exp(-M/\tau)\int p_\Y(\y) d\y
\\&=1+\exp(-M/\tau)\leq 1+\tfrac{\e}{10\tau+10}
 \]
for any $x\in \X$. Take $u=v_3$. We can then use Lemma \ref{20260426083507} to deduce that 
\[
&\L(u)-\L^*=\L(v_3)-\L^*=\L\bigykh{v_3\big|_{\X\times \Y}}-\L^*
\\&=\tau\E_\x\zkh{\log\ykh{\E_{\y\sim p^{+}_\x}\exp\bigykh{(v_3(\x,\y)-s^*(\x,\y))/\tau}}}-\E_{(\x,\y)\sim p}\zkh{v_3(\x,\y)-s^*(\x,\y)}
\\&=\tau\E_\x\zkh{\log\ykh{\E_{\y\sim p^+_\x}\exp\bigykh{(v_3(\x,\y)-v_0(\x,\y))/\tau}}}-\E_{(\x,\y)\sim p}\zkh{v_3(\x,\y)-v_0(\x,\y)}
\\&=\tau\E_\x\zkh{\log\ykh{I_\x+J_\x}}-\int\ykh{v_3(\x,\y)-v_0(\x,\y)}dQ(\x,\y)
\\&\leq 
\tau\E_\x\zkh{\log\ykh{H_\x+\exp\ykh{\tfrac{\e}{12\tau}}\cdot(1+\tfrac{\e}{10\tau+10})}}+\norm{v_3-v_0}_{L^1(Q)}
\\&=\tau\E_\x\zkh{\log\ykh{\E_{\y\sim p^{+}_\x}\zkh{\exp\bigykh{(M-v_0(\x,\y))/\tau}\cdot\mathbb{I}\hkh{(\x,\y)\in A}}+1+\tfrac{\e}{10\tau+10}}}\\&\;\;\;\;\;\;\;\;\;\;\;\;\;\;+\tfrac{\e}{12}+\norm{v_3-v_0}_{L^1(Q)}
\\&\leq\tau\E_\x\zkh{\E_{\y\sim p^{+}_\x}\zkh{\exp\bigykh{(M-v_0(\x,\y))/\tau}\cdot\mathbb{I}\hkh{(\x,\y)\in A}}+\tfrac{\e}{10\tau+10}}\\&\;\;\;\;\;\;\;\;\;\;\;\;\;\;+\tfrac{\e}{12}+\norm{v_3-v_0}_{L^1(Q)}
\\&=\tau\int_A{\exp\bigykh{(M-v_0(\x,\y))/\tau}}dQ(\x,\y)+\tfrac{\e\tau}{10\tau+10}+\tfrac{\e}{12}+\norm{v_3-v_0}_{L^1(Q)}
\\&\leq \tau\cdot \tfrac{\e}{11\tau}+\tfrac{\e\tau}{10\tau+10}+\tfrac{\e}{12}+\norm{v_3-v_1}_{L^1(Q)}+\norm{v_0-v_1}_{L^1(Q)}
\\&\leq  \tfrac{\e}{11}+\tfrac{\e}{10}+\tfrac{\e}{12}+\tfrac{\e}{13}+\tfrac{\e}{10}\leq \e/2. 
\] This proves the desired result. 
\end{proof}

\begin{proof}[Proof of Theorem \ref{univApproxDNNs20260425200219}]
For simplicity, we denote $p^+_\x:=p(\;\cdot\;|\x)$. 
 Fix arbitrary positive numbers   $\e\in(0,\infty)$,  $L_1\in\mb{N}\cap(1,\infty)$, and  ${L_2\in \mb{N}\cap  (1,\infty)}$. 
According to Lemma \ref{20260426091356}, there exists a compactly supported continuous function $u:\R^{d_1}\times\R^{d_2}\to\R$, such that ${\L(u)-\L^*\leq \e/2}$.  Then there exists   $B\in\mb{N}$ such that the support of $u$ is contained in $[-B,B]^{d_1+d_2}$ and $\sup\set{\abs{u(\x,\y)}}{(\x,\y)\in \R^{d_1}\times \R^{d_2}}<B$.  
For $i\in\hkh{1,2}$, let $A_i$  denote the space of all continuous functions from $[-3B,3B]^{d_i}$ to $\R$ and $\norm{\cdot}_{A_i}$ denote the uniform norm on $A_i$, that is, $\norm{w}_{A_i}=\sup_{\x\in [-3B,3B]^{d_i}}\abs{w(\x)}$.   Take
\[
\mathcal{A}:=\set{\sum_{j=1}^k h_j(\x)v_j(\y)}{k\in\mb{N}, \text{ and }h_j\in  A_1, v_j\in A_2 ,\;\forall\;j\in\hkh{1,2,\ldots, k}}. 
\] 
Obviously, $\mc{A}$ is an algebra of real continuous functions on the compact set $[-3B,3B]^{d_1+d_2}$. Moreover, $\mc{A}$ separates points on  $[-3B,3B]^{d_1+d_2}$ and vanishes at no point of $[-3B,3B]^{d_1+d_2}$. 
Then it follows from the  Stone--Weierstrass theorem (cf. Theorem 7.32  of \cite{rudin1976principles}) that 
for any continuous function $w$ from $[-3B,3B]^{d_1+d_2}$ to $\R$ there exists a sequence in $\mc{A}$  that converges uniformly to $w$ on $[-3B,3B]^{d_1+d_2}$. In particular, for any positive integer $q$, there exists
\[
H_q(\x,\y)=\sum_{j=1}^{k_q}h_{j,q}(\x)v_{j,q}(\y)
\] with $k_q\in\mb{N}$, $h_{j,q}\in A_1$, $v_{j,q}\in A_2$ ($j=1,2,\ldots, k_q$) such that 
\begin{align}\label{20260426134337}
\sup\set{\abs{H_q(\x,\y)-u(\x,\y)}}{\x\in[-3B,3B]^{d_1}, \y\in [-3B,3B]^{d_2}}\leq 1/q. 
\end{align}It is well-known that ReLU Neural networks with one hidden layer have the universal approximation property on compact sets (see, e.g., Theorem 1 of  \cite{leshno1993multilayer}). Therefore, for any $q\in\mb{N}$  there exists a sufficiently large positive integer $a_q$ and neural networks 
\begin{align}\label{20260426230823}
\widetilde{h}_{j,q}\in \fdnn(d_1,1, a_q,1),\;\;\widetilde{v}_{j,q}\in \fdnn(d_2,1, a_q,1)
\end{align} such that 
\[
&\norm{\widetilde{h}_{j,q}-h_{j,q}}_{A_1}\leq \frac{1}{2qk_q}\cdot\frac{1}{1+\norm{v_{j,q}}_{A_2}}, \;\;\norm{\widetilde{v}_{j,q}-v_{j,q}}_{A_2}\leq \frac{1}{2qk_q}\cdot\frac{1}{1+\norm{h_{j,q}}_{A_1}}, 
\] for all $j\in\hkh{1,2,\ldots, k_q}$. Take 
\begin{align}\label{20260426234319}
\widetilde{H}_q(\x,\y)=\sum_{j=1}^{k_q}\widetilde{h}_{j,q}(\x)\widetilde{v}_{j,q}(\y). 
\end{align}
Then we have that 
\[
&\abs{{H}_q(\x,\y)-\widetilde{H}_q(\x,\y)}=\abs{\sum_{j=1}^{k_q}\ykh{({h}_{j,q}(\x)-\widetilde{h}_{j,q}(\x)){v}_{j,q}(\y)+\widetilde{h}_{j,q}(\x)({v}_{j,q}(\y)-\widetilde{v}_{j,q}(\y))}}
\\&\leq \sum_{j=1}^{k_q}\ykh{\abs{{h}_{j,q}(\x)-\widetilde{h}_{j,q}(\x)}\abs{{v}_{j,q}(\y)}+\abs{\widetilde{h}_{j,q}(\x)}\abs{{v}_{j,q}(\y)-\widetilde{v}_{j,q}(\y)}}
\\&\leq 
\sum_{j=1}^{k_q}\ykh{\tfrac{1}{2qk_q}\cdot\frac{\norm{{v}_{j,q}}_{A_2}}{1+\norm{v_{j,q}}_{A_2}}+\tfrac{1}{2qk_q}\cdot\frac{\norm{\widetilde{h}_{j,q}}_{A_1}}{1+\norm{h_{j,q}}_{A_1}}}\\&\leq \frac{1}{2qk_q}\cdot\sum_{j=1}^{k_q}\ykh{1+\frac{\norm{{h}_{j,q}}_{A_1}+\norm{{h}_{j,q}-\widetilde{h}_{j,q}}_{A_1}}{1+\norm{h_{j,q}}_{A_1}}}
\leq
\frac{1}{2qk_q}\cdot\sum_{j=1}^{k_q}\ykh{1+\frac{\norm{{h}_{j,q}}_{A_1}+1}{1+\norm{h_{j,q}}_{A_1}}}\\&=\frac{1}{q},\;\forall\;\x\in [-3B,3B]^{d_1}, \;\y\in [-3B,3B]^{d_2}, 
\] which, together with \eqref{20260426134337}, yields 
\begin{align}\label{20260426233911}
\abs{{u}(\x,\y)-\widetilde{H}_q(\x,\y)}\leq \frac{2}{q},\;\forall\;\x\in [-3B,3B]^{d_1}, \;\y\in [-3B,3B]^{d_2}. 
\end{align} For $i=1,2$ we denote 
\[
T_{i}:\R^{d_i}&\to [-2B,2B]^{d_i}, \\
\begin{pmatrix}
x_1\\x_2\\
\vdots\\
x_{d_i}
\end{pmatrix}&\mapsto \begin{pmatrix}
\sigma(x_1+2B)-\sigma(x_1-2B)-\sigma(0x_1+2B)\\\sigma(x_2+2B)-\sigma(x_2-2B)-\sigma(0x_2+2B)\\
\vdots\\
\sigma(x_{d_i}+2B)-\sigma(x_{d_i}-2B)-\sigma(0x_{d_i}+2B)
\end{pmatrix}. 
\] It is easy to verify that for any $(\x,\y)\in\R^{d_1}\times \R^{d_2}$,  if $(\x,\y)\in[-2B,2B]^{d_1+d_2}$, then $(T_1(\x), T_2(\y))=(\x,\y)$, meaning that $u(\x,\y)=u(T_1(\x),T_2(\y))$; if, otherwise $(\x,\y)\notin [-2B,2B]^{d_1+d_2}$, then $(T_1(\x), T_2(\y))\notin [-B,B]^{d_1+d_2}$, meaning that $u(T_1(\x),T_2(\y))=0=u(\x,\y)$. Therefore, we have 
\begin{align}\label{20260426234133}
u(T_1(\x),T_2(\y))=u(\x,\y),\;\forall\; (\x,\y)\in\R^{d_1}\times\R^{d_2}. 
\end{align}
Define 
\begin{align}\label{20260426230653}
\widetilde{S}_q:\R^{d_1}\times\R^{d_2}\to\R, \;(\x,\y)\mapsto\sum_{j=1}^{k_q}\widetilde{h}_{j,q}(T_1(\x))\widetilde{v}_{j,q}(T_2(\y)).
\end{align} Obviously, \begin{align}\label{20260426230709}
T_i\in \fdnn(d_i, 1, 3d_i, d_i), \;\forall\;i\in\hkh{1,2}. \end{align} 
It follows from \eqref{20260426233911} that 
\[
&\abs{u(\x,\y)-\widetilde{S}_q(\x,\y)}\xlongequal{\because \eqref{20260426234319},  \eqref{20260426230653}} \abs{u(\x, \y)-\widetilde{H}_q(T_1(\x),T_2(\y))}\\&\xlongequal{\because\eqref{20260426234133}} \abs{u(T_1(\x), T_2(\y))-\widetilde{H}_q(T_1(\x),T_2(\y))}
\leq\frac{2}{q},\;\forall\; (\x,\y)\in\R^{d_1}\times\R^{d_2}, \;\forall\;q\in\mb{N}.  
\] Thus 
\begin{align}\label{20260427001950}
\sup_{(\x,\y)\in\R^{d_1}\times\R^{d_2}}{\abs{u(\x,\y)-\widetilde{S}_q(\x,\y)}}{}\leq \frac{2}{q},\;\forall\;q\in\mb{N}, 
\end{align} and 
\begin{align}\label{20260427000711}
\sup_{(\x,\y)\in\R^{d_1}\times\R^{d_2}}{\abs{\widetilde{S}_q(\x,\y)}}\leq\frac{2}{q}+\sup_{(\x,\y)\in\R^{d_1}\times\R^{d_2}}{\abs{u(\x,\y)}}\leq 2+B,\;\forall\;q\in\mb{N}. 
\end{align}
Combining \eqref{20260426230823},  \eqref{20260426230653},  \eqref{20260426230709}, and \eqref{20260427000711},  we  get 
\begin{align}\label{20260427014050}
\widetilde{S}_q\in\fscoring(2,(a_q+3d_1)\cdot k_q,2, (a_q+3d_2)\cdot k_q,k_q),\;\forall\; q\in\mb{N}. 
\end{align} From \eqref{20260427001950} we see $\hkh{\widetilde{S}_q}_{q=1}^\infty$ converges uniformly to $u$. We can then use Arzel\`a's  Bounded Convergence Theorem to deduce that 
\[
&\L({u})=\mathbb{E}_{\x}\mathbb{E}_{\y \sim p^+_\x}\zkh{\tau \log\mathbb{E}_{\y' \sim p_\Y}\exp\ykh{\tfrac{u(\x,\y')-u(\x,\y)}{\tau}}}
\\&=\mathbb{E}_{\x}\mathbb{E}_{\y \sim p^+_\x}\zkh{\tau \log\lim_{q\to\infty}\mathbb{E}_{\y' \sim p_\Y}\exp\ykh{\tfrac{\widetilde{S}_q(\x,\y')-\widetilde{S}_q(\x,\y)}{\tau}}}
\\&=\mathbb{E}_{\x}\mathbb{E}_{\y \sim p^+_\x}\zkh{\lim_{q\to\infty}\tau \log\mathbb{E}_{\y' \sim p_\Y}\exp\ykh{\tfrac{\widetilde{S}_q(\x,\y')-\widetilde{S}_q(\x,\y)}{\tau}}}
\\&=\lim_{q\to\infty}\mathbb{E}_{\x}\mathbb{E}_{\y \sim p^+_\x}\zkh{\tau \log\mathbb{E}_{\y' \sim p_\Y}\exp\ykh{\tfrac{\widetilde{S}_q(\x,\y')-\widetilde{S}_q(\x,\y)}{\tau}}}
=\lim_{q\to\infty} \L(\widetilde{S}_q). 
\] Hence  there exists a sufficiently large  $b\in\mb{N}$ such that $\abs{\L(\widetilde{S}_b)-\L(u)}\leq \e/3$.  Thus 
\[
\L(\widetilde{S}_b)-\L^*\leq \L(u)-\L^*+\abs{\L(\widetilde{S}_b)-\L(u)}\leq \e/2+\e/3<\e, 
\]  and it follows from \eqref{20260427014050} that  \[&\widetilde{S}_b\in\fscoring(2,(a_b+3d_1)\cdot k_b,2, (a_b+3d_2)\cdot k_b,k_b)\\&\subset \fscoring(L_1,(a_b+3d_1)\cdot k_b,L_2, (a_b+3d_2)\cdot k_b,k_b). 
\] Therefore, taking $N_1=(a_b+3d_1)\cdot k_b$, $N_2=(a_b+3d_2)\cdot k_b$, $D=k_b$,  and $\tilde{s}=\widetilde{S}_b$ yields the desired result. 
 \end{proof}

\section{Proof of Theorem \ref{4p3thm20260427130502}}\label{aE20260427130617}

We first present a result regarding the uniform approximation error bound for Hölder continuous functions using ReLU neural networks, which is a special case of Theorem 1.1 of \cite{shen2020DeepNe}. 

\begin{lemma}\label{20260428014457}
Let $d\in\mb{N}$, $R\in(0,\infty)$, $\beta\in(0,1]$,  and $B\in(0,\infty)$. Then there exist constants $\zzhcyi=\zzhcyi(d)\in\mb{N}$ and $\zzhcer=\zzhcer(d,R,\beta, B)\in(0,\infty)$ such that for any $L\in\mb{N}$, $N\in\mb{N}$ and  any $f:[-R,R]^d\to \R$ with $\norm{f}_{\beta,[-R,R]^d}\leq B$ there exists a neural network $\tilde{f}\in\fdnn(d,\zzhcyi L, \zzhcyi N, 1)$ such that 
\begin{align}\label{20260427145503}
\sup_{\y\in[-R,R]^d}\abs{f(\y)-\tilde{f}(\y)}\leq \zzhcer\cdot \ykh{\tfrac{1}{L^2N^2}}^{\beta/d}. 
\end{align}
\end{lemma}
\begin{proof}
Let $L$, $N$ be arbitrary positive integers, and let $f:[-R,R]^d\to \R$ be an arbitrary function satisfying $\norm{f}_{\beta,[-R,R]^d}\leq B$. Denote $\mathbf{1}:=(1,1,\ldots,1)\in\mb{R}^d$, and define 
\[
g:[0,1]^d\to\mb{R},\;\x\mapsto f(2R\x-R\mathbf{1}). 
\] Obviously, the modulus of continuity of $g$ satisfies 
\[
\omega_g(r)&:=\sup\set{\abs{g(\x)-g(\y)}}{\x,\y\in[0,1]^d,\norm{\x-\y}_2\leq r}
\\&=\sup\set{\abs{f(2R\x-R\mathbf{1})-f(2R\y-R\mathbf{1})}}{\x,\y\in[0,1]^d,\norm{\x-\y}_2\leq r}
\\&\leq 
\sup\set{B\cdot\norm{(2R\x-R\mathbf{1})-(2R\y-R\mathbf{1})}_\infty^\beta}{\x,\y\in[0,1]^d,\norm{\x-\y}_2\leq r}
\\&
=
\sup\set{B\cdot(2R)^\beta\norm{\x-\y}_\infty^\beta}{\x,\y\in[0,1]^d,\norm{\x-\y}_2\leq r}
\\&\leq
{B\cdot(2R)^\beta\cdot r^\beta},\;\forall\;r\in(0,\infty). 
\] Then it follows from  Theorem 1.1 of \cite{shen2020DeepNe} that there exists 
\[
\tilde{g}\in\fdnn(d,12L+14+2d, 3^{d+3}\max\hkh{dN,N+1}+1, 1)
\] such that 
\[
\sup_{\x\in[0,1]^d}\abs{g(\x)-\tilde{g}(\x)}\leq 19\sqrt{d}\cdot B\cdot (2R)^\beta\cdot \ykh{N^{-2/d}L^{-2/d}}^\beta. 
\] Take $\zzhcyi:=3^{d+3}\cdot (1+d)+1$, $\zzhcer:=19\sqrt{d}\cdot B\cdot (2R)^\beta$,  and  $\tilde{f}(\y):=\tilde{g}(\y/(2R)+\frac12\mathbf{1})$. Then we have 
\[
\tilde{f}&\in\fdnn(d,12L+14+2d, 3^{d+3}\max\hkh{dN,N+1}+1, 1)
\\&\subset \fdnn(d,(26+2d)\cdot L, 3^{d+3}\cdot(d+1)\cdot N+1, 1)
 \subset \fdnn(d,\zzhcyi L, \zzhcyi N, 1), 
\] and \[
&\sup_{\y\in[-R,R]^d}\abs{f(\y)-\tilde{f}(\y)}=\sup_{\x\in[0,1]^d}\abs{f(2R\x-R\mathbf{1})-\tilde{f}(2R\x-R\mathbf{1})}\\&=\sup_{\x\in[0,1]^d}\abs{g(\x)-\tilde{g}(\x)}\leq 19\sqrt{d}\cdot B\cdot (2R)^\beta\cdot \ykh{N^{-2/d}L^{-2/d}}^\beta=\zzhcer\cdot \ykh{\tfrac{1}{L^2N^2}}^{\beta/d}. 
\] This completes the proof. 
\end{proof}

\begin{lemma}\label{20260428072356}
Let Assumption \ref{bounded risk 20260421190226} hold. 
Denote $s^*(\x,\y):=\tau\log\frac{p(\y\mid\x)}{p_\Y(\y)}$.
Then for any bounded measurable function $u:\X\times\Y\to \R$ there holds 
\[
\L(u)-\L^*\leq \E_\x\zkh{\sup_{\y\in\Y}\tfrac{\abs{u(\x,\y)-s^*(\x,\y)}^2}{2\tau}}. 
\]
\end{lemma}

\begin{proof}For simplicity, we denote $p^+_\x:=p(\;\cdot\mid\x)$. 
Let $u$ be an arbitrary bounded measurable function from $\X\times\Y$ to $\R$. Define  \begin{align}\label{20260427211351}
a(\x):=\sup_{\y\in\Y}\abs{u(\x,\y)-s^*(\x,\y)},\;\forall\;\x\in \X, \end{align} and
\[H(\x,\y):=u(\x,\y)-s^*(\x,\y),\;\forall\; \x\in \X, \;\y\in\Y. 
\]  Then define 
 \[
 v(\x):=
 \begin{cases}
 \E_{\y\sim p_\x^+}{H(\x,\y)}, &\text{ if }\E_{\y\sim p_\x^+}{\abs{H(\x,\y)}<\infty},\\
 0, &\text{ if }\E_{\y\sim p_\x^+}{\abs{H(\x,\y)}=\infty}
 \end{cases} 
\] for all $\x\in\X$, and $A:=\set{\x\in \X}{a(\x)<\infty}$.  It follows from Assumption \ref{bounded risk 20260421190226} and the boundedness of $u$ that
\[\mathbb{E}_{\x}\abs{v(\x)}\leq \mathbb{E}_{\x} \E_{\y\sim p_\x^+}\abs{H(\x,\y)}
<\infty,\] which means that 
$\mathbb{E}_{\y \sim p^+_\x}\abs{H(\x,\y)}<\infty$ for almost all $\x\in\X$, and 
\begin{align}\label{20260427210738}
\mathbb{E}_{\x}\mathbb{E}_{\y \sim p^+_\x}H(\x,\y)=\mathbb{E}_{\x}v(\x). 
\end{align}
According to Hoeffding's lemma  (see, e.g., Lemma 2.2 \cite{boucheron2013concentration}), 
\[
&-v(\x)\cdot\lambda +\log\E_{\y\sim p_\x^+}\exp\bigykh{H(\x,\y)\cdot\lambda}=\log\E_{\y\sim p_\x^+}\exp\bigykh{(H(\x,\y)-v(\x))\cdot\lambda}\\&\leq \frac{\lambda^2\cdot(a(\x)-v(\x)-(-a(\x)-v(\x)))^2}{8}
=\frac{\lambda^2\cdot a(\x)^2}{2},\;\forall\;\lambda\in\R, \;\forall\;\x\in A. 
\] Taking $\lambda=1/\tau$ above, we get  
\begin{align}\label{20260427202719}
-v(\x) +\tau\log\E_{\y\sim p_\x^+}\exp\bigykh{H(\x,\y)/\tau}\leq \frac{ a(\x)^2}{2\tau},\;\forall\;\x\in A. 
\end{align} Furthermore, 
\begin{align}\label{20260427202725}
-v(\x) +\tau\log\E_{\y\sim p_\x^+}\exp\bigykh{H(\x,\y)/\tau}\leq+\infty= \frac{ a(\x)^2}{2\tau},\;\forall\;\x\in \X\setminus A. 
\end{align}  Combining \eqref{20260427202719} and \eqref{20260427202725}, we obtain 
\begin{align*}
-v(\x) +\tau\log\E_{\y\sim p_\x^+}\exp\bigykh{H(\x,\y)/\tau}\leq \frac{ a(\x)^2}{2\tau},\;\forall\;\x\in \X. 
\end{align*}
Then it follows from Lemma \ref{20260426083507}
\[
&\L(u)-\L^*=\mathbb{E}_{\x}\zkh{\tau \log{{\mathbb{E}_{\y \sim p^+_\x}{\exp\ykh{{(u(\x,\y)-s^*(\x,\y))}/{\tau}}}}}}-
\mathbb{E}_{\x}\mathbb{E}_{\y \sim p^+_\x}\zkh{{{{{u(\x,\y)-s^*(\x,\y)}}}}}
\\&=
\mathbb{E}_{\x}\zkh{\tau \log{{\mathbb{E}_{\y \sim p^+_\x}{\exp\ykh{{H(\x,\y)}/{\tau}}}}}}-
\mathbb{E}_{\x}\mathbb{E}_{\y \sim p^+_\x}\zkh{{{{{H(\x,\y)}}}}}
\\&\xlongequal{\because\eqref{20260427210738}}
\mathbb{E}_{\x}\zkh{\tau \log{{\mathbb{E}_{\y \sim p^+_\x}{\exp\ykh{{H(\x,\y)}/{\tau}}}}}}-
\mathbb{E}_{\x}\zkh{{{{{v(\x)}}}}}
\\&=
\mathbb{E}_{\x}\zkh{-v(\x)+\tau \log{{\mathbb{E}_{\y \sim p^+_\x}{\exp\ykh{{H(\x,\y)}/{\tau}}}}}}\\&\leq 
\mathbb{E}_{\x}\zkh{\frac{ a(\x)^2}{2\tau}}
\xlongequal{\because\eqref{20260427211351}}
\E_\x\zkh{\sup_{\y\in\Y}\tfrac{\abs{u(\x,\y)-s^*(\x,\y)}^2}{2\tau}}. 
\] This proves the desired result. 
\end{proof}

\begin{lemma}\label{20260428013219}
Let $\beta\in(0,1]$,  $d\in\mb{N}$, $\varnothing\subsetneqq U\subset V\subset\R^d$, and $f$ be a function from $U$ to $\R$. Then there exists a function $g:V\to\R$ such that 
$\norm{g}_{\beta,V}\leq \norm{f}_{\beta,U}$ and 
\[
f(\x)=g(\x),\;\forall\; \x\in U. 
\]
\end{lemma}
\begin{proof} The case $\norm{f}_{\beta,U}=\infty$ is trivial. In what follows, we assume $\norm{f}_{\beta,U}<\infty$. 
Let \[\rho:\R^d\times\R^d\to[0,\infty),\;(\x,\y)\mapsto \norm{\x-\y}_\infty^\beta, 
\] and $a:=\norm{f}_{\beta,U}$. Then it is easy to verify that $\rho$ is a metric on $\R^d$, and  \begin{align*}
\sup\set{\tfrac{\abs{f(\x)-f(\y)}}{\norm{\x-\y}^\beta_\infty}}{{U\ni \x\neq\y\in U}}\leq \norm{f}_{\beta,U}=a.\end{align*} Therefore,  $f$ is $a$-Lipschitz continuous on $U$ with respect to the metric $\rho$, that is, 
\[
\abs{f(\x)-f(\y)}\leq a\cdot \rho(\x,\y),\;\forall\; \x\in U, \;\y\in U. 
\]Then it follows from Theorem 1 of \cite{mcshane1934extension} that $f$ can be extended to $h:\R^d\to \R$ preserving the same Lipschitz condition, that is, 
\begin{align}\label{20260427231859}
h(\x)=f(\x),\;\forall\;\x\in U, 
\end{align} and 
\begin{align}\label{20260427232537}
\abs{h(\x)-h(\y)}\leq a\cdot \rho(\x,\y),\;\forall\; \x\in \R^d, \;\y\in \R^d. 
\end{align} Let 
$T:\R\to\R,\;t\mapsto \min\hkh{\max\hkh{t,-a}, a}$. Obviously, 
\begin{align}\label{20260427232019}
\abs{T(x)-T(y)}\leq \abs{x-y}, \forall\;x\in\R,\;\forall\; y\in\R,\end{align} and \begin{align}\label{20260427232055}
T(x)=x,\;\forall\;x\in[-a,a].\end{align} Note that $\abs{f(\x)}\leq \norm{f}_{\beta, U}=a$ for all  $\x\in U$. Therefore, 
\begin{align}\label{20260427235015}
T(h(\x))\xlongequal{\because\eqref{20260427231859}} T(f(\x))\xlongequal{\because\eqref{20260427232055}} f(\x),\;\forall\; \x\in U. 
\end{align}Combining \eqref{20260427232537} and \eqref{20260427232019}, we get 
\[
\abs{T(h(\x))-T(h(\y))}\leq \abs{h(\x)-h(\y)}\leq a\cdot\rho(\x,\y)=a\cdot \norm{\x-\y}_\infty^\beta,\;\forall\;\x,\y\in\R^d. 
\]Thus 
\[
\sup_{V\ni\x\neq \y\in V}{\tfrac{\abs{T(h(\x))-T(h(\y))}}{\norm{\x-\y}_\infty^\beta}}\leq a. 
\] Therefore, 
\[
&\norm{T\circ h}_{\beta, V}=\max\hkh{\sup_{\x\in V}\abs{T(h(\x))}, \sup_{V\ni\x\neq \y\in V}{\tfrac{\abs{T(h(\x))-T(h(\y))}}{\norm{\x-\y}_\infty^\beta}}}
\\&\leq \max\hkh{\sup_{\x\in V}\abs{T(h(\x))}, a} \leq \max\hkh{\sup_{t\in\R}\abs{T(t)}, a}=\max\hkh{a,a}=a=\norm{f}_{\beta,U}, 
\] which, together with \eqref{20260427235015}, implies that the function $g:V\to \R, \x\mapsto T(h(\x))$ has all the desired properties. This completes the proof. \end{proof}

\begin{lemma}\label{20260501042647}
Suppose $m_1\in\mb{N}$, $m_2\in\mb{N}$,  $R\in(0,\infty)$,  $\beta\in(0,1]$,   $B\in(0,\infty)$. 
Fix arbitrary positive numbers  $\theta_1$, $\theta_2$, $\theta_3$, $\theta_4$, $\theta_5$. 
Then for 
 sufficiently small $\e\in(0,\infty)$ and ${L_1\asymp \e^{-\theta_1}}$, $N_1 \asymp \e^{-\theta_2}$, $L_2 \asymp \e^{-\theta_3}$, $N_2 \asymp \e^{-\theta_4}$, $D\asymp \e^{-\theta_5}$ there holds  
\[&\sup_{\substack{h:[-R,R]^{m_1+m_2}\to \R\\{\norm{h}_{\beta, [-R,R]^{m_1+m_2}}\leq B}}}\inf\set{\sup_{\substack{\x\in[-R,R]^{m_1}\\\y\in[-R,R]^{m_2}}}\abs{f(\x)^\top g(\y)-h(\x,\y)}}{\begin{aligned}f\in \fdnn(m_1,L_1, N_1, D), \\g\in \fdnn(m_2,L_2, N_2, D)\end{aligned}}\\&\lesssim \abs{\e^{\theta_3+\theta_4-\min\hkh{\theta_2,\theta_4,\theta_5}}}^{2\beta/m_2}+\abs{\e^{\min\hkh{\theta_2,\theta_4, \theta_5}}}^{\beta/m_1}.
\]
\end{lemma}
\begin{proof}  Let $h:[-R,R]^{m_1+m_2}\to \R$ be an arbitrary function that satisfies ${\norm{h}_{\beta, [-R,R]^{m_1+m_2}}\leq B}$. 
Obviously,  \[\norm{h(\mathbf{a},\cdot)}_{\beta,[-R,R]^{m_2}}\leq \norm{h}_{\beta,[-R,R]^{m_1+m_2}}\leq B,\;\forall\;\mathbf{a}\in[-R,R]^{m_1}. 
\]
Then it follows from  Lemma \ref{20260428014457} that there exist constants  $\zzhcer=\zzhcer(m_2,R,\beta, B)\in(0,\infty)$ and $\zzhcyi=\zzhcyi(m_2)\in\mb{N}$ such that 
for any $L\in\mb{N}$, $N\in\mb{N}$ and  any $\mathbf{a}\in [-R,R]^{m_1}$ there exists a neural network \begin{align}\label{20260428051811}
\tilde{f}_{\mathbf{a},L,N}\in\fdnn(d,\zzhcyi L, \zzhcyi N, 1)\end{align} such that 
\begin{align}\label{20260428062157}
\sup_{\y\in[-R,R]^{m_2}}\abs{h(\mathbf{a},\y)-\tilde{f}_{\mathbf{a},L,N}(\y)}\leq \zzhcer\cdot \ykh{\tfrac{1}{L^2N^2}}^{\beta/m_2}. 
\end{align}
Let  \begin{align}\label{20260428054003}
G:=\flr{\ykh{\min\hkh{D^{1/m_1}, \ykh{\tfrac{N_1}{6m_1}}^{1/m_1}, \ykh{\tfrac{N_2}{2\zzhcyi}}^{1/m_1}}-1}\cdot\frac{1}{2}}\asymp \e^{-\min\hkh{\theta_2,\theta_4, \theta_5}/m_1}\end{align}
and  $N_3:=\flr{N_2/\ykh{\zzhcyi\cdot(2G+1)^{m_1}}}$. Obviously, 
\begin{align}\label{20260428054155}
2\leq \flr{N_2/\ykh{\zzhcyi\cdot(2G+1)^{m_1}}}=N_3\asymp\e^{\min\hkh{\theta_2,\theta_4,\theta_5}-\theta_4}. 
\end{align}We then take

\begin{align}\label{20260428052852}
L_3:=\sup\set{L\in\mb{N}\cap[2,\infty)}{\zzhcyi L\leq L_2}. 
\end{align} From $L_2\asymp \e^{-\theta_3}$ and $N_3\asymp\e^{\min\hkh{\theta_2,\theta_4,\theta_5}-\theta_4}$ we get 
$L_3\asymp \e^{-\theta_3}$. 
For any  $k\in \mb{Z}$, let  $l_k:\R\to[0,1]$ be the hat  function which  is linear on $[\tfrac{R\cdot (k-1)}{G},\tfrac{R\cdot k}{G})]$ and $[\tfrac{R\cdot k}{G},\tfrac{R\cdot (k+1)}{G})]$ respectively, and  satisfies
\begin{align}\label{20260428035101}
l_k(\tfrac{R\cdot k}{G})=1\geq 0=l_k(x),\;\forall\;x\in \R\setminus (\tfrac{R\cdot (k-1)}{G},\tfrac{R\cdot (k+1)}{G}). 
\end{align}
Obviously, $l_k$ can be expressed as 
 \begin{align}\label{20260428044647}
 l_k(x)=\ykh{\sigma(Gx-Rk+R)+\sigma(Gx-Rk-R)-2\sigma(Gx-Rk)}/R,\;\forall\;x\in\R, \end{align}  and $\hkh{l_k}_{k=-G}^G$ forms a partition of unity on $[R,R]$, that is, 
 \begin{align}\label{20260428034423}
 \sum_{k=-G}^Gl_k(x)=1,\;\forall\;x\in[-R,R]. 
 \end{align}
Let
 \[
 \Omega:=\set{\tfrac{Rk}{G}}{k\in[-G,G]\cap \mb{Z}}\subset [-R,R]. 
 \] Then for any $\mathbf{a}=(\tfrac{Rk_1}{G}, \tfrac{Rk_2}{G},\ldots, \tfrac{Rk_{m_1}}{G})\in \Omega^{m_1}$, define 
 \[{u}_{\mathbf{a}}:\R^{m_1}\to[0,1], \;\x=(x_1,\ldots,x_{m_1})\mapsto \prod_{j=1}^{m_1}l_{k_j}(x_j). 
 \] Then we have that   \begin{align}\label{20260428060604}
  u_{\mathbf{a}}(\x)\xlongequal{\because\eqref{20260428035101}}0,\;\forall\;\mathbf{a}\in\Omega^{m_1}, \forall\;\x\in\set{\z\in\R^{m_1}}{\norm{\z-\mathbf{a}}_\infty\geq R/G}, 
  \end{align} and 
 \[
 &\sum_{\mathbf{a}\in\Omega^{m_1}}u_{\mathbf{a}}(\x)=\sum_{k_1=-G}^{G}\sum_{k_2=-G}^G\cdots \sum_{k_{m_1}=-G}^G \prod_{j=1}^{m_1}l_{k_j}(x_j)
 \\&=\prod_{j=1}^{m_1}\sum_{k_j=-G}^Gl_{k_j}(x_j)
 \xlongequal{\because\eqref{20260428034423}} \prod_{j=1}^{m_1}1=1,\;\forall\;\x=(x_1,\ldots,x_{m_1})\in[-R,R]^{m_1}. 
 \] Thus 
 \begin{align}\label{20260428060044}
 \sum_{\mathbf{a}\in\Omega^{m_1}}u_{\mathbf{a}}(\x)=1,\;\forall\;\x\in[-R,R]^{m_1}. 
 \end{align} 
Let $L_4:=\flr{\frac{L_1-2}{\ceil{\log_2m_1}}}-5\asymp \e^{-\theta_1}$. 
According to Lemma A.3 of \cite{schmidt2020nonparametric}, there exists a neural network \begin{align}\label{20260428044723}
H\in\fdnn(m_1,(L_4+5)\cdot\ceil{\log_2m_1}, 6m_1, 1)\end{align} such that 
\[
\abs{H(x_1,\ldots, x_{m_1})-\prod_{j=1}^{m_1}x_j}\leq m_1^2\cdot 2^{-L_4},\;\forall\;x_1,x_2,\ldots, x_{m_1}\in[0,1]. 
\]  For each  $\mathbf{a}=(\tfrac{Rk_1}{G}, \tfrac{Rk_2}{G},\ldots, \tfrac{Rk_{m_1}}{G})\in \Omega^{m_1}$, define 
\[
\tilde{u}_{\mathbf{a}}:\R^{m_1}\to\R, \x=(x_1,\ldots, x_{m_1})\mapsto H(l_{k_1}(x_1), l_{k_2}(x_2),\ldots, l_{k_{m_1}}(x_{m_1})). 
\]
Then we have that 
\[
&\abs{\tilde{u}_{\mathbf{a}}(\x)-u_{\mathbf{a}}(\x)}=\abs{H(l_{k_1}(x_1), l_{k_2}(x_2),\ldots, l_{k_{m_1}}(x_{m_1}))-\prod_{j=1}^{m_1}l_{k_j}(x_j)}
\\&\leq m_1^2\cdot 2^{-L_4},\;\forall\;\x=(x_1,\ldots,x_{m_1})\in\R^{m_1}. 
\] Thus 
\begin{align}\label{20260428062247}
\abs{\tilde{u}_{\mathbf{a}}(\x)-u_{\mathbf{a}}(\x)}\leq m_1^2\cdot 2^{-L_4},\;\forall\;\x\in\R^{m_1},\;\forall\;\mathbf{a}\in \Omega^{m_1}. 
\end{align}
Combining \eqref{20260428044647} and \eqref{20260428044723}, we obtain that 
\[
\tilde{u}_{\mathbf{a}}&\in \fdnn(m_1,2+(L_4+5)\cdot\ceil{\log_2m_1}, 6m_1, 1)
\\&=\fdnn(m_1,2+\flr{\tfrac{L_1-2}{\ceil{\log_2m_1}}}\cdot\ceil{\log_2m_1}, 6m_1, 1)
\subset \fdnn(m_1,L_1, 6m_1, 1),\;\forall\;\mathbf{a}\in\Omega^{m_1}. 
\] Thus 
\begin{align}\label{20260428051845}
\tilde{u}_{\mathbf{a}} \in\fdnn(m_1,L_1, 6m_1, 1),\;\forall\;\mathbf{a}\in\Omega^{m_1}. 
\end{align} Let
\[
{q}:\R^{m_1}\times[-R,R]^{m_2}\to\R, (\x,\y)\mapsto \sum_{\mathbf{a}\in\Omega^{m_1}}{u}_{\mathbf{a}}(\x)h(\mathbf{a}, \y), 
\] and
\[
\tilde{q}:\R^{m_1}\times\R^{m_2}\to\R, (\x,\y)\mapsto \sum_{\mathbf{a}\in\Omega^{m_1}}\tilde{u}_{\mathbf{a}}(\x)\tilde{f}_{\mathbf{a}, L_3,N_3}(\y).  
\] 
Obviously, \[
&\abs{q(\x,\y)-h(\x,\y)}\xlongequal{\because\eqref{20260428060044}}\abs{\sum_{\mathbf{a}\in\Omega^{m_1}}{u}_{\mathbf{a}}(\x)h(\mathbf{a}, \y)-\sum_{\mathbf{a}\in\Omega^{m_1}}{u}_{\mathbf{a}}(\x)h(\x, \y)}\\&=
\abs{\sum_{\mathbf{a}\in\Omega^{m_1}}{u}_{\mathbf{a}}(\x)(h(\mathbf{a}, \y)-h(\x,\y))}\leq
\sum_{\mathbf{a}\in\Omega^{m_1}}{u}_{\mathbf{a}}(\x)\abs{h(\mathbf{a}, \y)-h(\x,\y)}
\\&\xlongequal{\because\eqref{20260428060604}}
\sum_{\mathbf{a}\in\Omega^{m_1},\norm{\x-\mathbf{a}}_\infty\leq \frac{R}{G}}{u}_{\mathbf{a}}(\x)\abs{h(\mathbf{a}, \y)-h(\x,\y)}
\leq \sum_{\mathbf{a}\in\Omega^{m_1},\norm{\x-\mathbf{a}}_\infty\leq \frac{R}{G}}{u}_{\mathbf{a}}(\x)\cdot B\cdot\norm{\mathbf{a}-\x}_\infty^\beta
\\&\leq 
\sum_{\mathbf{a}\in\Omega^{m_1},\norm{\x-\mathbf{a}}_\infty\leq \frac{R}{G}}{u}_{\mathbf{a}}(\x)\cdot B\cdot\ykh{\tfrac{R}{G}}^\beta\leq 
\sum_{\mathbf{a}\in\Omega^{m_1}}{u}_{\mathbf{a}}(\x)\cdot B\cdot\ykh{\tfrac{R}{G}}^\beta\\&\xlongequal{\because\eqref{20260428060044}}B\cdot\ykh{\tfrac{R}{G}}^\beta,\;\forall\;(\x,\y)\in[-R,R]^{m_1}\times[-R,R]^{m_2}. 
\] Thus \begin{align}\label{20260428070404}
 \sup_{\substack{\x\in[-R,R]^{m_1}\\\y\in[-R,R]^{m_2}}}\abs{q(\x,\y)-h(\x,\y)}\leq  B\cdot\ykh{\tfrac{R}{G}}^\beta. 
\end{align}
 Combining \eqref{20260428062157} and \eqref{20260428062247}, we obtain that 
\[
&\abs{q(\x,\y)-\tilde{q}(\x,\y)}=\abs{\sum_{\mathbf{a}\in\Omega^{m_1}}{u}_{\mathbf{a}}(\x)h(\mathbf{a}, \y)-\sum_{\mathbf{a}\in\Omega^{m_1}}\tilde{u}_{\mathbf{a}}(\x)\tilde{f}_{\mathbf{a}, L_3,N_3}(\y)}
\\&=\abs{\sum_{\mathbf{a}\in\Omega^{m_1}}{u}_{\mathbf{a}}(\x)(h(\mathbf{a}, \y)-\tilde{f}_{\mathbf{a}, L_3,N_3}(\y))+\sum_{\mathbf{a}\in\Omega^{m_1}}(u_{\mathbf{a}}(\x)-\tilde{u}_{\mathbf{a}}(\x))\tilde{f}_{\mathbf{a}, L_3,N_3}(\y)}
\\&\leq \sum_{\mathbf{a}\in\Omega^{m_1}}{u}_{\mathbf{a}}(\x)\abs{h(\mathbf{a}, \y)-\tilde{f}_{\mathbf{a}, L_3,N_3}(\y)}+\sum_{\mathbf{a}\in\Omega^{m_1}}\abs{u_{\mathbf{a}}(\x)-\tilde{u}_{\mathbf{a}}(\x)}\abs{\tilde{f}_{\mathbf{a}, L_3,N_3}(\y)}
\\&\leq
\sum_{\mathbf{a}\in\Omega^{m_1}}{u}_{\mathbf{a}}(\x)\abs{\zzhcer\cdot \ykh{\tfrac{1}{L_3^2N_3^2}}^{\beta/m_2}}+\sum_{\mathbf{a}\in\Omega^{m_1}}\abs{u_{\mathbf{a}}(\x)-\tilde{u}_{\mathbf{a}}(\x)}\abs{\abs{h(\mathbf{a},\y)}+\zzhcer\cdot \ykh{\tfrac{1}{L^2N^2}}^{\beta/m_2}}
\\&\leq 
\sum_{\mathbf{a}\in\Omega^{m_1}}{u}_{\mathbf{a}}(\x)\abs{\zzhcer\cdot \ykh{\tfrac{1}{L_3^2N_3^2}}^{\beta/m_2}}+\sum_{\mathbf{a}\in\Omega^{m_1}}\abs{m_1^2\cdot 2^{-L_4}}\abs{B+\zzhcer\cdot \ykh{\tfrac{1}{L^2N^2}}^{\beta/m_2}}
\\&\xlongequal{\because\eqref{20260428060044}}
{\zzhcer\cdot \ykh{\tfrac{1}{L_3^2N_3^2}}^{\beta/m_2}}+(2G+1)^{m_1}\cdot\abs{m_1^2\cdot 2^{-L_4}}\abs{B+\zzhcer\cdot \ykh{\tfrac{1}{L^2N^2}}^{\beta/m_2}}\\&<\infty, \;\forall\; (\x,\y)\in[-R,R]^{m_1}\times [-R,R]^{m_2}, 
\]  which, together with \eqref{20260428070404}, yields 
\[
&\sup_{[-R,R]^{m_1+m_2}}\abs{\tilde{q}-h}\leq \sup_{[-R,R]^{m_1+m_2}}\abs{\tilde{q}-q}+\sup_{[-R,R]^{m_1+m_2}}\abs{{q}-h}
\\&\leq{\zzhcer\cdot \ykh{\tfrac{1}{L_3^2N_3^2}}^{\beta/m_2}}+(2G+1)^{m_1}\cdot\abs{m_1^2\cdot 2^{-L_4}}\abs{B+\zzhcer\cdot \ykh{\tfrac{1}{L_3^2N_3^2}}^{\beta/m_2}}+B\cdot\ykh{\tfrac{R}{G}}^\beta<\infty. 
\]
From \eqref{20260428051811}, \eqref{20260428054003},  \eqref{20260428054155}, \eqref{20260428052852},  and \eqref{20260428051845} we see there exist
\[
\tilde{v}&\in\fdnn(m_1,L_1, 6m_1\cdot \abs{\Omega^{m_1}}, \abs{\Omega^{m_1}})=\fdnn(m_1,L_1, 6m_1\cdot {(2G+1)^{m_1}}, {(2G+1)^{m_1}})\\
&\subset \fdnn(m_1,L_1, N_1, {(2G+1)^{m_1}})
\] and 
\[
\tilde{g}&\in\fdnn(m_2, \zzhcyi L_3, \zzhcyi N_3\cdot\abs{\Omega^{m_1}},  \abs{\Omega^{m_1}})\\
&\subset\fdnn(m_2, L_2, \zzhcyi N_3\cdot{(2G+1)^{m_1}},  {(2G+1)^{m_1}})\subset \fdnn(m_2, L_2, N_2,  {(2G+1)^{m_1}})
\] such that 
\[
\tilde{q}(\x,\y)=\sum_{\mathbf{a}\in\Omega^{m_1}}\tilde{u}_{\mathbf{a}}(\x)\tilde{f}_{\mathbf{a}, L_3,N_3}(\y)=\tilde{v}(\x)^\top \tilde{g}(\y),\;\forall\; (\x,\y)\in\R^{m_1}\times\R^{m_2}, 
\] which, together with the fact that $(2G+1)^{m_1}\leq D$, yields  
\[&\inf\set{\sup_{\substack{\x\in[-R,R]^{m_1}\\\y\in[-R,R]^{m_2}}}\abs{f(\x)^\top g(\y)-h(\x,\y)}}{\begin{aligned}f\in \fdnn(m_1,L_1, N_1, D), \\g\in \fdnn(m_2,L_2, N_2, D)\end{aligned}}\\
&\leq \inf\set{\sup_{\substack{\x\in[-R,R]^{m_1}\\\y\in[-R,R]^{m_2}}}\abs{f(\x)^\top g(\y)-h(\x,\y)}}{\begin{aligned}f\in \fdnn(m_1,L_1, N_1, (2G+1)^{m_1}), \\g\in \fdnn(m_2,L_2, N_2, (2G+1)^{m_1})\end{aligned}}
\\&\leq  \sup_{\substack{\x\in[-R,R]^{m_1}\\\y\in[-R,R]^{m_2}}}\abs{\tilde{v}(\x)^\top \tilde{g}(\y)-h(\x,\y)}=\sup_{\substack{\x\in[-R,R]^{m_1}\\\y\in[-R,R]^{m_2}}}\abs{\tilde{q}(\x,\y)-h(\x,\y)}
\\&\leq {{\zzhcer\cdot \ykh{\tfrac{1}{L_3^2N_3^2}}^{\beta/m_2}}+(2G+1)^{m_1}\cdot\abs{m_1^2\cdot 2^{-L_4}}\abs{B+\zzhcer\cdot \ykh{\tfrac{1}{L_3^2N_3^2}}^{\beta/m_2}}+B\cdot\ykh{\tfrac{R}{G}}^\beta}. 
\] Since $h$ is arbitrary, we obtain that
\[&\sup_{\substack{h:[-R,R]^{m_1+m_2}\to \R\\{\norm{h}_{\beta, [-R,R]^{m_1+m_2}}\leq B}}}\inf\set{\sup_{\substack{\x\in[-R,R]^{m_1}\\\y\in[-R,R]^{m_2}}}\abs{f(\x)^\top g(\y)-h(\x,\y)}}{\begin{aligned}f\in \fdnn(m_1,L_1, N_1, D), \\g\in \fdnn(m_2,L_2, N_2, D)\end{aligned}}
\\&\leq {{\zzhcer\cdot \ykh{\tfrac{1}{L_3^2N_3^2}}^{\beta/m_2}}+(2G+1)^{m_1}\cdot\abs{m_1^2\cdot 2^{-L_4}}\abs{B+\zzhcer\cdot \ykh{\tfrac{1}{L_3^2N_3^2}}^{\beta/m_2}}+B\cdot\ykh{\tfrac{R}{G}}^\beta}
\\&\lesssim {{ \ykh{\tfrac{1}{L_3^2N_3^2}}^{\beta/m_2}}+G^{m_1}\cdot{ 2^{-L_4}}+\ykh{\tfrac{1}{G}}^\beta}\lesssim
 \ykh{\tfrac{1}{L_3N_3}}^{2\beta/m_2}+\ykh{\tfrac{1}{G}}^{\beta}+G^{m_1}\cdot{ 2^{-L_4}}
 \\&\lesssim \abs{\e^{\theta_4-\min\hkh{\theta_2,\theta_4,\theta_5}}\cdot \e^{\theta_3}}^{2\beta/m_2}+\abs{\e^{\min\hkh{\theta_2,\theta_4, \theta_5}}}^{\beta/m_1}
 \\&\;\;\;\;\;\;\;+\e^{-\min\hkh{\theta_2,\theta_4, \theta_5}}\cdot 2^{-\e^{-\theta_1/2}} 
 \\&\lesssim \abs{\e^{\theta_3+\theta_4-\min\hkh{\theta_2,\theta_4,\theta_5}}}^{2\beta/m_2}+\abs{\e^{\min\hkh{\theta_2,\theta_4, \theta_5}}}^{\beta/m_1}. 
\]

This proves the desired result. \end{proof}

\begin{lemma}\label{20260501210547}
Suppose $d_1\in\mb{N}$, $d_2\in\mb{N}$,  $R\in(0,\infty)$,  $\beta\in(0,1]$,  and  ${B\in(0,\infty)}$.  
Fix arbitrary positive numbers  $\theta_{1,1}$, $\theta_{1,2}$, $\theta_{2,1}$, $\theta_{2,2}$, $\theta_0$. Denote $\alpha_1:=\theta_{1,1}+\theta_{1,2}$, $\alpha_2:=\theta_{2,1}+\theta_{2,2}$,  and $\alpha_0:=\min\hkh{\theta_0, \theta_{1,2},\theta_{2,2}}$.  
Then for 
 sufficiently small $\e\in(0,\infty)$ and ${L_1\asymp \e^{-\theta_{1,1}}}$, $N_1 \asymp \e^{-\theta_{1,2}}$, $L_2 \asymp \e^{-\theta_{2,1}}$, $N_2 \asymp \e^{-\theta_{2,2}}$, $D\asymp \e^{-\theta_0}$ there holds  
\[&\sup_{{\substack{h:[-R,R]^{d_1+d_2}\to\R\\\norm{h}_{\beta,[-R,R]^{d_1+d_2}}\leq B}}}\inf\set{\sup_{\substack{\x\in[-R,R]^{d_1}\\\y\in[-R,R]^{d_2}}}\abs{f(\x)^\top g(\y)-h(\x,\y)}}{\begin{aligned}f\in \fdnn(d_1,L_1, N_1, D), \\g\in \fdnn(d_2,L_2, N_2, D)\end{aligned}}\\&\lesssim \min\hkh{\abs{\e^{\alpha_2-\alpha_0}}^{\frac{2\beta}{d_2}}+\abs{\e^{\alpha_0}}^{\frac{\beta}{d_1}}, \abs{\e^{\alpha_1-\alpha_0}}^{\frac{2\beta}{d_1}}+\abs{\e^{\alpha_0}}^{\frac{\beta}{d_2}}}.  
\]
\end{lemma}

\begin{proof}
Applying Lemma \ref{20260501042647} directly, we get 
\[&\sup_{\substack{h:[-R,R]^{d_1+d_2}\to\R\\\norm{h}_{\beta,[-R,R]^{d_1+d_2}}\leq B}}{\inf\set{\sup_{\substack{\x\in[-R,R]^{d_1}\\\y\in[-R,R]^{d_2}}}\abs{f(\x)^\top g(\y)-h(\x,\y)}}{\begin{aligned}f\in \fdnn(d_1,L_1, N_1, D)\\ g\in \fdnn(d_2,L_2, N_2, D)\end{aligned}}}\\&\lesssim {\abs{\e^{\alpha_2-\alpha_0}}^{\frac{2\beta}{d_2}}+\abs{\e^{\alpha_0}}^{\frac{\beta}{d_1}}}, 
\] and 
\[&\sup_{\substack{h:[-R,R]^{d_2+d_1}\to\R\\\norm{h}_{\beta,[-R,R]^{d_2+d_1}}\leq B}}{\inf\set{\sup_{\substack{\x\in[-R,R]^{d_2}\\\y\in[-R,R]^{d_1}}}\abs{f(\x)^\top g(\y)-h(\x,\y)}}{\begin{aligned}f\in \fdnn(d_2,L_2, N_2, D)\\ g\in \fdnn(d_1,L_1, N_1, D)\end{aligned}}}\\&\lesssim { \abs{\e^{\alpha_1-\alpha_0}}^{\frac{2\beta}{d_1}}+\abs{\e^{\alpha_0}}^{\frac{\beta}{d_2}}}. 
\] Obviously, 
\[
&\sup_{\substack{h:[-R,R]^{d_1+d_2}\to\R\\\norm{h}_{\beta,[-R,R]^{d_1+d_2}}\leq B}}{\inf\set{\sup_{\substack{\x\in[-R,R]^{d_1}\\\y\in[-R,R]^{d_2}}}\abs{f(\x)^\top g(\y)-h(\x,\y)}}{\begin{aligned}f\in \fdnn(d_1,L_1, N_1, D)\\ g\in \fdnn(d_2,L_2, N_2, D)\end{aligned}}}
\\&=\sup_{\substack{h:[-R,R]^{d_2+d_1}\to\R\\\norm{h}_{\beta,[-R,R]^{d_2+d_1}}\leq B}}{\inf\set{\sup_{\substack{\x\in[-R,R]^{d_2}\\\y\in[-R,R]^{d_1}}}\abs{f(\x)^\top g(\y)-h(\x,\y)}}{\begin{aligned}f\in \fdnn(d_2,L_2, N_2, D)\\ g\in \fdnn(d_1,L_1, N_1, D)\end{aligned}}}. 
\] Combining the above, the desired result immediately follows.
\end{proof}

\begin{theorem}\label{20260503141526}
Let Assumption \ref{20260427030032} hold. Fix arbitrary positive numbers  $\theta_{1,1}$, $\theta_{1,2}$, $\theta_{2,1}$, $\theta_{2,2}$, $\theta_0$. Denote $\alpha_1:=\theta_{1,1}+\theta_{1,2}$, $\alpha_2:=\theta_{2,1}+\theta_{2,2}$,  and $\alpha_0:=\min\hkh{\theta_0, \theta_{1,2},\theta_{2,2}}$. 
Then for 
 sufficiently small $\e\in(0,\infty)$ and ${L_1\asymp \e^{-\theta_{1,1}}}$, $N_1 \asymp \e^{-\theta_{1,2}}$, $L_2 \asymp \e^{-\theta_{2,1}}$, $N_2 \asymp \e^{-\theta_{2,2}}$, $D\asymp \e^{-\theta_0}$ there holds  
\[&\inf\set{\L(\tilde{s})-\L^*}{ \tilde{s}\in \fscoring(L_1,N_1,L_2,N_2,D)}\\&\lesssim \min\hkh{\abs{\e^{\alpha_2-\alpha_0}}^{\frac{4\beta}{d_2}}+\abs{\e^{\alpha_0}}^{\frac{2\beta}{d_1}}, \abs{\e^{\alpha_1-\alpha_0}}^{\frac{4\beta}{d_1}}+\abs{\e^{\alpha_0}}^{\frac{2\beta}{d_2}}}.  
\]
\end{theorem}

\begin{proof}
According to Lemma \ref{20260428013219} and Assumption \ref{20260427030032}, there exists a function ${h:[-R,R]^{d_1+d_2}\to \R}$ such that  \begin{align}\label{20260501201902}
h(\x,\y)=s^*(\x,\y),\;\forall\;(\x,\y)\in\X\times\Y,
\end{align}  and $\norm{h}_{\beta,[-R,R]^{d_1+d_2}}\leq \norm{s^*}_{\beta,\X\times\Y}\leq B$. For  any  
$f\in \fdnn(d_1,L_1, N_1, D)$ and any ${g\in \fdnn(d_2,L_2, N_2, D)}$,  denote 
\[
U_{f,g}:\R^{d_1}\times\R^{d_2},\;(\x,\y)\mapsto { f(\x)^\top g(\y)}. 
\] Obviously,  $U_{f,g}\in  \fscoring(L_1,N_1,L_2,N_2,D)$. 
Then it follows from Lemma \ref{20260428072356}   that 
\[
&\inf\set{\L(\tilde{s})-\L^*}{ \tilde{s}\in \fscoring(L_1,N_1,L_2,N_2,D)}\leq \L(U_{f,g})-\L^*
\\&\leq \E_\x\zkh{\sup_{\y\in\Y}\tfrac{\abs{U_{f,g}(\x,\y)-s^*(\x,\y)}^2}{2\tau}}\xlongequal{\because\eqref{20260501201902}}
\E_\x\zkh{\sup_{\y\in\Y}\tfrac{\abs{U_{f,g}(\x,\y)-h(\x,\y)}^2}{2\tau}}\\&\leq\frac{1}{2\tau}\cdot\sup_{\x\in\X,\y\in\Y}\abs{U_{f,g}(\x,\y)-h(\x,\y)}^2= \frac{1}{2\tau}\cdot\sup_{\substack{\x\in\X, \y\in\Y}}\abs{f(\x)^\top g(\y)-h(\x,\y)}^2. 
\] Since $f$ and $g$ are arbitrary, we can use Lemma \ref{20260501210547} to obtain that 
\[
&\inf\set{\L(\tilde{s})-\L^*}{ \tilde{s}\in \fscoring(L_1,N_1,L_2,N_2,D)}
\\&\leq\frac{1}{2\tau}\cdot \abs{\inf\set{\sup_{\substack{\x\in\X, \y\in\Y}}\abs{f(\x)^\top g(\y)-h(\x,\y)}}{\begin{aligned}f\in \fdnn(d_1,L_1, N_1, D)\\{g\in \fdnn(d_2,L_2, N_2, D)}\end{aligned}}}^2
\\&\leq
\frac{1}{2\tau}\cdot\sup_{{\substack{u:[-R,R]^{d_1+d_2}\to\R\\\norm{u}_{\beta,[-R,R]^{d_1+d_2}}\leq B}}} \abs{\inf\set{\sup_{\substack{\x\in\X\\ \y\in\Y}}\abs{f(\x)^\top g(\y)-u(\x,\y)}}{\begin{aligned}f\in \fdnn(d_1,L_1, N_1, D)\\{g\in \fdnn(d_2,L_2, N_2, D)}\end{aligned}}}^2
\\&\lesssim\abs{\min\hkh{\abs{\e^{\alpha_2-\alpha_0}}^{\frac{2\beta}{d_2}}+\abs{\e^{\alpha_0}}^{\frac{\beta}{d_1}}, \abs{\e^{\alpha_1-\alpha_0}}^{\frac{2\beta}{d_1}}+\abs{\e^{\alpha_0}}^{\frac{\beta}{d_2}}}}^2
\\&\lesssim 
\min\hkh{\abs{\e^{\alpha_2-\alpha_0}}^{\frac{4\beta}{d_2}}+\abs{\e^{\alpha_0}}^{\frac{2\beta}{d_1}}, \abs{\e^{\alpha_1-\alpha_0}}^{\frac{4\beta}{d_1}}+\abs{\e^{\alpha_0}}^{\frac{2\beta}{d_2}}}. 
\] This proves the desired result. 
\end{proof}

\begin{corollary}\label{20260503193310}
Let Assumption \ref{20260427030032} hold. 
Denote ${I}:=\set{i\in\hkh{1,2}}{d_i=\max\hkh{d_1,d_2}}$. 
Fix arbitrary positive numbers  $\theta_{1,1}$, $\theta_{1,2}$, $\theta_{2,1}$, $\theta_{2,2}$ that satisfy $\theta_{1,2}\geq 1$, $\theta_{2,2}\geq 1$,  and  \begin{align}\label{20260503154500}
\max\set{\theta_{i,1}+\theta_{i,2}}{{i\in I}}\geq  1+\tfrac12\cdot\tfrac{\max\hkh{d_1,d_2}}{\min\hkh{d_1,d_2}}.\end{align} 
Then for 
 sufficiently small $\delta\in(0,\infty)$ and ${L_1\asymp \delta^{-\theta_{1,1}}}$, $N_1 \asymp \delta^{-\theta_{1,2}}$, $L_2 \asymp \delta^{-\theta_{2,1}}$, $N_2 \asymp \delta^{-\theta_{2,2}}$, $D\asymp \delta^{-1}$ there holds  
\[&\inf\set{\L(\tilde{s})-\L^*}{ \tilde{s}\in \fscoring(L_1,N_1,L_2,N_2,D)}\lesssim \delta^{\frac{2\beta}{\min\hkh{d_1,d_2}}}.  
\]
\end{corollary}
\begin{proof}From \eqref{20260503154500} we see there exists $\iota\in I$ such that $\theta_{\iota,1}+\theta_{\iota,2}\geq 1+\frac12\cdot\frac{\max\hkh{d_1,d_2}}{\min\hkh{d_1,d_2}}=1+\frac{d_{\iota}}{2\cdot d_{3-\iota}}$. 
Taking $\theta_0=1$ in Theorem \ref{20260503141526}, we obtain that  
\[&\inf\set{\L(\tilde{s})-\L^*}{ \tilde{s}\in \fscoring(L_1,N_1,L_2,N_2,D)}\\&\lesssim \min\hkh{\abs{\delta^{\theta_{2,1}+\theta_{2,2}-1}}^{\frac{4\beta}{d_2}}+\abs{\delta^{}}^{\frac{2\beta}{d_1}}, \abs{\delta^{\theta_{1,1}+\theta_{1,2}-1}}^{\frac{4\beta}{d_1}}+\abs{\delta^{}}^{\frac{2\beta}{d_2}}}\lesssim \abs{\delta^{\theta_{\iota,1}+\theta_{\iota,2}-1}}^{\frac{4\beta}{d_\iota}}+\abs{\delta^{}}^{\frac{2\beta}{d_{3-\iota}}} 
\\&\lesssim \abs{\delta^{1+\frac{d_{\iota}}{2\cdot d_{3-\iota}}-1}}^{\frac{4\beta}{d_\iota}}+\abs{\delta^{}}^{\frac{2\beta}{d_{3-\iota}}} =2\cdot \abs{\delta^{}}^{\frac{2\beta}{d_{3-\iota}}}=2\cdot \abs{\delta^{}}^{\frac{2\beta}{\min\hkh{d_1,d_2}}}. 
\] This proves the desired result. 
\end{proof}

\begin{proof}[Proof of Theorem \ref{4p3thm20260427130502}] 
Take  $\theta_{1,1}:=\frac{d_1}{4d_2}$, $\theta_{1,2}:=1+\frac{d_1}{4d_2}$, $\theta_{2,1}:=\frac{d_2}{4d_1}$, $\theta_{2,2}:=1+\frac{d_2}{4d_1}$. Obviously,  these constants $\theta_{1,1}$, $\theta_{1,2}$, $\theta_{2,1}$, $\theta_{2,2}$ satisfy the conditions  of Corollary \ref{20260503193310}. 
We can then take ${\delta=\e^{\frac{\min\hkh{d_1,d_2}}{2\beta}}}$
  in Corollary \ref{20260503193310} to obtain the  desired result. 
\end{proof}

\section{Proof of Proposition~\ref{prop:latent-spectrum_informal}}\label{app:proof_lanten_spectrum}
First, we state a formal version of Proposition~\ref{prop:latent-spectrum_informal}.
\begin{proposition}[Spectral structure under latent classes]
    \label{prop:latent-spectrum}
    Under Assumption~\ref{ass:latent-class-model}, the
    PMI operator $T:L_2(P_{\X})\to L_2(P_{\Y})$ has rank at most $K$. 
    Let
    \begin{align*}
        e_k(\x) &= \frac{1}{\sqrt{\pi_k}}\,\mathbb{I}\{z_\x=k\},\quad k\in[K],\\
        b_k(\y) &= \log\frac{p(z=k|\y)}{\pi_k},\quad k\in[K],
    \end{align*}
    and define the $K\times K$ matrix $M$ with entries
    $M_{kl} = \E_{\y\sim p_{\Y}}[b_k(\y)b_l(\y)]$.
    Then the non-zero eigenvalues $\{\lambda_i\}_{i=1}^{K}$ of $T^*T$ are exactly
    the eigenvalues of $M$, with corresponding orthonormal eigenfunctions
    \[
        \psi_i(\x) = \sum_{k=1}^{K} v^{(i)}_k\, e_k(\x),
    \]
    where $v^{(i)}\in\R^K$ are the eigenvectors of $M$.  The eigenfunctions of $TT^*$ are
    \[
        \phi_i(\y) = \frac{1}{\sqrt{\lambda_i}}\sum_{k=1}^{K} \sqrt{\pi_k}\,v^{(i)}_k\, b_k(\y).
    \]
    Consequently, the subspaces spanned by the eigenfunctions and by the (log)conditional probability coincide:
    \[
        \mathrm{span}\{\psi_1,\dots,\psi_K\} = \mathrm{span}\{p(\cdot=k\mid\x)\}_{k=1}^K,
        \mathrm{span}\{\phi_1,\dots,\phi_K\} = \mathrm{span}\{\log p(\cdot=k\mid\y)\}_{k=1}^K.
    \]
    There also holds
    \[ \log\frac{p(\x,\y)}{p_{\X}(\x)p_{\Y}(\y)}=\sum_{k=1}^K p(k\mid\x)\log\frac{p(z=k|\y)}{\pi_k}.\]
\end{proposition}

\begin{proof}[Proof of Proposition~\ref{prop:latent-spectrum}]
According to Bayes rule, the
pointwise mutual information simplifies to be
\begin{align}
    \log\frac{p(\x,\y)}{p_{\X}(\x)p_{\Y}(\y)}
    &= \log\sum_{k=1}^{K} \frac{p(z=k\mid\x)\,p(z=k\mid\y)}{\pi_k}
       =\log \frac{p(z_{\x}\mid\y)}{\pi_{z_\x}}.\notag\\
       &= \sum_{k=1}^{K} \mathbb{I}\{z_\x=k\} \log\frac{p(z=k|\y)}{\pi_k}
    = \sum_{k=1}^{K} \mathbb{I}\{\x\in C_k\}\, b_k(\y),
       \label{eq:r-latent}
\end{align}
where $C_k = \{\x : z_\x=k\}$ and $b_k(\y) = \log\frac{p(z=k|\y)}{\pi_k}$.

\paragraph{Finite rank.}
For any $f\in L_2(P_{\X})$,
\begin{align*}
    (Tf)(\y)
    &= \int_{\X}  \log\frac{p(\x,\y)}{p_{\X}(\x)p_{\Y}(\y)} f(\x) p_{\X}(\x) \,d\x \\
    &= \sum_{k=1}^{K} b_k(\y) \int_{C_k} f(\x) p_{\X}(\x) \,d\x .
\end{align*}
Hence $\operatorname{Im}(T) \subseteq \operatorname{span}\{b_1,\dots,b_K\}$, so $\operatorname{rank}(T)\le K$.

\paragraph{Singular functions of $T$.}
The sets $C_1,\dots,C_K$ form a partition of $\X$.  The functions $\{\mathbb{I}\{\x\in C_k\}\}_{k=1}^K$ are orthogonal in $L_2(P_{\X})$ with norms $\|\mathbb{I}\{\x\in C_k\}\|_{L_2(P_{\X})}^2 = P_{\X}(C_k) = \pi_k$.  Thus
\[
    e_k(\x) = \frac{\mathbb{I}\{\x\in C_k\}}{\sqrt{\pi_k}}, \quad k=1,\dots,K,
\]
is an orthonormal basis of $\mathcal{H}_0 = \operatorname{span}\{\mathbb{I}\{\x\in C_k\}\}\subset L_2(P_{\X})$.

Let $\mathcal{H}_0^{\perp}$ be the orthogonal complement
of $\mathcal{H}_0$ in $L_2(P_\X)$.
For any $f\in\mathcal{H}_0^{\perp}$ we have $\int_{C_k} f(\x) p_{\X}(\x) d\x=0$ for all $k$, and consequently $Tf=0$.  Therefore $T^*T$ vanishes on $\mathcal{H}_0^{\perp}$ and its non‑zero eigenspace lies in $\mathcal{H}_0$.

Take $f = \sum_{k=1}^{K} \alpha_k e_k \in \mathcal{H}_0$.  Then
\[
    \int_{C_k} f(\x) p_{\X}(\x)\,d\x 
    = \sum_{l=1}^{K} \alpha_l \frac{1}{\sqrt{\pi_l}} \int_{C_k} \mathbb{I}\{\x\in C_l\} p_{\X}(\x)\,d\x
    = \alpha_k \sqrt{\pi_k}.
\]
Hence
\[
    (Tf)(\y) = \sum_{k=1}^{K} \alpha_k \sqrt{\pi_k}\, b_k(\y).
\]

Now compute $T^*Tf$.  For any $g\in L_2(P_{\Y})$,
\[
    (T^*g)(\x) = \int_{\Y} \log\frac{p(\x,\y)}{p_{\X}(\x)p_{\Y}(\y)} g(\y) p_{\Y}(\y)\,d\y.
\]
Using \eqref{eq:r-latent}, for $\x\in C_k$,
\[
    (T^*g)(\x) = \int_{\Y} b_k(\y) g(\y) p_{\Y}(\y)\,d\y .
\]
Applying this with $g = Tf$ yields, for $\x\in C_k$,
\begin{align*}
    (T^*T f)(\x)
    &= \int_{\Y} b_k(\y) \biggl( \sum_{l=1}^{K} \alpha_l \sqrt{\pi_l}\, b_l(\y) \biggr) p_{\Y}(\y)\,d\y \\
    &= \sum_{l=1}^{K} \alpha_l \sqrt{\pi_l}\, M_{kl},
\end{align*}
where $M_{kl} = \E_{\y\sim p_{\Y}}[b_k(\y) b_l(\y)]$.

Since $\mathbb{I}\{\x\in C_k\} = \sqrt{\pi_k}\,e_k(\x)$, we rewrite
\begin{align*}
    (T^*T f)(\x)
    &= \sum_{k=1}^{K} e_k(\x) \sqrt{\pi_k} \sum_{l=1}^{K} \alpha_l \sqrt{\pi_l}\, M_{kl} \\
    &= \sum_{k=1}^{K} \Bigl( \sum_{l=1}^{K} A_{kl} \alpha_l \Bigr) e_k(\x),
\end{align*}
with $A_{kl} = \sqrt{\pi_k\pi_l}\, M_{kl}$.  Thus, in the orthonormal basis $\{e_k\}$, the operator $T^*T$ is represented by the matrix $A$.  $A$ is symmetric and positive semidefinite because
\[
    v^{\top} A v = \biggl\| \sum_{k} v_k \sqrt{\pi_k}\, b_k \biggr\|_{L_2(P_{\Y})}^2 \ge 0.
\]

Let $\lambda_1\ge\cdots\ge\lambda_K>0$ be the non‑zero eigenvalues of $A$ and $v^{(1)},\dots,v^{(K)}\in\R^K$ the corresponding orthonormal eigenvectors.  Define
\[
    \psi_i(\x) = \sum_{k=1}^{K} v^{(i)}_k e_k(\x), \quad i=1,\dots,K.
\]
Then $T^*T\psi_i = \lambda_i\psi_i$ and $\langle\psi_i,\psi_j\rangle_{p_{\X}} = \delta_{ij}$ because $\{e_k\}$ is orthonormal and $v^{(i)}$ are orthonormal vectors.
According to Lemma~\ref{lem:spectral_decomposition},
the right singular functions are obtained by
\[
    \phi_i(\y) = \frac{1}{\sqrt{\lambda_i}} (T\psi_i)(\y)
    = \frac{1}{\sqrt{\lambda_i}} \sum_{k=1}^{K} v^{(i)}_k \sqrt{\pi_k}\, b_k(\y).
\]

\paragraph{Equivalence of spanned subspaces.}
From the definitions,
\[
    \operatorname{span}\{\psi_1,\dots,\psi_K\} = \operatorname{span}\{e_1,\dots,e_K\}
    = \operatorname{span}\{\mathbb{I}\{\x\in C_k\}\} = \operatorname{span}\{p(z=k\mid\x)\}.
\]
Similarly,
\[
    \operatorname{span}\{\phi_1,\dots,\phi_K\} = \operatorname{span}\{b_k(\y)\}
    = \operatorname{span}\Bigl\{\log\frac{p(z=k|\y)}{\pi_k}\Bigr\}
    = \operatorname{span}\{\log p(z=k|\y)\},
\]
because the class priors are constants per class.  This completes the proof.
\end{proof}

\subsection{Discussion on a high-probability posterior condition}
\label{app:high-probability-lcgm}
In the deterministic case, the PMI function is

$$
\operatorname{PMI}(\x,\y)
=\sum_{k=1}^K
\mathbf 1\{z_\x=k\}
\log\frac{P(z=k\mid \y)}{\pi_k}
=:h_K(\x,\y),
$$

so a rank-$K$ score represents PMI exactly.

In high probability case, we suppose instead that
$P(z=z_\x\mid \x)\geq1-\delta.$ In this case,  the rank-$K$ score $\tau h_K(\x,\y)$ can approximate the optimal scoring function $s^*$ well, with an approximation error dependent on $\delta$.

 Let $\eta_k(\x)=p(z=k\mid \x),
r_k(\y)=\frac{p(z=k\mid \y)}{\pi_k}.$
Conditional independence and Bayes' rule give
$\frac{p(\x,\y)}{p_\X(\x)p_\Y(\y)}
=\sum_{k=1}^K\eta_k(\x)r_k(\y).$

Define $N(\x,\y)=\sum_{k\neq z_\x}\eta_k(\x)r_k(\y),
\qquad
\varepsilon_\delta(\x,\y)
=\log\left(
1+\frac{N(\x,\y)}
{\eta_{z_\x}(\x)r_{z_\x}(\y)}
\right).$

Then $\log\frac{p(\x,\y)}{p_\X(\x)p_\Y(\y)}
=\log\eta_{z_\x}(\x)
+\log r_{z_\x}(\y)
+\varepsilon_\delta(\x,\y).$
Let $\pi_{\min}=\min_k\pi_k$. We have
$r_k(\y)\leq\frac1{\pi_{\min}},
\qquad
N(\x,\y)\leq\frac{\delta}{\pi_{\min}}.$
According to Assumption 4.2, there holds $\eta_{z_\x}(\x)r_{z_\x}(\y)
\geq
\frac1\alpha-\frac{\delta}{\pi_{\min}}.$
Thus, whenever $\delta<\pi_{\min}/\alpha$,
$0\leq\varepsilon_\delta(\x,\y)
\leq
-\log\left(
1-\frac{\alpha\delta}{\pi_{\min}}
\right).$

The term $\log\eta_{z_\x}(\x)$ depends only on $x$ and therefore does not affect the contrastive objective. Applying the same $L_2$ comparison argument as in Theorem 4.2 yields

$$
\mathcal L(\tau h_K)-\mathcal L^*
\lesssim
\tau^2\log^2\left(
1-\frac{\alpha\delta}{\pi_{\min}}
\right).
$$

Hence the PMI operator need not remain exactly rank $K$, but the rank-$K$ score incurs only an additional $\delta$-dependent approximation error. When $\delta=0$, we recover the exact deterministic result.

\section{Proof of Theorem \ref{20260505072733}}\label{20260505072637}

Throughout this section, we will use $\operatorname{dist}(A_1,A_2)$ to denote the $\norm{\cdot}_\infty$-distance between the sets $A_1$ and $A_2$ in $\R^d$, namely 
\[
\operatorname{dist}(A_1,A_2):=\inf\setm{2}{\norm{\x-\y}_\infty}{{\x\in A_1, \y\in A_2}}. 
\]

\begin{lemma}\label{20260506192530}Let $d\in\mb{N}$, $\varnothing \neq A\subset \R^d$,  $\varnothing \neq B\subset \R^d$, and $\delta\in(0,\infty)$. Suppose $\operatorname{dist}(A,B)\geq \delta$. Define 
\[
f:\R^d\to\R, \;\x\mapsto \frac{\operatorname{dist}(\hkh{\x},B)}{\operatorname{dist}(\hkh{\x},A)+\operatorname{dist}(\hkh{\x},B)}. 
\] Then $f$ is well-defined, and  $\norm{f}_{1,\R^d}\leq \max\hkh{1, 1/\delta}$. 
\end{lemma}
\begin{proof}
For any $\x\in\R^d$ and any $\e\in(0,\infty)$, there exist $\mathbf{a}\in A$ and $\mathbf{b}\in B$ such that \[&\operatorname{dist}(\hkh{\x},A)+\operatorname{dist}(\hkh{\x},B)> \norm{\x-\mathbf{a}}_\infty-\e+\norm{\x-\mathbf{b}}_\infty-\e
\\&\geq -2\e+\norm{\mathbf{a}-\mathbf{b}}_\infty\geq-2\e+\operatorname{dist}(A,B)\geq -2\e+\delta. 
\] Since $\x$ and $\e$ are arbitrary, we get 
\begin{align}\label{20260506172422}
\operatorname{dist}(\hkh{\x},A)+\operatorname{dist}(\hkh{\x},B)\geq \delta,\;\forall\;\x\in\R^d. 
\end{align}Therefore, the denominator in the definition of $f$ is always positive, meaning that $f$ is well defined. For any $\x, \y \in\R^d$ and any $\mathbf{a}\in A$, we have 
\[
\norm{\x-\mathbf{a}}_\infty\leq \norm{\x-\y}_\infty+\norm{\y-\mathbf{a}}_\infty. 
\] Taking the infimum over $\mathbf{a}\in A$ on both sides of the above inequality gives 
\[
\operatorname{dist}(\hkh{\x},A)\leq \norm{\x-\y}_\infty+\operatorname{dist}(\hkh{\y}, {A}), 
\] meaning that 
\begin{align}\label{20260506154100}
\operatorname{dist}(\hkh{\x},A)- \operatorname{dist}(\hkh{\y}, {A}) \leq \norm{\x-\y}_\infty.  
\end{align}Similarly, we have 
\begin{align}\label{20260506154114}
\operatorname{dist}(\hkh{\y},A)- \operatorname{dist}(\hkh{\x}, {A}) \leq \norm{\y-\x}_\infty.  
\end{align} Combining \eqref{20260506154100} and \eqref{20260506154114}, we obtain $\abs{\operatorname{dist}(\hkh{\y},A)- \operatorname{dist}(\hkh{\x}, {A})} \leq \norm{\y-\x}_\infty$. Since $\x,\y$ are arbitrary, we obtain 
\begin{align}\label{20260506172221}
\abs{\operatorname{dist}(\hkh{\y},A)- \operatorname{dist}(\hkh{\x}, {A})} \leq \norm{\y-\x}_\infty, \;\forall\;\x\in \R^d, \y\in\R^d. \end{align} 
Similarly, we have \begin{align}\label{20260506172259}
\abs{\operatorname{dist}(\hkh{\y},B)- \operatorname{dist}(\hkh{\x}, {B})} \leq \norm{\y-\x}_\infty, \;\forall\;\x\in \R^d, \y\in\R^d. \end{align}Obviously, 
\begin{align}\label{20260506172116}
0\leq f(\x)\leq 1=\sup_{\z\in\R^d} f(\z),\;\forall\;\x\in\R^d.\end{align}
Define 
\[
F:\set{(w,z)\in[0,\infty)\times[0,\infty)}{w+z\geq\delta}\to[0,\infty), (x,y)\mapsto x/(x+y).  
\]
Then we have that
\[
&\abs{F(x,y)-F(u,v)}=\abs{\frac{x}{x+y}-\frac{u}{u+v}}=\abs{\frac{xv-yu}{(u+v)\cdot(x+y)}}=\abs{\frac{v\cdot(x-u)+u\cdot(v-y)}{(u+v)\cdot(x+y)}}
\\&\leq\abs{\frac{v\cdot(x-u)}{(u+v)\cdot(x+y)}}+\abs{\frac{u\cdot(v-y)}{(u+v)\cdot(x+y)}}\leq \abs{\frac{v}{u+v}}\cdot\frac{\abs{x-u}}{\delta}+\abs{\frac{u}{u+v}}\cdot\frac{\abs{v-y}}{\delta}
\\&\leq \abs{\frac{v}{u+v}}\cdot\frac{\max\hkh{\abs{x-u}, \abs{v-y}}}{\delta}+\abs{\frac{u}{u+v}}\cdot\frac{\max\hkh{\abs{x-u}, \abs{v-y}}}{\delta}\\&=\frac{\max\hkh{\abs{x-u}, \abs{v-y}}}{\delta},\;\forall\; (x,y), (u,v)\in \set{(w,z)\in[0,\infty)\times[0,\infty)}{w+z\geq\delta}, 
\] which, together with 
\eqref{20260506172422}, \eqref{20260506172221}, \eqref{20260506172259}, \eqref{20260506172116},  yields that 
\[
&\abs{f(\x)-f(\y)}=\abs{F(\operatorname{dist}(\hkh{\x},B), \operatorname{dist}(\hkh{\x},A))-F(\operatorname{dist}(\hkh{\y},B), \operatorname{dist}(\hkh{\y},A))}
\\&\leq \frac{\max\hkh{\abs{\operatorname{dist}(\hkh{\x},B)-\operatorname{dist}(\hkh{\y},B)},\abs{\operatorname{dist}(\hkh{\x},A)-\operatorname{dist}(\hkh{\y},A)}}}{\delta}
\\&\leq \frac{\max\hkh{{\norm{\x-\y}_\infty},\norm{\x-\y}_\infty}}{\delta}=\frac{\norm{\x-\y}_\infty}{\delta},\;\forall\;\x,\y\in\R^d. 
\] Therefore, 
\[
&\norm{f}_{1,\R^d}=\max\hkh{\sup_{\z\in\R^d}\abs{f(\z)}, \sup_{\R^d\ni \x\neq \y\in\R^d}\frac{\abs{f(\x)-f(\y)}}{\norm{\x-\y}_\infty}}\\&\leq \max\hkh{\sup_{\z\in\R^d}\abs{f(\z)}, \sup_{\R^d\ni \x\neq \y\in\R^d}\frac{\norm{\x-\y}_\infty/\delta}{\norm{\x-\y}_\infty}}\\&=\max\hkh{\sup_{\z\in\R^d}\abs{f(\z)}, 1/\delta}\xlongequal{\because\eqref{20260506172116}}\max\hkh{1, 1/\delta}. 
\] This proves the desired result. 
\end{proof}

\begin{lemma} \label{20260506200157}
Suppose $R\in(0,\infty)$, $B\in(0,\infty)$, $\beta\in(0,1]$, $d\in\mb{N}$. 
 Then for sufficiently small $\e\in(0,\infty)$ and $L_1\gtrsim \log\frac1\e$, $N_1\gtrsim\log\frac1\e$  there holds
\[&\sup_{\substack{h:[-R,R]^{d}\to \R\\{\norm{h}_{\beta, [-R,R]^{d}}\leq B}}}\inf\set{\sup_{{\y\in[-R,R]^{d}}}\abs{f(\y)-h(\y)}}{\begin{aligned}f\in \fdnn(d,L_1, N_1, 1) \end{aligned}}\lesssim {\frac{1}{\log^{4\beta/d}(1/\e)}}. 
\] 
\end{lemma}
\begin{proof}
According to Lemma \ref{20260428014457}, there exist constants  $\zzhcyi=\zzhcyi(d)\in\mb{N}$ and $\zzhcer=\zzhcer(d,R,\beta, B)\in(0,\infty)$ such that for any $L\in\mb{N}$, $N\in\mb{N}$ and any  $h:[-R,R]^d\to \R$ with $\norm{h}_{\beta,[-R,R]^d}\leq B$ there exists a neural network \[f\in\fdnn(d,\zzhcyi L, \zzhcyi N, 1)\] such that 
\[
{\sup_{\y\in[-R,R]^d}\abs{f(\y)-h(\y)}}\leq \zzhcer\cdot \ykh{\tfrac{1}{L^2N^2}}^{\beta/d}. 
\] Take $N_3:=\flr{N_1/{\zzhcyi}}$ and \[
L_3:=\sup\set{L\in\mb{N}\cap[2,\infty)}{\zzhcyi L\leq L_1}. 
\] Then we have that $\zzhcyi L_3\leq L_1$, $\zzhcyi N_3\leq N_1$, $N_3\asymp N_1\gtrsim \log\frac1\e$, 
and  $L_3\asymp L_1\gtrsim \log\tfrac1\e$.  Therefore, 
for any  $h:[-R,R]^d\to \R$ with $\norm{h}_{\beta,[-R,R]^d}\leq B$ there exists a neural network \[f&\in\fdnn(d,\zzhcyi L_3, \zzhcyi N_3, 1)
\subset \fdnn(d,L_1, N_1, 1)
\] such that 
\[
{\sup_{\y\in[-R,R]^d}\abs{f(\y)-h(\y)}}\leq \zzhcer\cdot \ykh{\tfrac{1}{L_3^2N_3^2}}^{\beta/d}. 
\]Since $h$ is arbitrary, we obtain that 
\[&\sup_{\substack{h:[-R,R]^{d}\to \R\\{\norm{h}_{\beta, [-R,R]^{d}}\leq B}}}\inf\set{\sup_{{\y\in[-R,R]^{d}}}\abs{f(\y)-h(\y)}}{\begin{aligned}f\in \fdnn(d,L_1, N_1, 1) \end{aligned}}\leq\zzhcer\cdot \ykh{\tfrac{1}{L_3^2N_3^2}}^{\beta/d}\\&\lesssim \ykh{\tfrac{1}{L_3^2N_3^2}}^{\beta/d}\lesssim \ykh{\tfrac{1}{L_1^2N_1^2}}^{\beta/d}\lesssim \ykh{\tfrac{1}{\log^4(1/\e)}}^{\beta/d}={\tfrac{1}{\log^{4\beta/d}(1/\e)}}. 
\]
This proves the desired result. 
\end{proof}

\begin{lemma}\label{20260506221445}
Suppose $R\in(0,\infty)$, $B\in(0,\infty)$, $\beta\in(0,1]$, $d\in\mb{N}$. 
Fix an arbitrary positive number $\theta$.  Then for sufficiently small $\e\in(0,\infty)$ and $L_2\gtrsim \log\frac1\e$, $N_2\gtrsim\log\frac1\e$ with $L_2\cdot N_2\asymp \e^{-\theta}$ there holds
\[&\sup_{\substack{h:[-R,R]^{d}\to \R\\{\norm{h}_{\beta, [-R,R]^{d}}\leq B}}}\inf\set{\sup_{{\y\in[-R,R]^{d}}}\abs{f(\y)-h(\y)}}{\begin{aligned}f\in \fdnn(d,L_2, N_2, 1) \end{aligned}}\lesssim \e^{2\theta\beta/d}. 
\] 
\end{lemma}
\begin{proof}
According to Lemma \ref{20260428014457}, there exist constants  $\zzhcyi=\zzhcyi(d)\in\mb{N}$ and $\zzhcer=\zzhcer(d,R,\beta, B)\in(0,\infty)$ such that for any $L\in\mb{N}$, $N\in\mb{N}$ and any  $h:[-R,R]^d\to \R$ with $\norm{h}_{\beta,[-R,R]^d}\leq B$ there exists a neural network \[f\in\fdnn(d,\zzhcyi L, \zzhcyi N, 1)\] such that 
\[
{\sup_{\y\in[-R,R]^d}\abs{f(\y)-h(\y)}}\leq \zzhcer\cdot \ykh{\tfrac{1}{L^2N^2}}^{\beta/d}. 
\] Take $N_3:=\flr{N_2/{\zzhcyi}}$ and \[
L_3:=\sup\set{L\in\mb{N}\cap[2,\infty)}{\zzhcyi L\leq L_2}. 
\] Then we have that $\zzhcyi L_3\leq L_2$, $\zzhcyi N_3\leq N_2$, $N_3\asymp N_2\gtrsim \log\frac1\e$, 
 and  $L_3\asymp L_2\gtrsim\log\tfrac1\e$.  Therefore, 
for any  $h:[-R,R]^d\to \R$ with $\norm{h}_{\beta,[-R,R]^d}\leq B$ there exists a neural network \[f&\in\fdnn(d,\zzhcyi L_3, \zzhcyi N_3, 1)
\subset \fdnn(d,L_2, N_2, 1)
\] such that 
\[
{\sup_{\y\in[-R,R]^d}\abs{f(\y)-h(\y)}}\leq \zzhcer\cdot \ykh{\tfrac{1}{L_3^2N_3^2}}^{\beta/d}. 
\]Since $h$ is arbitrary, we obtain that 
\[&\sup_{\substack{h:[-R,R]^{d}\to \R\\{\norm{h}_{\beta, [-R,R]^{d}}\leq B}}}\inf\set{\sup_{{\y\in[-R,R]^{d}}}\abs{f(\y)-h(\y)}}{\begin{aligned}f\in \fdnn(d,L_2, N_2, 1) \end{aligned}}\\&\leq\zzhcer\cdot \ykh{\tfrac{1}{L_3^2N_3^2}}^{\beta/d}\lesssim \ykh{\tfrac{1}{L_2^2N_2^2}}^{\beta/d}\asymp\ykh{\tfrac{1}{\e^{-2\theta}}}^{\beta/d}= \e^{2\theta\beta/d}. 
\]  This completes the proof. 
\end{proof}

\begin{proof}[Proof of Theorem \ref{20260505072733}]
For simplicity, we denote $A_k:=\set{\x\in\X}{p(z=k\mid\x)=1}$, and $g_k(\y):=\tau\cdot\log \frac{p(z=k\mid\y)}{\pi_k}$. According to Assumption \ref{ass:latent-class-model}, $\hkh{A_k}_{k=1}^K$ is  a partition of $\X$. 
 It follows from  Proposition \ref{prop:latent-spectrum} that 
\[&s^*(\x,\y):=\tau\log\frac{p(\y|\x)}{p_\Y(\y)}=\tau\log\frac{p(\x,\y)}{p_\X(\x)p_\Y(\y)}=\tau\cdot\sum_{k=1}^Kp(z=k\mid\x)\log\frac{p(z=k\mid\y)}{\pi_k}
\\&=\tau\cdot\sum_{k=1}^K\mathbb{I}\hkh{\x\in A_k}\cdot \log\frac{p(z=k\mid\y)}{\pi_k}=\sum_{k=1}^Kg_k(\y)\cdot\mathbb{I}\hkh{\x\in A_k},\;\forall\;(\x,\y)\in\X\times\Y. 
\] Take $J:=\set{k\in\hkh{1,2,\ldots, K}}{A_k\neq \varnothing}$, and $\mc{G}:=\set{g_k}{k\in J}$. For each $f\in \G$, denote 
\[
B_f:=\bigcup_{k\in J, g_k=f} A_k. 
\] 
Then we have that   $\hkh{B_f}_{f\in\G}$ is a partition of $\X$, and 
\[
&s^*(\x,\y)=\sum_{k=1}^Kg_k(\y)\cdot\mathbb{I}\hkh{\x\in A_k}=\sum_{k\in J}g_k(\y)\cdot\mathbb{I}\hkh{\x\in A_k}
\\&=\sum_{f\in\G}\sum_{\substack{k\in J\\ g_k=f}}g_k(\y)\cdot\mathbb{I}\hkh{\x\in  A_k}
=
\sum_{f\in\G}\sum_{\substack{k\in J\\ g_k=f}}f(\y)\cdot\mathbb{I}\hkh{\x\in  A_k}
\\&=\sum_{f\in\G}f(\y)\cdot\mathbb{I}\hkh{\x\in  B_f}
,\;\forall\;(\x,\y)\in\X\times\Y. 
\] For each $\G\ni f_1\neq f_2\in \G$, there exists $\y_0\in\Y$ such that $f_1(\y_0)\neq f_2(\y_0)$. Then for any $\x_1\in B_{f_1}$ and any $\x_2\in B_{f_2}$, we have that
\[
&\abs{f_1(\y_0)-f_2(\y_0)}=\abs{s^*(\x_1,\y_0)-s^*(\x_2,\y_0)}\leq \norm{s^*}_{\beta,\X\times\Y}\cdot \norm{(\x_1,\y_0)-(\x_2,\y_0)}^\beta_\infty
\\&\leq B\cdot \norm{(\x_1,\y_0)-(\x_2,\y_0)}^\beta_\infty= B\cdot \norm{\x_1-\x_2}^\beta_\infty, 
\]meaning that 
$\norm{\x_1-\x_2}_\infty\geq \ykh{\tfrac{\abs{f_1(\y_0)-f_2(\y_0)}}{B}}^{1/\beta}$. Since $\x_1$, $\x_2$ are arbitrary, we obtain that 
\[
\operatorname{dist}(B_{f_1},B_{f_2})=\inf_{\x_1\in B_{f_1}, \x_2\in B_{f_2}}\norm{\x_1-\x_2}_\infty\geq \ykh{\tfrac{\abs{f_1(\y_0)-f_2(\y_0)}}{B}}^{1/\beta}>0. 
\]Since $f_1$ and $f_2$ are arbitrary and $\abs{\G}\leq K<\infty$, we get
\[
\delta:=&\inf_{f_1\in\G}\operatorname{dist}(B_{f_1},\X\setminus B_{f_1})=\inf_{f_1\in\G}\operatorname{dist}(B_{f_1},{}{\textstyle \bigcup_{f_2\in\G\setminus\hkh{f_1}}}B_{f_2})
\\&=\inf_{f_1\in\G}\inf_{{f_2\in\G\setminus\hkh{f_1}}}\operatorname{dist}(B_{f_1},{} B_{f_2})=\inf_{\G\ni f_1\neq f_2\in\G}\operatorname{dist}(B_{f_1},{} B_{f_2})>0. 
\]For any $f\in\G$, define 
\[
U_f:\R^{d_1}\to[0,1],\;\x\mapsto\frac{\operatorname{dist}(\hkh{\x}, \X\setminus B_f)}{\operatorname{dist}(\hkh{\x}, B_f)+\operatorname{dist}(\hkh{\x}, \X\setminus B_f)}
\] if $\X\setminus B_f\neq\varnothing$; define  $U_f(\x):\equiv 1$  if $\X\setminus B_f=\varnothing$. According to Lemma \ref{20260506192530}, \[\norm{U_f}_{1,[-R,R]^{d_1}}\leq\norm{U_f}_{1,\R^{d_1}}\leq \max\hkh{1,1/\delta},\;\forall\;f\in\G. \]
Note that $L_1-1\gtrsim  L_1\gtrsim \log\frac1\e$ and $\flr{N_1/K}\gtrsim N_1\gtrsim \log\frac1\e$. Then it  follows from Lemma \ref{20260506200157} that 
\[
&\sup_{\substack{h:[-R,R]^{d_1}\to \R\\{\norm{h}_{1, [-R,R]^{d_1}}\leq \max\hkh{1,\frac1\delta}}}}\inf\set{\sup_{{\y\in[-R,R]^{d_1}}}\abs{f(\y)-h(\y)}}{\begin{aligned}f\in \fdnn(d_1,L_1-1, \flr{N_1/K}, 1) \end{aligned}}\\&\lesssim {\frac{1}{\log^{4/d_1}(1/\e)}}\to 0. 
\]Therefore, in what follows, we may assume  that $\e$ is sufficiently small  so that 
\[
&\sup_{\substack{h:[-R,R]^{d_1}\to \R\\{\norm{h}_{1, [-R,R]^{d_1}}\leq \max\hkh{1,\frac1\delta}}}}\inf\set{\sup_{{\y\in[-R,R]^{d_1}}}\abs{f(\y)-h(\y)}}{\begin{aligned}f\in \fdnn(d_1,L_1-1, \flr{N_1/K}, 1) \end{aligned}}\\&<1/9<3\leq \flr{N_1/K}. 
\] Thus, for any $f\in\G$, there exists \begin{align}\label{20260506210355}
\widetilde{U}_f\in \fdnn(d_1,L_1-1,\flr{N_1/K},1)\end{align} such that \[\sup_{\x\in[-R,R]^{d_1}}\abs{U_f(\x)-\widetilde{U}_f(\x)}\leq 1/9. \] Therefore, 
\[
\widetilde{U}_f(\x)\geq U_f(\x)-1/9=1-1/9=8/9,\;\forall\;\x\in B_f,\;\forall\;f\in \G, 
\] and 
\[
\widetilde{U}_f(\x)\leq U_f(\x)+1/9=0+1/9=1/9,\;\forall\;\x\in \X\setminus B_f,\;\forall\;f\in\G. 
\] 
We then define $T(t):=3\sigma\ykh{t-1/3}- 3\sigma\ykh{t-2/3}$ and  $\widetilde{H}_f:=T\circ \widetilde{U}_f$. Obviously,
\[
T(x)=0<1=T(y),\;\forall\;x\in(-\infty,1/3], \;\forall\;y\in[2/3,\infty). 
\]  Hence
\[
\widetilde{H}_f(\x)=T(\widetilde{U}_f(\x))\xlongequal{\because\widetilde{U}_f(\x)\geq8/9}1,\;\forall\;\x\in B_f,\;\forall\;f\in\G, 
\] and 
\[
\widetilde{H}_f(\x)=T(\widetilde{U}_f(\x))\xlongequal{\because\widetilde{U}_f(\x)\leq1/9}0,\;\forall\;\x\in \X\setminus B_f,\;\forall\;f\in\G, 
\] meaning that
\begin{align}\label{20260507010211}
\widetilde{H}_f(\x)=\mathbb{I}\hkh{\x\in B_f},\;\forall\;\x\in \X,\;\forall\;f\in\G. 
\end{align}
Note that
\[T\in\fdnn(1,1,2,1)\subset \fdnn(1,1,\flr{N_1/K},1),\]
 which, together with \eqref{20260506210355}, yields 
\begin{align}\label{20260506222359}
\widetilde{H}_f=T\circ \widetilde{U}_f\in \fdnn(d_1,L_1,\flr{N_1/K},1),\;\forall\;f\in\G. 
\end{align}
For any $f\in\G$, there exist $k\in J$ and $\x_0\in A_k$ such that 
\[s^*(\x_0,\y)=\sum_{j=1}^Kg_j(\y)\cdot\mathbb{I}\hkh{\x_0\in A_j}=g_k(\y)=f(\y),\;\forall\;y\in\Y,\]  which, together with Assumption \ref{20260427030032}, yields
\[
\norm{f}_{\beta,\Y}= \norm{s^*(\x_0,\cdot)}_{\beta,\Y}\leq  \norm{s^*}_{\beta,\X\times\Y}\leq B. 
\] Thus 
\[
\norm{f}_{\beta,\Y}\leq B,\;\forall\;f\in\G. 
\] Then it follows from Lemma \ref{20260428013219} that for any $f\in \G$ there exists $Q_f:[-R,R]^{d_2}\to\R$ such that $\norm{Q_f}_{\beta, [-R,R]^{d_2}}\leq\norm{f}_{\beta,\Y}\leq B$, and 
\begin{align}\label{20260507002822}
f(\y)=Q_f(\y),\;\forall\;\y\in \Y. 
\end{align}
We then take \begin{align}\label{20260506222437}
\widetilde{Q}_f\in \fdnn(d_2,L_2, \flr{N_2/K},1)\end{align} such that 
\[
&\sup_{\y\in[-R,R]^{d_2}}\abs{Q_f(\y)-\widetilde{Q}_f(\y)}\\&\leq \inf\set{\sup_{{\y\in[-R,R]^{d_2}}}\abs{Q_f(\y)-u(\y)}}{\begin{aligned}u\in \fdnn(d_2,L_2,\flr{N_2/K}, 1) \end{aligned}}+\e. 
\] Define $\widetilde{S}(\x,\y):=\sum_{f\in\G}\widetilde{H}_f(\x)\cdot \widetilde{Q}_f(\y)$. Combining  \eqref{20260506222359} and \eqref{20260506222437}, we obtain that 
\[&\widetilde{S}\in\fscoring(L_1,\flr{N_1/K}\cdot \abs{\G},L_2,\flr{N_2/K}\cdot\abs{\G},\abs{\G})
\\&\subset \fscoring(L_1,\flr{N_1/K}\cdot K,L_2,\flr{N_2/K}\cdot K,K)\subset  \fscoring(L_1,N_1,L_2, N_2,K). 
\] Then it follows from Lemma \ref{20260428072356} that 
\[
&\inf\set{\L(\tilde{s})-\L^*}{ \tilde{s}\in \fscoring(L_1,N_1,L_2,N_2,K)}\leq\L(\widetilde{S})-\L^*
\\&\leq 
 \E_\x\zkh{\sup_{\y\in\Y}\tfrac{\abs{\widetilde{S}(\x,\y)-s^*(\x,\y)}^2}{2\tau}}
 \leq \sup_{(\x,\y)\in\X\times\Y}\abs{\widetilde{S}(\x,\y)-s^*(\x,\y)}^2\cdot \frac{1}{2\tau}
 \\&=
 \sup_{(\x,\y)\in\X\times\Y}\abs{\sum_{f\in\G}\widetilde{H}_f(\x)\cdot \widetilde{Q}_f(\y)-\sum_{f\in\G}f(\y)\cdot\mathbb{I}\hkh{\x\in  B_f}}^2\cdot \frac{1}{2\tau}
 \\&\xlongequal{\because\eqref{20260507010211}}
 \sup_{(\x,\y)\in\X\times\Y}\abs{\sum_{f\in\G}\mathbb{I}\hkh{\x\in B_f}\cdot \widetilde{Q}_f(\y)-\sum_{f\in\G}f(\y)\cdot\mathbb{I}\hkh{\x\in  B_f}}^2\cdot \frac{1}{2\tau}
 \\&=\sup_{f\in\G}\sup_{\y\in Y}\abs{\widetilde{Q}_f(\y)-f(\y)}^2\cdot \frac{1}{2\tau}\xlongequal{\because\eqref{20260507002822}} \sup_{f\in\G}\sup_{\y\in Y}\abs{\widetilde{Q}_f(\y)-Q_f(\y)}^2\cdot \frac{1}{2\tau}
 \\&\leq \sup_{f\in\G}\sup_{\y\in [-R,R]^{d_2}}\abs{\widetilde{Q}_f(\y)-Q_f(\y)}^2\cdot \frac{1}{2\tau}. 
\]
Thus 
\[
&\inf\set{\L(\tilde{s})-\L^*}{ \tilde{s}\in \fscoring(L_1,N_1,L_2,N_2,K)}\leq \sup_{f\in\G}\sup_{\y\in [-R,R]^{d_2}}\abs{\widetilde{Q}_f(\y)-Q_f(\y)}^2\cdot \frac{1}{2\tau}
\\&\leq \frac{1}{2\tau}\cdot\sup_{f\in\G} \abs{\inf\set{\sup_{{\y\in[-R,R]^{d_2}}}\abs{Q_f(\y)-u(\y)}}{\begin{aligned}u\in \fdnn(d_2,L_2,\flr{N_2/K}, 1) \end{aligned}}+\e}^2\\&\leq \frac{\e^2}{\tau}+\frac{1}{\tau}\cdot \sup_{f\in\G} \abs{\inf\set{\sup_{{\y\in[-R,R]^{d_2}}}\abs{Q_f(\y)-u(\y)}}{\begin{aligned}u\in \fdnn(d_2,L_2,\flr{N_2/K}, 1) \end{aligned}}}^2
\\&\leq \frac{\e^2}{\tau}+\frac{1}{\tau}\cdot \sup_{\substack{h:[-R,R]^{d_2}\to\R\\\norm{h}_{\beta,[-R,R]^{d_2}}\leq B}} \abs{\inf\set{\sup_{{\y\in[-R,R]^{d_2}}}\abs{h(\y)-u(\y)}}{\begin{aligned}u\in \fdnn(d_2,L_2,\flr{N_2/K}, 1) \end{aligned}}}^2. 
\] Hence 
\[
&\inf\set{\L(\tilde{s})-\L^*}{ \tilde{s}\in \fscoring(L_1,N_1,L_2,N_2,K)}
\\&\leq  \frac{\e^2}{\tau}+\frac{1}{\tau}\cdot \sup_{\substack{h:[-R,R]^{d_2}\to\R\\\norm{h}_{\beta,[-R,R]^{d_2}}\leq B}} \abs{\inf\set{\sup_{{\y\in[-R,R]^{d_2}}}\abs{h(\y)-u(\y)}}{\begin{aligned}u\in \fdnn(d_2,L_2,\flr{N_2/K}, 1) \end{aligned}}}^2. 
\]
Note that $L_2\cdot \flr{N_2/K}\asymp L_2\cdot N_2\asymp \e^{-\frac{d_2}{4\beta}}$. Therefore,  we can use   Lemma \ref{20260506221445} to obtain that 
\[
&\inf\set{\L(\tilde{s})-\L^*}{ \tilde{s}\in \fscoring(L_1,N_1,L_2,N_2,K)}
\\&\lesssim  {\e^2}+ \sup_{\substack{h:[-R,R]^{d_2}\to\R\\\norm{h}_{\beta,[-R,R]^{d_2}}\leq B}} \abs{\inf\set{\sup_{{\y\in[-R,R]^{d_2}}}\abs{h(\y)-u(\y)}}{\begin{aligned}u\in \fdnn(d_2,L_2,\flr{N_2/K}, 1) \end{aligned}}}^2
\\&\lesssim \e^2+\abs{\e^{1/2}}^2\lesssim \e. 
\]
This completes the proof. 
\end{proof}

\section{Pseudo-Dimension of $\pddm_{\R^{d_1+d_2}}\ykh{\fscoring(L_1,N_1, L_2,N_2, D)}$ }\label{20260502111322}

Let $\Omega$ be a nonempty set and $\mc{G}$ be a class of real-valued functions whose domains  contain $\Omega$. The pseudo-dimension of $\mc{G}$ on $\Omega$ is defined as 
\[
\pddm_\Omega(\mc{G}):=\sup\set{m\in\mb{N}}{\mathbf{G}_{\Omega,\mc{G}}(m)=2^m}, 
\] where
$\mathbf{G}_{\Omega,\mc{G}}(m)$ is the so-called  growth function of $\mc{G}$ on $\Omega$,  defined as
\[
&\mathbf{G}_{\Omega,\mc{G}}(m)\\&:=\max_{\substack{\x_1,\ldots,\x_m\in \Omega\\t_1,\ldots, t_m\in\R}}\abs{\setm{2.3}{\ykh{\mathbb{I}\hkh{f(\x_1)\geq t_1}, \mathbb{I}\hkh{f(\x_2)\geq t_2}, \ldots, \mathbb{I}\hkh{f(\x_m)\geq t_m}}}{f\in\mc{G}}}. 
\]

\begin{theorem}\label{20260503025931}
For any   $L_1\in\mb{N}$,  $L_2\in\mb{N}$,  
$N_1\in\mb{N}\cap(1,\infty)$, $N_2\in\mb{N}\in(1,\infty)$,     $D\in\mb{N}$, there holds \[
&\pddm_{\R^{d_1+d_2}}\ykh{\fscoring(L_1,N_1, L_2,N_2, D)}
\\&\leq 32(L_1^2N_1^2+L_1N_1d_1+N_1D+L_2^2N_2^2+L_2N_2d_2+N_2D)\log(L_1L_2+L_1N_1+L_2N_2). 
\]
\end{theorem}
\begin{proof}
For $j=1,2$, we write 
\[
\mathbf{f}_j(\x, \w_j)&=\ykh{{f}_{j,1}(\x, \w_j), {f}_{j,2}(\x, \w_j), \ldots, {f}_{j,D}(\x, \w_j)}^\top\\&=\bm W_{L_j}\sigma_{\v_{L_j}}\bm W_{L_j-1}\sigma_{\v_{L_j-1}}\cdots\bm W_2\sigma_{\v_2}\bm W_1\sigma_{\v_1}\bm W_0\x, 
\]
where  $\v_i\in\mbR^{N_j}$ ($1\leq i\leq L_j$),  $\bm{W}_i\in\mbR^{N_j\times N_j}$ ($1\leq i\leq L_j-1$), $\bm{W}_0\in\mbR^{N_j\times d_j}$, $\bm{W}_{L_j}\in\mbR^{D\times N_j}$, and $\w_j$ is the vector whose components consist of all the components of $\bm{W}_0,\ldots, \bm{W}_{L_j}, \v_1,\ldots \v_{L_j}$. 
For $\iota\in\hkh{1,2,\ldots, L_j}$, we use $k_{j,\iota}$ to denote the total number of entries in $\bm{W}_0,\ldots, \bm{W}_{\iota-1}, \v_1,\ldots \v_{\iota}$. Then let  $k_{j, L_{j}+1}$ denote the total number of entries in $\bm{W}_0,\ldots, \bm{W}_{L_j}, \v_1,\ldots \v_{\iota}$.  Thus $k_{j, L_j+1}=(L_j+d_j+D)\cdot N_j+(L_j-1)\cdot N_j^2$,  and  $\w_j\in\R^{k_{j, L_j+1}}$. For simplicity, we denote $K_j:=k_{j, L_j+1}$, $\Omega:=\R^{d_1+d_2}$, and
\[
S(\x,\y,\w_1,\w_2):=\mathbf{f}_1(\x, \w_1)^\top \mathbf{f}_2(\y, \w_2)=\sum_{i=1}^D f_{1,i}(\x,\w_1)f_{2,i}(\y,\w_2). 
\]
Obviously, 
\[
\mc{G}&:=\set{(\x,\y)\mapsto S(\x,\y,\w_1,\w_2)}{\w_1\in\R^{K_1}, \w_2\in \R^{K_2}}
={\fscoring(L_1,N_1, L_2,N_2, D)}, 
\] meaning that  \begin{align}\label{20260429093848}
\pddm_{\Omega}\ykh{\fscoring(L_1,N_1, L_2,N_2, D)}= \pddm_{\Omega}\ykh{\mc{G}}=\sup\set{m\in\mb{N}}{\mathbf{G}_{\Omega,\mc{G}}(m)=2^m}. 
\end{align}We next establish upper bounds for the growth function $\mathbf{G}_{\Omega,\mc{G}}(m)$.

Let $m\in\mb{N}$, $\x_{1,1},\ldots, \x_{1,m}\in\R^{d_1}$, $\x_{2,1},\ldots, \x_{2,m}\in\R^{d_2}$, and $t_1,\ldots, t_m\in\R$ be arbitrary. By the arguments used in the proof of Theorem 7 in \cite{bartlett2019nearly}, there exist a partition  $\mc{U}_1$ of $\R^{K_1}$ and a partition $\mc{U}_2$ of $\mc{R}^{K_2}$, such that \begin{align*}
&\abs{\mc{U}_j}\leq\prod_{l=1}^{L_j}2\cdot\ykh{\frac{2\mathrm{e}m N_j\cdot 1\cdot(1+(l-1)\cdot 1^{l-1})}{k_{j,l}}}^{k_{j,l}}
\\&\leq
\prod_{l=1}^{L_j}\ykh{\frac{12m N_j\cdot L_j}{k_{j,l}}}^{k_{j,l}}
,\;\forall\;j\in\hkh{1,2}, 
\end{align*} and that 
 for any $j\in\hkh{1,2}$, $U\in\mc{U}_j$,  $r\in\hkh{1,2,\ldots, m}$, $i\in\hkh{1,2,\ldots, D}$, the function
 \[
 U\ni \w\mapsto  f_{j,i}(\x_{j,r}, \w)\in\R
 \] coincides with a polynomial with total degree no more than $(1+L_j)$. Therefore, for any   $U_1\in\mc{U}_1$, $U_2\in\mc{U}_2$,  $r\in\hkh{1,2,\ldots, m}$, the function 
 \[
 U_1\times U_2\ni(\w_1,\w_2)\mapsto  S(\x_{1,r},
 \x_{2,r},\w_1,\w_2)\in\R
 \]  coincides with a polynomial, which we denote by $Q_{r,U_1,U_2}$, with total degree no more than ${(1+L_1)(1+L_2)}$. We can then apply Lemma 17 of \cite{bartlett2019nearly} to obtain that 
 \[
 &\abs{\setm{2.3}{\ykh{\mathbb{I}\hkh{H(\x_{1,1},\x_{2,1})\geq t_1}, \ldots, \mathbb{I}\hkh{H(\x_{1,m},\x_{2,m})\geq t_m}}}{H\in\mc{G}}}
 \\&=\abs{\bigcup_{\substack{\w_1\in\R^{K_1}\\ \w_2\in \R^{K_2}}}\hkh{\ykh{\mathbb{I}\hkh{S(\x_{1,1},\x_{2,1},\w_1,\w_2)\geq t_1}, \ldots, \mathbb{I}\hkh{S(\x_{1,m},\x_{2,m},\w_1,\w_2)\geq t_m}}}{}}
 \\&\leq \sum_{\substack{U_1\in \mc{U}_1\\U_2\in\mc{U}_2}}\abs{\bigcup_{\substack{\w_1\in U_1\\ \w_2\in U_2}}\hkh{\ykh{\mathbb{I}\hkh{S(\x_{1,1},\x_{2,1},\w_1,\w_2)\geq t_1}, \ldots, \mathbb{I}\hkh{S(\x_{1,m},\x_{2,m},\w_1,\w_2)\geq t_m}}}{}}
 \\&=\sum_{\substack{U_1\in \mc{U}_1\\U_2\in\mc{U}_2}}\abs{\bigcup_{\substack{\w_1\in U_1\\ \w_2\in U_2}}\hkh{\ykh{\mathbb{I}\hkh{t_1-Q_{1,U_1,U_2}(\w_1,\w_2)\leq 0}, \ldots, \mathbb{I}\hkh{t_m-Q_{m,U_1,V_1}(\w_1,\w_2)\leq 0}}}{}}
 \\&=\sum_{\substack{U_1\in \mc{U}_1\\U_2\in\mc{U}_2}}\abs{\bigcup_{\substack{\w_1\in U_1\\ \w_2\in U_2}}\hkh{\ykh{\mathbb{I}\hkh{t_1-Q_{1,U_1,U_2}(\w_1,\w_2)>0}, \ldots, \mathbb{I}\hkh{t_m-Q_{m,U_1,V_1}(\w_1,\w_2)> 0}}}{}}
\\& \leq \sum_{\substack{U_1\in \mc{U}_1\\U_2\in\mc{U}_2}} 2\cdot (2\mathrm{e}m (1+L_1)(1+L_2)/(K_1+K_2))^{K_1+K_2}. 
 \] Thus 
 \[
& \abs{\setm{2.3}{\ykh{\mathbb{I}\hkh{H(\x_{1,1},\x_{2,1})\geq t_1}, \ldots, \mathbb{I}\hkh{H(\x_{1,m},\x_{2,m})\geq t_m}}}{H\in\mc{G}}}
 \\&\leq  \sum_{\substack{U_1\in \mc{U}_1\\U_2\in\mc{U}_2}}2\cdot (2\mathrm{e}m (1+L_1)(1+L_2)/(K_1+K_2))^{K_1+K_2}
 \\&={{\abs{ \mc{U}_1}\abs{\mc{U}_2}}}2\cdot (2\mathrm{e}m (1+L_1)(1+L_2)/(K_1+K_2))^{K_1+K_2}
 \\&\leq \ykh{\frac{48mL_1L_2}{K_1+K_2}}^{K_1+K_2}\prod_{j=1}^2\prod_{l=1}^{L_j}\ykh{\frac{12m N_j\cdot L_j}{k_{j,l}}}^{k_{j,l}}. 
 \] Since $\x_{1,1},\ldots, \x_{1,m}\in\R^{d_1}$, $\x_{2,1},\ldots, \x_{2,m}\in\R^{d_2}$, and $t_1,\ldots, t_m\in\R$ are arbitrary, we get 
 \[
 &\mathbf{G}_{\Omega,\mc{G}}(m)\\&=\max_{\substack{\x_{1,1},\ldots, \x_{1,m}\in\R^{d_1}\\\x_{2,1},\ldots, \x_{2,m}\in\R^{d_2}\\t_1,\ldots, t_m\in\R}}\abs{\setm{2.3}{\ykh{\mathbb{I}\hkh{H(\x_{1,1},\x_{2,1})\geq t_1}, \ldots, \mathbb{I}\hkh{H(\x_{1,m},\x_{2,m})\geq t_m}}}{H\in\mc{G}}}\\&\leq \ykh{\frac{48mL_1L_2}{K_1+K_2}}^{K_1+K_2}\prod_{j=1}^2\prod_{l=1}^{L_j}\ykh{\frac{12m N_j\cdot L_j}{k_{j,l}}}^{k_{j,l}}. 
 \] Using the  weighted AM-GM inequality, we obtain 
 \[
 &\mathbf{G}_{\Omega,\mc{G}}(m)\leq \ykh{\frac{48mL_1L_2}{K_1+K_2}}^{K_1+K_2}\prod_{j=1}^2\prod_{l=1}^{L_j}\ykh{\frac{12m N_j\cdot L_j}{k_{j,l}}}^{k_{j,l}}
 \\&\leq\ykh{\frac{48mL_1L_2+12mN_1L_1^2+12mN_2L_2^2}{K_1+K_2+\sum_{j=1}^2\sum_{l=1}^{L_j} k_{j,l}}}^{K_1+K_2+\sum_{j=1}^2\sum_{l=1}^{L_j} k_{j,l}}. 
 \] Denote $A={48L_1L_2+12N_1L_1^2+12N_2L_2^2}$, $W=K_1+K_2+\sum_{j=1}^2\sum_{l=1}^{L_j} k_{j,l}$. The we have that 
 \[
 \mathbf{G}_{\Omega,\mc{G}}(m)\leq \ykh{\frac{Am}{W}}^{W}. 
 \] Since $m$ is arbitrary, we obtain that $\pddm_{\Omega}\ykh{\mc{G}}=\sup\set{m\in\mb{N}}{\mathbf{G}_{\Omega,\mc{G}}(m)=2^m}<\infty$, and 
 \[
 2^{\pddm_{\Omega}\ykh{\mc{G}}}=\mathbf{G}_{\Omega,\mc{G}}({\pddm_{\Omega}\ykh{\mc{G}}})\leq \ykh{\frac{A\cdot {\pddm_{\Omega}\ykh{\mc{G}}}}{W}}^W. 
 \] We can then apply Lemma 18 of \cite{bartlett2019nearly} to derive from the above inequality that 
 \[
 &{\pddm_{\Omega}\ykh{\mc{G}}}\leq W\log_2(2A\log_2A)\leq 4W\log A
 \leq (K_1+K_2+L_1k_{1,L_1}+L_2k_{2,L_2})\cdot 4\log A
 \\&\leq  (K_1+K_2+L_1k_{1,L_1}+L_2k_{2,L_2})\cdot 16\log(L_1L_2+L_1N_1+L_2N_2)
 \\&=\abs{\sum_{j=1}^2((L_j+1)\ykh{(L_j-1)N_j^2+L_jN_j+N_jd_j}+N_jD)}\cdot 16\log(L_1L_2+L_1N_1+L_2N_2)
 \\&\leq \abs{\sum_{j=1}^2(L_j^2N_j^2+L_jN_jd_j+N_jD)}\cdot 32\log(L_1L_2+L_1N_1+L_2N_2)
 \\&= 32(L_1^2N_1^2+L_1N_1d_1+N_1D+L_2^2N_2^2+L_2N_2d_2+N_2D)\log(L_1L_2+L_1N_1+L_2N_2),
 \] which, together with \eqref{20260429093848}, proves the desired inequality. \end{proof}

\section{Proof of Theorem~\ref{thm:fast_rate_tsy_crl}}\label{sec:fast_rate_tsy_crl}

Throughout this section, we write
$r(\x,\y):=p(\y\mid\x)/p_\Y(\y)$ for simplicity.
We first recall the following technical lemma from
\citet[Theorem E.14]{li2026statisticalconsistencygeneralizationcontrastive}. It controls the excess risk of ranking through the misalignment of the order induced by scoring function and density ratio as well as the gap of density ratio.

\begin{lemma}\label{lem:ranking-regret-upper}
Let Assumption~\ref{ass:pi_+} hold. For any measurable
scoring function $s:\X\times\Y\to\R$, define
\begin{align*}
A_s&:=\bigl\{(\x,\y,\y')\in\X\times\Y\times\Y:
\bigl(r(\x,\y)-r(\x,\y')\bigr)
\bigl(s(\x,\y)-s(\x,\y')\bigr)\le0\bigr\}.
\end{align*}
Then
\begin{equation}\label{eq:ranking-regret-upper}
\mathcal E^*-\mathcal E(s)
\le\frac12\E_{\x,\y,\y'}\left[
\I\{(\x,\y,\y')\in A_s\}|r(\x,\y)-r(\x,\y')|\right].
\end{equation}
\end{lemma}

We next state the Hellinger--KL inequality used to control the difference
between the square roots of two densities by their KL divergence.

\begin{lemma}[Hellinger--KL inequality]\label{lem:hellinger-kl}
Let $P$ and $Q$ be probability measures on the same measurable space,
with densities $p$ and $q$ with respect to a common  measure
$\nu$. Suppose that
$P\ll Q$. Define the squared Hellinger distance by
\[
H^2(P,Q):=\int(\sqrt p-\sqrt q)^2\,d\nu.
\]
 Then
\begin{equation}\label{eq:hellinger-kl}
H^2(P,Q)\le\operatorname{KL}(P\,\|\,Q),
\end{equation}
where $\operatorname{KL}(P\,\|\,Q)=\int p\log(p/q)\,d\nu$ . 
\end{lemma}
\begin{proof}
Suppose $\operatorname{KL}(P\,\|\,Q)<
+\infty$, then $q>0$ holds $P$-almost everywhere.
We apply the elementary inequality $\log u\le u-1$ for $u>0$ to
$u=\sqrt{q/p}$ on $\{p>0\}$, gives
\[
p\log\frac pq=-2p\log\sqrt{\frac qp}
\ge2p-2\sqrt{pq}.
\]
Integrating over $\{p>0\}$ and using $\int p\,d\nu=\int q\,d\nu=1$,
we obtain
\[
\operatorname{KL}(P\,\|\,Q)
\ge2-2\int\sqrt{pq}\,d\nu
=\int(\sqrt p-\sqrt q)^2\,d\nu
=H^2(P,Q).
\]
This completes the proof.
\end{proof}

\begin{proof}[Proof of Theorem~\ref{thm:fast_rate_tsy_crl}]
Define $Z_s(\x):=\E_{\y\sim P_\Y}\exp(s(\x,\y)/\tau)$ and
\[
a_s(\x,\y):=\frac{\exp(s(\x,\y)/\tau)}{Z_s(\x)},
\qquad dQ_{s,\x}(\y):=a_s(\x,\y)dP_\Y(\y).
\]
Then $Q_{s,\x}$ is a probability measure on $\Y$.
We further have
\begin{equation}\label{eq:excess-risk-kl-calibration}
\mathcal L(s)-\mathcal L^*=\tau\E_{\x\sim P_\X}\left[
\operatorname{KL}\bigl(P_{\Y\mid\x}\,\|\,Q_{s,\x}\bigr)\right],
\end{equation}
where $P_{\Y\mid\x}$ denotes the conditional distribution on $\Y$
given $\x$. Define
$h_s(\x,\y):=\sqrt{r(\x,\y)}-\sqrt{a_s(\x,\y)}$.
Applying Lemma~\ref{lem:hellinger-kl} to $P_{\Y\mid\x}$ and
$Q_{s,\x}$, whose densities with respect to $P_\Y$ are $r(\x,\cdot)$
and $a_s(\x,\cdot)$, and then taking expectation over $\x$, we obtain
\begin{equation}\label{eq:hellinger-controlled-by-risk}
\E_{\x,\y}
\left[h_s(\x,\y)^2\right]=\E_\x\int \ykh{\sqrt{\frac{dP_{\Y\mid\x}(\y)}{dP_\Y(\y)}}-\sqrt{\frac{dQ_{s,\x}(\y)}{dP_\Y(\y)}}}^2dP_\Y(\y)\le\frac{\mathcal L(s)-\mathcal L^*}{\tau}.
\end{equation}

Fix $(\x,\y,\y')\in A_s$.
Without loss of generality, we assume that
$r(\x,\y)\ge r(\x,\y')$. If $r(\x,\y)> r(\x,\y')$, the definition of $A_s$ gives
$s(\x,\y)\le s(\x,\y')$, and hence
$a_s(\x,\y)\le a_s(\x,\y')$. It follows that
\[
0\le\sqrt{r(\x,\y)}-\sqrt{r(\x,\y')}
\le |h_s(\x,\y)|+|h_s(\x,\y')|.
\]
Using $r(\x,\y),r(\x,\y')\le\alpha$, we obtain
\begin{align*}
|r(\x,\y)-r(\x,\y')|^2
&=\bigl(\sqrt{r(\x,\y)}+\sqrt{r(\x,\y')}\bigr)^2
\bigl(\sqrt{r(\x,\y)}-\sqrt{r(\x,\y')}\bigr)^2\\
&\le4\alpha\bigl(|h_s(\x,\y)|+|h_s(\x,\y')|\bigr)^2\\
&\le8\alpha\bigl(h_s(\x,\y)^2+h_s(\x,\y')^2\bigr).
\end{align*}
The above inequality naturally holds for $r(\x,\y)= r(\x,\y')$.
Taking expectations
and applying \eqref{eq:hellinger-controlled-by-risk}, we have
\begin{align}
&\E_{\x,\y,\y'}\left[
\I\{(\x,\y,\y')\in A_s\}|r(\x,\y)-r(\x,\y')|^2\right]
\le  \E_{\x,\y,\y'}\zkh{8\alpha\bigl(h_s(\x,\y)^2+h_s(\x,\y')^2\bigr)}\notag\\
=&16\alpha\E_{\x,\y}
\left[h_s(\x,\y)^2\right]
\le\frac{16\alpha}{\tau}\ykh{\mathcal L(s)-\mathcal L^*}.\label{eq:misranking-square-bound}
\end{align}

For any $t>0$, Lemma~\ref{lem:ranking-regret-upper},
\eqref{eq:misranking-square-bound}, and Assumption~\ref{ass:tsy_crl} imply
\begin{align*}
\mathcal E^*-\mathcal E(s)
&\le\frac12\E_{\x,\y,\y'}[\I\{(\x,\y,\y')\in A_s\}|r(\x,\y)-r(\x,\y')|\I\{|r(\x,\y)-r(\x,\y')|>t\}]\\
 &+\frac12\E_{\x,\y,\y'}[\I\{(\x,\y,\y')\in A_s\}|r(\x,\y)-r(\x,\y')|\I\{|r(\x,\y)-r(\x,\y')|\le t\}]\\
&\le\frac{1}{2t}\E_{\x,\y,\y'}[\I\{(\x,\y,\y')\in A_s\}|r(\x,\y)-r(\x,\y')|^2]
 +\frac{t}{2}\P(|r(\x,\y)-r(\x,\y')|\le t)\\
&\le\frac{8\alpha (\mathcal L(s)-\mathcal L^*)}{\tau t}+\frac{ct^{\mu+1}}{2}.
\end{align*}
Taking
$t=(16\alpha (\mathcal L(s)-\mathcal L^*)/(c\tau ))^{1/(\mu+2)}$ gives
\[
\mathcal E^*-\mathcal E(s)
\le c^{\frac1{\mu+2}}
\left(\frac{16\alpha (\mathcal L(s)-\mathcal L^*)}{\tau}\right)^{\frac{\mu+1}{\mu+2}}.
\]
This completes the proof.
\end{proof}

\section{Numerical Experiments}
\label{sec:experiments}
In this section, we 
We conduct experiments on Gaussian mixture models to examine three aspects of
our theory. First, a balanced Gaussian-mixture model with infinite rank illustrates the
prediction in Section~\ref{sec:rkhs_approximation} that approximation error
decreases as the representation dimension increases. Second, 
we train neural networks on imbalanced Gaussian mixture models to examine the representation structure and
dimension dependence described in Section~\ref{sec:latent-class-model}.
Finally, we verify the
density-ratio upper bound in Assumption~\ref{ass:pi_+} and Tsybakov noise condition in Assumption~\ref{ass:tsy_crl}.

\subsection{Approximation error in a Gaussian mixture model}\label{app:gmm_experiments}
\paragraph{Setup.} We first study spectral decay and approximation error as the representation
dimension $D$ increases in a balanced binary Gaussian mixture model with an infinite-rank PMI
operator. Let $z$ be uniform on $\{-1,+1\}$, with conditionally independent views
$\x\mid z\sim\mathcal N(z\mu_x,\sigma_x^2)$ and
$\y\mid z\sim\mathcal N(z\mu_y,\sigma_y^2)$.
Bayes' rule gives
\[
\mathrm{PMI}(\x,\y)=\log[1+u(\x)v(\y)],\qquad
u(\x)=\tanh(\mu_x\x/\sigma_x^2),\quad v(\y)=\tanh(\mu_y\y/\sigma_y^2).
\]
For nonzero means, the expansion
$\mathrm{PMI}(\x,\y)=\sum_{j\ge1}(-1)^{j+1}u(\x)^jv(\y)^j/j$
has infinite rank. 

We use $\mu_x=\mu_y=1.2$ and $\sigma_x=\sigma_y=1$.
Gauss--Hermite quadrature under the mixture marginals gives the following
approximate eigenvalues of $T^*T$:
\begin{center}
\begin{tabular}{crrrrrr}
\toprule
$j$&1&2&4&8&12&20\\
\midrule
$\lambda_j$&1.291&$4.763\!\times\!10^{-1}$&$4.056\!\times\!10^{-3}$&$1.592\!\times\!10^{-5}$&$2.257\!\times\!10^{-7}$&$1.973\!\times\!10^{-10}$\\
\bottomrule
\end{tabular}
\end{center}
The observed spectrum decays rapidly despite infinite rank  (of order $10^{-i/2}$ approximately).. 

\paragraph{Results.}
we evaluate (i) the rank-$D$ spectral truncation for Theorem~\ref{thm:appro_rkhs} and (ii) two  ReLU encoders 
trained with the contrastive objective. We report both $L_2$ norm and 
approximation error for contrastive loss.
\begin{center}
\small
\begin{tabular}{rrrrr}
\toprule
$D$ & Spectral tail & \shortstack[r]{spectral truncation\\approximation error} & MLP $L_2$ distance & \shortstack[r]{MLP approximation\\error}\\
\midrule
1 & $4.902\times10^{-1}$ & $3.808\times10^{-2}$ & $3.763\times10^{-1}$ & $9.941\times10^{-3}$\\
2 & $1.397\times10^{-2}$ & $9.984\times10^{-4}$ & $4.607\times10^{-2}$ & $4.616\times10^{-4}$\\
4 & $6.822\times10^{-4}$ & $5.638\times10^{-5}$ & $3.952\times10^{-3}$ & $1.877\times10^{-5}$\\
8 & $6.165\times10^{-6}$ & $5.587\times10^{-7}$ & $1.529\times10^{-3}$ & $9.188\times10^{-6}$\\
\bottomrule
\end{tabular}
\captionof{table}{Approximation errors for infinite-rank Gaussian mixture model.}
\label{tab:gmm_errors}
\end{center}
As we can see, the approximation error gets smaller as $D$ grows for both
spectral truncation and MLP.

\subsection{Neural-network experiments under LCGM}
\label{app:lcgm_experiments}
We next examine the structured-data predictions in
Section~\ref{sec:latent-class-model} using an imbalanced Gaussian
mixture with a cleaner anchor view. We evaluate whether the two encoders
recover the posterior features $p(z=k\mid\x)$ and log-posterior-ratio
features $\log[p(z=k\mid\y)/\pi_k]$, and whether representation quality
saturates around the latent-class dimension $K$.

\paragraph{Setup.}
A latent class
$z\in[K]$ is sampled from a prior distribution $\pi$, where
$\pi_k\propto 1/(k+1)$. Conditional on $z=k$, the two views are generated
independently as
\[
\x\mid z=k \sim \mathcal N(\mu_k^1,\sigma_1^2 I),
\qquad
\y\mid z=k \sim \mathcal N(\mu_k^2,\sigma_2^2 I).
\]
The class centers $\{\mu_k^1\}_{k=1}^K$ and $\{\mu_k^2\}_{k=1}^K$ are sampled
independently. We choose
$\sigma_1=0.35<\sigma_2=1$, so that the $\X$-view is substantially cleaner than
the $\Y$-view. As a result, $p(z|\x)$ is close to a one-hot latent identity
feature, approximating the deterministic setting in \eqref{eq:deterministic}. We use $K=8$ latent classes and
$d_1=d_2=16$. We train three-layer MLP encoders with hidden dimension $128$
using the inner-product score with $\tau=0.5$. The model is optimized by AdamW for $20$ epochs
with batch size $512$.  The experiments are performed on Google colab with T4 GPU. Our code is available at \url{https://github.com/YuanfanLi2233/approximation_CRL_LCGM_NIPS2026}.

\begin{figure}[htbp]
    \centering
    \begin{subfigure}[t]{0.42\linewidth}
        \centering
        \includegraphics[width=\linewidth]{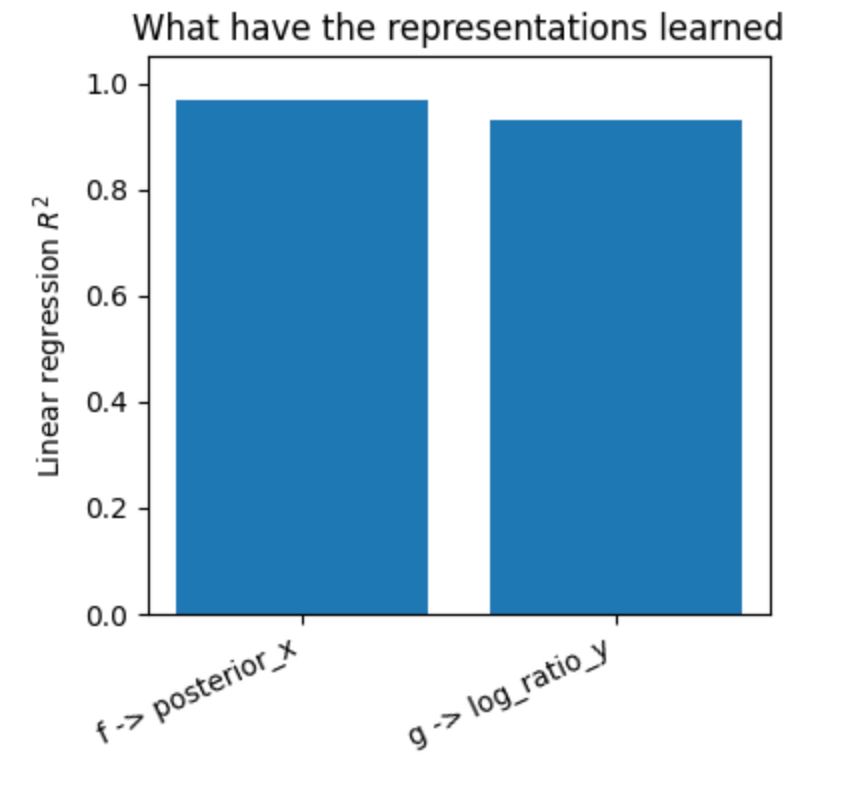}
        \caption{Recovered singular functions.}
        \label{fig:representation-factors}
    \end{subfigure}
    \hfill
    \begin{subfigure}[t]{0.50\linewidth}
        \centering
        \includegraphics[width=\linewidth]{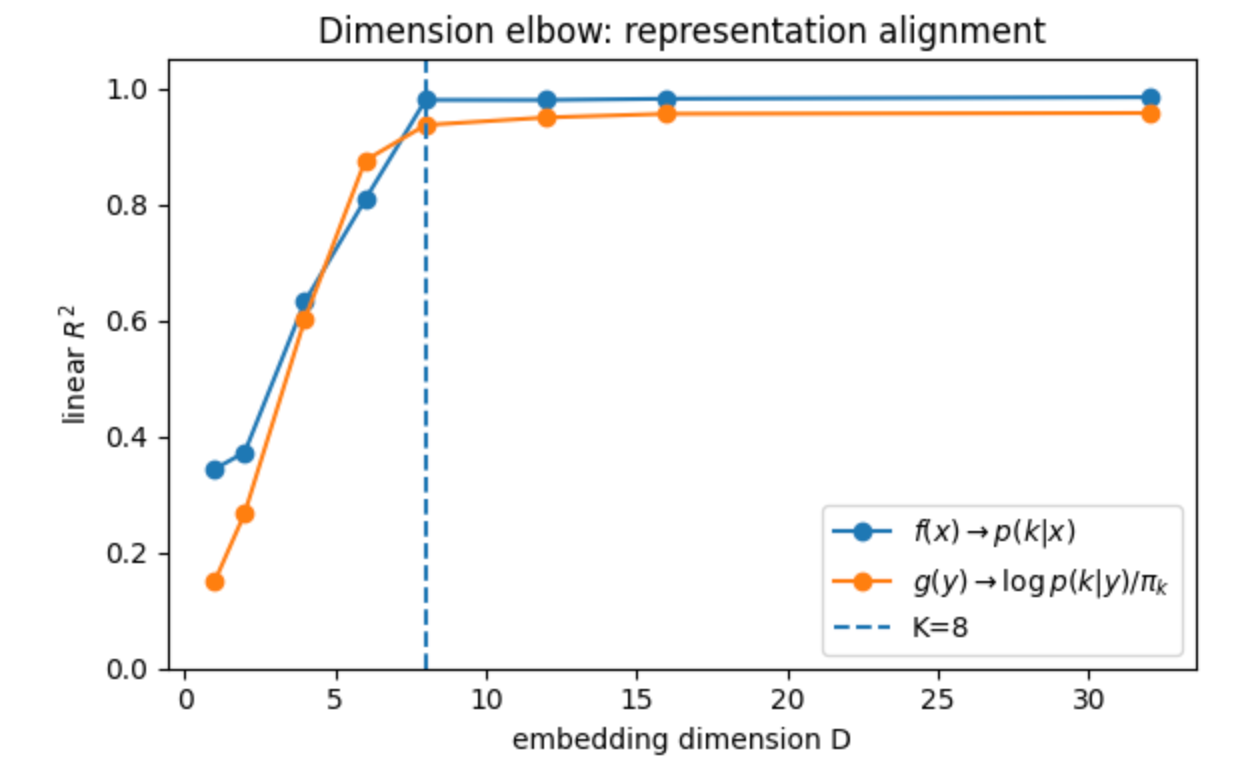}
        \caption{Dimension elbow.}
        \label{fig:dimension-elbow}
    \end{subfigure}
    \caption{
    \textbf{Synthetic validation of the representation learning in LCGM.}
    (a) The two encoders recover the probabilistic features predicted by
    Proposition~\ref{prop:latent-spectrum_informal}.
    (b) Representation alignment improves as the embedding dimension increases
    and saturates around the latent rank $K=8$.
    }
    \label{fig:synthetic-experiments}
    \vspace*{-4mm}
\end{figure}

\paragraph{Evaluation.}
According to Proposition~\ref{prop:latent-spectrum_informal}, the encoder
$f$ should recover the posterior feature $p(z=k|\x)$, while the encoder $g$
should recover the log-posterior-ratio feature
$\log(p(z=k|\y)/\pi_k)$. In the Gaussian mixture model, both quantities can be
computed in closed form. 
We therefore evaluate the learned representations by linear probing: we report
the linear regression $R^2$ from $f(\x)$ to $p(z=k|\x)$ and from $g(\y)$ to
$\log(p(z=k|\y)/\pi_k)$. We use $R^2$ because
Proposition~\ref{prop:latent-spectrum_informal} identifies the relevant
singular-function spaces only up to an invertible linear transformation; hence
a high linear-regression $R^2$ indicates that the learned representation spans
the predicted feature space.
We also vary the embedding dimension $D$.

\paragraph{Representations recover singular functions.}
Figure~\ref{fig:representation-factors} shows what the two representations learn. The
clean-side encoder $f$ linearly predicts the posterior feature $p(z=k|\x)$ with
high accuracy, achieving $R^2\approx0.97$. The noisy-side encoder $g$ linearly
predicts the log-posterior-ratio feature $\log(p(z=k|\y)/\pi_k)$ with
$R^2\approx0.93$. This supports the spectral interpretation in 
Proposition~\ref{prop:latent-spectrum_informal}: 
the learned representations recover singular functions of the 
PMI operator.
\paragraph{Representation quality for different dimensions.}
Figure~\ref{fig:dimension-elbow} further examines the role of the embedding
dimension. Under LCGM, the PMI operator has rank at most $K$, so a
$K$-dimensional inner-product representation is sufficient at the population
level. Consistent with this prediction, the representation quality improves as
$D$ increases and saturates around $D=K$. When $D<K$, the representation is
rank-limited and cannot recover all latent semantic directions.

\subsection{Verification of bounded density-ratio and Tsybakov noise condition}\label{app:gmm_diagnostics}
In this section, we verify the bounded density-ratio condition in
Assumption~\ref{ass:pi_+} and empirically assess the Tsybakov noise condition in Assumption~\ref{ass:tsy_crl} 
using the Gaussian mixture model in Section~\ref{app:gmm_experiments}.
\paragraph{Bounded density ratio.}
In the above model, the density ratio
\[
\frac{p(\y\mid\x)}{p_\Y(\y)}
=1+\tanh(1.2\x)\tanh(1.2\y)
\]
lies in $(0,2)$. Thus Assumption~\ref{ass:pi_+}
holds with $\alpha=2$.
\paragraph{Tsybakov noise condition.}
For independent $(\x,\y,\y')\sim P_\X\otimes P_\Y\otimes P_\Y$,
we estimate the margin probabilities ${\P}\!\left(\left|
\frac{p(\y\mid\x)}{p_\Y(\y)}
-\frac{p(\y'\mid\x)}{p_\Y(\y')}
\right|\le t\right)$ for different values of $t$ in Assumption~\ref{ass:tsy_crl} through Monte Carlo sampling.
The estimated probabilities are shown in the following table.
\begin{center}
\begin{tabular}{cc}
\toprule
$t$ & $\widehat{\P}\!\left(\left|
\frac{p(\y\mid\x)}{p_\Y(\y)}
-\frac{p(\y'\mid\x)}{p_\Y(\y')}
\right|\le t\right)$\\
\midrule
$10^{-5}$ & $1.7360\times10^{-4}$\\
$3\times10^{-5}$ & $4.5680\times10^{-4}$\\
$10^{-4}$ & $1.3220\times10^{-3}$\\
$3\times10^{-4}$ & $3.4742\times10^{-3}$\\
$10^{-3}$ & $9.7374\times10^{-3}$\\
\bottomrule
\end{tabular}
\end{center}
A log--log fit gives the noise exponent
$\widehat\mu=0.878890$. 

\newpage
\section*{NeurIPS Paper Checklist}


\begin{enumerate}

\item {\bf Claims}
    \item[] Question: Do the main claims made in the abstract and introduction accurately reflect the paper's contributions and scope?
    \item[] Answer: \answerYes{} 
    \item[] Justification: We clarify our contributions and the scope of the paper in the abstract and
introduction.

\item {\bf Limitations}
    \item[] Question: Does the paper discuss the limitations of the work performed by the authors?
    \item[] Answer: \answerYes{} 
    \item[] Justification: 
    We have discusses the limitations of the work. For example,
    in the conclusion, we mention that our analysis can not cover the cosine similarity scoring function.

\item {\bf Theory assumptions and proofs}
    \item[] Question: For each theoretical result, does the paper provide the full set of assumptions and a complete (and correct) proof?
    \item[] Answer: \answerYes{} 
    \item[] Justification: We present all the necessary assumptions and include the proof of all theoretical
results.

    \item {\bf Experimental result reproducibility}
    \item[] Question: Does the paper fully disclose all the information needed to reproduce the main experimental results of the paper to the extent that it affects the main claims and/or conclusions of the paper (regardless of whether the code and data are provided or not)?
    \item[] Answer: \answerYes{} 
    \item[] Justification: As a theoretical paper, our main contributions are theoretical analyses and
proofs. The numerical experiments are designed solely to provide basic validation of our
theoretical findings. We have fully disclosed all experimental details in Section~\ref{sec:experiments} of the
paper.

\item {\bf Open access to data and code}
    \item[] Question: Does the paper provide open access to the data and code, with sufficient instructions to faithfully reproduce the main experimental results, as described in supplemental material?
    \item[] Answer: \answerYes{} 
    \item[] The details and code of experiments can be found in Appendix~\ref{sec:experiments}.

\item {\bf Experimental setting/details}
    \item[] Question: Does the paper specify all the training and test details (e.g., data splits, hyperparameters, how they were chosen, type of optimizer) necessary to understand the results?
    \item[] Answer: \answerYes{} 
    \item[] Justification: The details of experiments can be found in Section~\ref{sec:experiments}.

\item {\bf Experiment statistical significance}
    \item[] Question: Does the paper report error bars suitably and correctly defined or other appropriate information about the statistical significance of the experiments?
    \item[] Answer: \answerNA{} 
    \item[] Justification: Our paper is primarily theoretical in nature. The simple numerical experiments
are designed as deterministic verifications of our theoretical results rather than empirical
evaluations requiring statistical analysis.

\item {\bf Experiments compute resources}
    \item[] Question: For each experiment, does the paper provide sufficient information on the computer resources (type of compute workers, memory, time of execution) needed to reproduce the experiments?
    \item[] Answer: \answerYes{} 
    \item[] Justification: Given the simplicity of our numerical experiments (which are minimal and for
verification only), we conducted our experiments on Google Colab with T4 GPU.
    
\item {\bf Code of ethics}
    \item[] Question: Does the research conducted in the paper conform, in every respect, with the NeurIPS Code of Ethics \url{https://neurips.cc/public/EthicsGuidelines}?
    \item[] Answer: \answerYes{} 
    \item[] Justification: The research conducted in the paper confirm with the NeurIPS Code of Ethics.

\item {\bf Broader impacts}
    \item[] Question: Does the paper discuss both potential positive societal impacts and negative societal impacts of the work performed?
    \item[] Answer: \answerNA{} 
    \item[] Justification: This is a theoretical paper, and there is no societal impact of the work performed.
    
\item {\bf Safeguards}
    \item[] Question: Does the paper describe safeguards that have been put in place for responsible release of data or models that have a high risk for misuse (e.g., pre-trained language models, image generators, or scraped datasets)?
    \item[] Answer: \answerNA{} 
    \item[] Justification: This is a theoretical paper, which poses no such risks.

\item {\bf Licenses for existing assets}
    \item[] Question: Are the creators or original owners of assets (e.g., code, data, models), used in the paper, properly credited and are the license and terms of use explicitly mentioned and properly respected?
    \item[] Answer: \answerNA{} 
    \item[] Justification: This is a theoretical paper, and we do not use any existing assets.

\item {\bf New assets}
    \item[] Question: Are new assets introduced in the paper well documented and is the documentation provided alongside the assets?
    \item[] Answer: \answerNA{} 
    \item[] Justification:  We do not release any new assets.

\item {\bf Crowdsourcing and research with human subjects}
    \item[] Question: For crowdsourcing experiments and research with human subjects, does the paper include the full text of instructions given to participants and screenshots, if applicable, as well as details about compensation (if any)? 
    \item[] Answer: \answerNA{} 
    \item[] Justification: This is a theoretical paper, which does not involve crowdsourcing nor research
with human subjects.

\item {\bf Institutional review board (IRB) approvals or equivalent for research with human subjects}
    \item[] Question: Does the paper describe potential risks incurred by study participants, whether such risks were disclosed to the subjects, and whether Institutional Review Board (IRB) approvals (or an equivalent approval/review based on the requirements of your country or institution) were obtained?
    \item[] Answer: \answerNA{} 
    \item[] Justification: This is a theoretical paper, which does not involve crowdsourcing nor research
with human subjects.

\item {\bf Declaration of LLM usage}
    \item[] Question: Does the paper describe the usage of LLMs if it is an important, original, or non-standard component of the core methods in this research? Note that if the LLM is used only for writing, editing, or formatting purposes and does \emph{not} impact the core methodology, scientific rigor, or originality of the research, declaration is not required.
    \item[] Answer: \answerNA{} 
    \item[] Justification: The core method development in this research does not involve LLMs as any important, original, or non-standard components.

\end{enumerate}

\end{document}